\documentclass[lettersize,journal]{IEEEtran}
\usepackage{amsmath,amsfonts}
\usepackage{array}
\usepackage{colortbl}
\usepackage{diagbox}
\usepackage{adjustbox}
\usepackage[caption=false,font=normalsize,labelfont=sf,textfont=sf]{subfig}
\usepackage{cite}
\usepackage{textcomp}
\usepackage{stfloats}
\usepackage{amsthm}
\usepackage{url}
\usepackage{bm}
\usepackage{algorithm}
\usepackage{algpseudocode}
\usepackage{verbatim}
\usepackage{booktabs}
\usepackage{graphicx}
\usepackage{multirow} 
\usepackage{color}
\usepackage{diagbox}
\usepackage{siunitx}
\def\BibTeX{{\rm B\kern-.05em{\sc i\kern-.025em b}\kern-.08em
    T\kern-.1667em\lower.7ex\hbox{E}\kern-.125emX}}
\usepackage{balance}
\newcommand{\mytheoremname}{\bfseries Theorem}

\newcommand{\mylemmaname}{\bfseries Lemma}
\newcommand{\myassumptionname}{\bfseries Assumption}
\newcommand{\mydefinitionname}{\bfseries Definition}
\newcommand{\mypropositionname}{\bfseries Proposition}
\newcommand{\myclaimname}{\bfseries Claim}

\newtheorem{theorem}{\mytheoremname}

\newtheorem{Lemma}[theorem]{\mylemmaname} 
\newtheorem{Definition}{\mydefinitionname} 
\newtheorem{Proposition}{\mypropositionname} 
\renewcommand{\arraystretch}{1.2}

\begin{document}
\title{Hypergraph Neural Stochastic Diffusion: An SDE Framework for Uncertainty Estimation}
\author{ Zhiheng Zhou, Mengyao Zhou, Dengyi Zhao, Xingqin Qi, Guiying Yan
\thanks{ This work was supported by the National Natural Science Foundation of China (No.12231018 and No.12471330), the Shandong Provincial Natural
Science Foundation (No.ZR2025MS71) and the Postdoctoral Innovation Program of Shandong Province (No.SDCX-ZG-202603012). (Corresponding author:  Xingqin Qi, Guiying Yan.)}
\thanks{Zhiheng Zhou, Dengyi Zhao and Xingqin Qi are with the School of Mathematics and Statistics, Shandong University, Weihai, Shandong 264209, China (e-mail: zhouzhiheng@amss.ac.cn; zhaodengyi@mail.sdu.edu.cn; qixingqin@sdu.edu.cn).}
\thanks{Mengyao Zhou and Guiying Yan are with the Academy of Mathematics and Systems Science, Chinese Academy of Sciences and also with the University of Chinese Academy of Sciences, Beijing 100190, China (e-mail: zhoumengyao@amss.ac.cn; yangy@amss.ac.cn).}
}

\markboth{Journal of \LaTeX\ Class Files,~Vol.~18, No.~9, September~2020}%
{How to Use the IEEEtran \LaTeX \ Templates}\maketitle

This work has been submitted to the IEEE for possible publication. Copyright may be transferred without notice, after which this version may no longer be accessible.

\begin{abstract}
Hypergraph neural networks have shown powerful capability in modeling higher-order relations, yet their predictive uncertainty remains underexplored. Unlike pairwise graphs, uncertainty in hypergraphs arises not only from noisy attributes and ambiguous labels, but also from variations in node--hyperedge incidence structures and complex higher-order dependencies. Existing approaches mainly estimate uncertainty from final predictions or rely on computationally expensive ensembles and Bayesian inference, limiting their ability to capture uncertainty evolution during representation learning. In this paper, we propose Hypergraph Neural Stochastic Diffusion (HyperNSD), a stochastic differential equation framework for uncertainty estimation on hypergraphs. HyperNSD models hypergraph representations as stochastic processes evolving over node--hyperedge incidence structures. A learnable drift function captures deterministic higher-order diffusion dynamics, while a learnable stochastic forcing function characterizes structural ambiguity and representation noise. Predictive uncertainty is directly quantified through the variability of stochastic representation trajectories, providing an intrinsic uncertainty measure beyond post-hoc confidence scores. We formulate HyperNSD with neural drift and diffusion networks, enabling joint learning of prediction and uncertainty propagation. Theoretical analyses establish well-posedness, perturbation stability, permutation equivariance, and numerical convergence of the proposed stochastic dynamics. Experiments on multiple hypergraph benchmarks demonstrate that HyperNSD achieves reliable uncertainty estimation for out-of-distribution and misclassification detection while preserving competitive prediction accuracy. These results provide a principled stochastic-dynamical framework for trustworthy higher-order representation learning.
\end{abstract}


\section{Introduction}
Hypergraphs provide a natural representation for complex systems involving higher-order relations, where one interaction may simultaneously connect more than two entities. Unlike ordinary pairwise graphs, hypergraphs explicitly encode group-wise dependencies through hyperedges and are therefore well suited for modeling relational patterns in social networks, biological systems, recommendation, and brain network analysis~\cite{bretto2013hypergraph,kim2024survey,wang2025hypergraph,wang2026hypersynergyx}. Motivated by this expressive capability, hypergraph neural networks (HGNNs) have been developed to extend graph neural message passing from pairwise edges to node--hyperedge incidence structures~\cite{feng2019hypergraph,chien2021you,gao2022hgnn+,zhou2026hypergraph,zhou2026tackling}. By aggregating information between nodes and hyperedges, these models have achieved promising performance in various prediction tasks.

Despite their empirical success, the predictive reliability of existing HGNNs remains insufficiently explored. Most HGNNs are formulated as deterministic representation learning models: given a hypergraph and node attributes, they produce a single deterministic representation path through stacked propagation layers and a single predictive distribution~\cite{feng2019hypergraph,chien2021you,gao2022hgnn+}. Such deterministic behavior is problematic in high-stakes applications, where a model is expected not only to make accurate predictions but also to indicate when its predictions are unreliable~\cite{guo2017calibration,lakshminarayanan2017simple}. This issue is particularly important for hypergraph learning, where uncertainty may arise from attribute-level noise, label-space ambiguity, variability in node--hyperedge incidence patterns, and nontrivial dependencies among higher-order interactions. In these cases, an HGNN may produce overconfident predictions even when the underlying higher-order relational context is noisy, ambiguous, or shifted from the training distribution~\cite{wu2023energy,lin2024graph}.

Uncertainty estimation on hypergraphs presents additional challenges compared with ordinary graphs. Existing graph uncertainty methods mainly focus on uncertainty associated with node attributes, labels, or pairwise edges~\cite{wu2023energy,fuchsgruber2024energy,lin2024graph,xu2026uncertainty}. In hypergraphs, however, uncertainty is also coupled with the incidence structure between nodes and hyperedges. Since a hyperedge represents a group-wise relation, modifying one hyperedge can simultaneously change the aggregation contexts of all its incident nodes~\cite{feng2019hypergraph,chien2021you,kim2024survey}. Therefore, uncertainty estimation should not be restricted to output-level confidence scores; rather, stochastic representation variability should be modeled along node--hyperedge incidence relations during higher-order message passing. For instance, a distribution-shifted hyperedge structure may alter the representation dynamics of multiple nodes even when their input features remain unchanged. This observation suggests that uncertainty estimation for HGNNs should go beyond output-level confidence and explicitly characterize stochastic variability along node--hyperedge incidence structures.

Existing uncertainty estimation methods are not fully satisfactory for this purpose. Deep ensembles provide strong empirical uncertainty estimates but incur substantial computational and storage costs~\cite{lakshminarayanan2017simple}. Bayesian neural networks offer a principled probabilistic formulation, yet posterior inference or approximation is often expensive and difficult to scale to large relational structures~\cite{blundell2015weight}. Deterministic and post-hoc uncertainty methods avoid repeated training or posterior sampling, but they usually estimate uncertainty from final logits or predictive distributions, without modeling how uncertainty propagates during representation learning~\cite{guo2017calibration,stadler2021graph}. Recent graph uncertainty methods have made important progress by considering graph-dependent uncertainty, energy propagation, and stochastic graph diffusion~\cite{wu2023energy,fuchsgruber2024energy,lin2024graph,bergna2025uncertainty,wang2025gold,xu2026uncertainty}. However, these methods are primarily developed for pairwise graphs and do not directly address the higher-order incidence mechanism of hypergraphs. Moreover, when hypergraphs are approximated by pairwise graph expansions, the original incidence-level geometry may be distorted or partially lost, making it difficult to model stochastic perturbations directly on node--hyperedge interactions.

In this paper, we propose \emph{Hypergraph Neural Stochastic Diffusion} (HyperNSD), a stochastic differential equation (SDE) framework for uncertainty estimation on hypergraphs. HyperNSD formulates hypergraph representation learning as a stochastic diffusion process over the node--hyperedge incidence domain. Instead of producing a single deterministic representation trajectory, HyperNSD evolves node representations along multiple structure-informed stochastic paths. The resulting pathwise variability provides a representation-level basis for uncertainty estimation. Specifically, a learnable drift term models deterministic higher-order neural diffusion, while a learnable stochastic forcing term models data-dependent perturbations associated with representation noise and structural ambiguity. Since both terms are defined through the hypergraph gradient operator, the resulting stochastic representation dynamics are explicitly coupled with the node--hyperedge incidence structure. This design embeds uncertainty into the representation dynamics rather than appending it as a post-hoc confidence score, while avoiding multiple independently trained models, complex posterior approximation, or fixed-form output-level priors.

We further provide theoretical guarantees for the proposed HyperNSD, including well-posedness, moment stability, stability with respect to initial-state and structural perturbations, structure-constrained stochastic modes, permutation equivariance, and convergence of the Euler--Maruyama approximation. In practice, HyperNSD is instantiated with neural parameterizations of the drift and stochastic forcing terms, and uncertainty is estimated from the variability of sampled terminal representations and their corresponding predictive distributions. This enables HyperNSD to support out-of-distribution detection and misclassification detection under semantic, attribute-level, and structural distribution shifts. Extensive experiments on multiple hypergraph benchmarks show that HyperNSD consistently improves out-of-distribution (OOD) and misclassification detection across different distribution-shift settings while maintaining competitive in-distribution predictive accuracy. Additional analyses further demonstrate the importance of incidence-aware stochastic forcing, adaptive noise modulation, and direct modeling of the original hypergraph incidence domain.

The main contributions of this paper are summarized here:
\begin{itemize}
\item We identify incidence-coupled structural variability as a key challenge for uncertainty estimation on hypergraphs, where stochastic representation variability can be shaped by group-wise node--hyperedge interactions rather than captured solely by output-level confidence.

\item We propose \emph{Hypergraph Neural Stochastic Diffusion} (HyperNSD), an SDE-based framework that unifies deterministic higher-order neural diffusion and learnable stochastic forcing over the node--hyperedge incidence domain.

\item We establish theoretical guarantees for HyperNSD, including solution well-posedness, moment stability, stability under initial-state and structural perturbations, and numerical convergence. We further show that the proposed dynamics preserve structure-constrained stochastic modes and are permutation equivariant.

\item We conduct extensive experiments under label leave-out, feature interpolation, and structure manipulation settings, demonstrating improved OOD and misclassification detection. Additional ablation, graph-expansion, and perturbation-strength analyses further validate the necessity of incidence-aware stochastic forcing and direct hypergraph-domain modeling.
\end{itemize}

\section{Related Work}
\subsection{Hypergraph Neural Networks}
Hypergraph neural networks extend pairwise graph learning to higher-order relational structures by propagating information between nodes and hyperedges. Existing methods can be broadly categorized according to how they construct higher-order message-passing operators. Spectral methods such as HGNN perform convolution based on normalized hypergraph Laplacians, while graph-expansion methods such as HyperGCN approximate each hyperedge using representative pairwise connections~\cite{feng2019hypergraph,yadati2019hypergcn}. In contrast, node--hyperedge message-passing methods explicitly model interactions between nodes and hyperedges. For example, HNHN maintains and updates both node and hyperedge representations, AllSet formulates node--hyperedge aggregation through learnable multiset functions, and HGNN+ provides a flexible framework for general hypergraph message passing~\cite{dong2020hnhn,chien2021you,gao2022hgnn+,aponte2022hypergraph}. Unified formulations such as UniGNN further connect graph and hypergraph propagation rules, facilitating the adaptation of established graph architectures to higher-order relational data~\cite{huang2021unignn}.

Subsequent studies have improved the expressiveness, symmetry, and optimization properties of hypergraph learning. Equivariant architectures preserve permutation symmetries of hypergraph representations, while ED-HNN characterizes a broad class of continuous equivariant hypergraph diffusion operators~\cite{kim2022equivariant,wang2022equivariant}. Energy-based formulations derive hypergraph neural architectures from learnable regularized energy functions, providing an optimization-oriented interpretation of higher-order representation learning~\cite{wang2023hypergraph}. Recent analyses have also investigated hypergraph-specific homophily, architectural design, transferability, and information bottlenecks such as oversquashing~\cite{telyatnikov2023hypergraph,hayhoe2024transferable,yadati2025oversquashing}. Closely related to our formulation, diffusion-based models such as HNDiffN and HND interpret node--hyperedge message passing as continuous or discretized diffusion dynamics, while curvature-guided approaches adapt the diffusion process according to higher-order geometric structures~\cite{lu2025hypergraph,zhou2026hypergraph,zhou2026tackling}. Despite these advances, existing hypergraph models mainly focus on representation expressiveness, propagation depth, and predictive performance. They are mostly deterministic and do not explicitly model pathwise stochastic variability or uncertainty propagation over node--hyperedge incidence structures.

\subsection{Uncertainty Estimation on Graphs}
Predictive uncertainty estimation has been extensively studied through Bayesian neural networks, Monte Carlo dropout, deep ensembles, and confidence-based scoring methods~\cite{blundell2015weight,gal2016dropout,guo2017calibration,lakshminarayanan2017simple,hendrycks2016baseline,liang2017enhancing,lee2018simple,liu2020energy}. In OOD and misclassification detection, these methods typically construct uncertainty scores from logits, predictive entropy, energy values, or feature-space distances. Although effective in many settings, they are primarily designed for independently sampled data and do not explicitly characterize how stochastic variability is shaped by relational structures during representation learning.

For graph-structured data, recent studies incorporate graph dependencies into posterior inference, energy propagation, graph-dependent uncertainty modeling, and stochastic representation dynamics~\cite{stadler2021graph,wu2023energy,fuchsgruber2024energy,wang2025gold}. More closely related to our work, GNSD formulates graph representations as stochastic diffusion processes driven by a $\mathcal Q$-Wiener process, LGNSDE models latent graph dynamics using Brownian perturbations and parameter uncertainty, and structure-informed graph stochastic differential equations introduce spatially correlated perturbations into graph message passing~\cite{lin2024graph,bergna2025uncertainty,xu2026uncertainty}. Hypergraph OOD generalization methods such as HyperGOOD further consider distribution shifts in hypergraph learning, but they are not designed to estimate predictive uncertainty through stochastic incidence-domain dynamics.

Most existing graph uncertainty methods are constructed on pairwise graph operators and therefore do not directly capture the higher-order propagation geometry induced by node--hyperedge incidence relations. Pairwise graph expansions may also distort or partially lose the original incidence-level geometry of hypergraphs, making it difficult to model stochastic perturbations directly on node--hyperedge interactions. In contrast, HyperNSD defines both deterministic drift and stochastic forcing through incidence-based hypergraph operators. Node-space Wiener perturbations are first mapped to the incidence domain, adaptively modulated according to local node--hyperedge contexts, and then propagated back to the node space. This formulation embeds stochastic variability directly into higher-order representation dynamics and enables predictive uncertainty to be estimated from multiple structure-informed stochastic trajectories.

\section{Preliminaries}
\subsection{Notations}
Let $\mathcal{G}=(\mathcal{V},\mathcal{E},\omega)$ be a weighted hypergraph, where $\mathcal{V}$ is the node set with $|\mathcal{V}|=n$, $\mathcal{E}$ is the hyperedge set, and $\omega:\mathcal{E}\rightarrow \mathbb{R}_{>0}$ assigns a positive weight $\omega_e$ to each hyperedge $e\in\mathcal{E}$. Each hyperedge is a subset of nodes, i.e., $e\subseteq \mathcal{V}$ and $|e|\geq 2$. For each node \(v\in\mathcal{V}\), its degree is defined as $d_v=\sum_{e\in\mathcal{E}}\mathbb{I}(v\in e)$, where \(\mathbb{I}(\cdot)\) denotes the indicator function. The node--hyperedge incidence set is denoted by \(\mathcal{I}=\{(e,v)\mid e\in\mathcal{E}, v\in e\},\) with $|\mathcal{I}|=N=\sum_{e\in\mathcal{E}}|e|$.

The node degree matrix $\mathbf{D_v} \in \mathbb{R}^{n \times n}$ is defined by $\mathbf {D_v}({v,v})=d_v$. Let $\mathbf \Omega_{\mathcal I}=\operatorname{diag}(\omega_e)_{(e,v)\in\mathcal{I}}\in\mathbb{R}^{N\times N}$ be the incidence-level weight matrix, the input matrix of all nodes is denoted as $\mathbf{X}_{\text{init}}$, and  \(\|\cdot\|_F\) denotes the Frobenius norm. We introduce two finite-dimensional Hilbert spaces. The node function space is $L(\mathcal{V}) := \{\mathbf f: \mathcal{V} \to \mathbb{R}^d \} \cong \mathbb{R}^{n\times d}$ and the hyperedge--node pair function space is $L(\mathcal{E}, \mathcal{V}) := \{ \mathbf g: \mathcal{I} \to \mathbb{R}^d \} \cong \mathbb{R}^{N\times d}$, where $d$ denotes the embedding dimension.

\subsection{Hypergraph Neural Diffusion}
Following the hypergraph neural differential formulation in~\cite{zhou2026hypergraph}, we recall the deterministic diffusion operator used as the backbone of our model. Given a hypergraph $\mathcal{G}=(\mathcal{V},\mathcal{E},\omega)$, the gradient operator maps node signals to incidence-level variations between nodes and hyperedges, while the divergence operator aggregates incidence-level flows back to nodes. 

For a node function $\mathbf f\in L(\mathcal{V})$, its gradient over the incidence pair $(e,v)\in\mathcal{I}$ is defined as
\begin{equation}
(\nabla \mathbf f)(e,v)=   \frac{\mathbf f(v)}{\sqrt{d_v}} - \frac{1}{|e|} \sum_{u \in e} \frac{\mathbf f(u)}{\sqrt{d_u}} .
\label{eq:hypergraph_gradient}
\end{equation}
The corresponding divergence operator is defined as
\begin{equation}
\label{eq:divergence} 
(\operatorname{div} \mathbf g)(v) = \sum_{e \ni v}\frac{\omega_e}{\sqrt{d_v}}\left(  \mathbf g(e, v) - \frac{1}{|e|} \sum_{u \in e} \mathbf g(e, u)\right). 
\end{equation}

Based on these operators, the deterministic hypergraph neural diffusion equation can be written as
\begin{equation}
\frac{\partial \mathbf{X}(t)}{\partial t}=-\operatorname{div}
\left[\mathbf{A}_{\theta}(\mathbf{X}(t))\nabla \mathbf{X}(t)
\right],
\label{eq:deterministic_hypergraph_diffusion}
\end{equation}
where $\mathbf{X}(t)\in L(\mathcal{V})$ denotes the time-dependent node representation, and $\mathbf{A}_{\theta}(\mathbf{X}(t))\in\mathbb{R}^{N\times N}$ is a learnable diagonal modulation operator defined on the incidence space. Its diagonal entry associated with $(e,v)$ controls the diffusion strength between node $v$ and hyperedge $e$:
\begin{equation}
\mathbf{A}_{\theta}(\mathbf{X}(t))=\operatorname{diag}
\left(a_{\theta}(\mathbf{X}_v(t),\mathbf{X}_e(t))
\right)_{(e,v)\in\mathcal{I}}\in\mathbb{R}^{N\times N},
\label{eq:modulation_operator}
\end{equation}
where $\mathbf{X}_v(t)$ denotes the node representation of node $v$, and $\mathbf{X}_e(t)$ denotes an aggregation of node representations within hyperedge $e$. In matrix form, Eq.~\eqref{eq:deterministic_hypergraph_diffusion} becomes

\begin{equation}
\frac{\partial \mathbf{X}(t)}{\partial t}=-\mathbf G^{\top}\mathbf{A}_{\theta}(\mathbf{X}(t))\mathbf G
\mathbf{X}(t).
\label{eq:G_matrix_hypergraph_diffusion}
\end{equation}
Here, \(\mathbf G\in\mathbb R^{N\times n}\) encodes the node--hyperedge incidence structure, and \(\mathbf G^{\top}\in\mathbb R^{n\times N}\) maps the resulting flows back to the node domain.

\subsection{Uncertainty Source} Predictive uncertainty can be decomposed into aleatoric uncertainty (due to inherent data randomness) and epistemic uncertainty (due to uncertainty in model parameters)~\cite{gal2016dropout,kendall2017uncertainties,abdar2021review}. Given a training dataset $\mathcal{D}$ and a hypergraph $\mathcal{G}$, the predictive distribution of the target $\mathbf{y}$ is given by marginalizing over model parameters $\boldsymbol{\Theta}$: 
\begin{equation}
\label{eq:bayes_pred} 
\mathbb P(\mathbf{y} \mid \mathcal G,\mathcal{D}) = \int \mathbb P(\mathbf{y} \mid \boldsymbol{\Theta}, \mathcal G) \, \mathbb P(\boldsymbol{\Theta} \mid \mathcal{D}) \, \mathrm{d} \boldsymbol{\Theta}. \end{equation} 

The total uncertainty is measured by the entropy of the marginal predictive distribution: 

\begin{equation} 
\mathcal{H}\left[ \mathbb{E}_{\mathbb P(\boldsymbol{\Theta} \mid \mathcal{D})} \left[\mathbb P(\mathbf{y} \mid \boldsymbol{\Theta}, \mathcal{G}) \right] \right]. 
\end{equation} 
where $\mathcal H$ denotes the Shannon’s entropy of a probability distribution. Aleatoric uncertainty is commonly characterized by the posterior expectation of the conditional predictive entropy:
\begin{equation} 
\mathbb{E}_{\mathbb P(\boldsymbol{\Theta} \mid \mathcal{D})} \left[ \mathcal{H}\left[\mathbb P(\mathbf{y} \mid \boldsymbol{\Theta}, \mathcal{G}) \right] \right]. 
\end{equation} 
Epistemic uncertainty, reflecting uncertainty in model parameters, is the difference between the total and aleatoric uncertainties. It can also be interpreted as the mutual information between the prediction $\mathbf{y}$ and the parameters $\boldsymbol{\Theta}$~\cite{depeweg2018decomposition,zhao2020uncertainty}: \begin{equation}
\label{eq:mutual_info} 
\begin{aligned} 
I(\mathbf{y}, \boldsymbol{\Theta} \mid \mathcal{G}, \mathcal{D}) = \mathcal{H}\left[ \mathbb{E}_{\mathbb P(\boldsymbol{\Theta} \mid \mathcal{D})} \left[\mathbb P(\mathbf{y} \mid \boldsymbol{\Theta}, \mathcal{G}) \right] \right] \\- \mathbb{E}_{\mathbb P(\boldsymbol{\Theta} \mid \mathcal{D})} \left[ \mathcal{H}\left[\mathbb P(\mathbf{y} \mid \boldsymbol{\Theta}, \mathcal{G}) \right] \right]. 
\end{aligned} 
\end{equation}


\section{Methodology}
This section introduces the proposed Hypergraph Neural Stochastic Diffusion (HyperNSD) framework. We first formulate its incidence-aware deterministic drift and stochastic forcing, and then analyze its well-posedness, moment stability, stability under initial-state and structural perturbations, structure-constrained stochastic modes, permutation equivariance, and numerical convergence. Finally, we present the Euler--Maruyama discretization, uncertainty estimation strategy, training objective, and computational complexity.

\subsection{Hypergraph Neural Stochastic Diffusion}
\subsubsection{General Formulation}
The general form of the proposed hypergraph stochastic diffusion equation (HSDE) is defined as
\begin{equation}
\mathrm d\mathbf X(t)=\mathbf F(\mathbf X(t))\mathrm dt+\boldsymbol{\Sigma}(\mathbf X(t))\mathrm d\mathbf Z(t),
\label{eq:general_hsde}
\end{equation}
subject to an initial condition \(\mathbf{X}(0)\), where $\mathbf{X}(t)\in L(\mathcal{V})$ denotes the stochastic node state at time $t$. The drift field \(\mathbf F:L(\mathcal{V})\rightarrow L(\mathcal{V})\) governs the deterministic hypergraph diffusion dynamics. The stochastic diffusion operator \(\boldsymbol\Sigma(\mathbf{X}):L(\mathcal{E},\mathcal{V})\rightarrow L(\mathcal{V})\) determines how stochastic perturbations defined on node--hyperedge incidence relations are propagated to the node state space. The driving process \(\mathbf{Z}(t)\) is an \(L({\mathcal E,\mathcal V})\)-valued $\mathcal Q$-Wiener process~\cite{lord2014introduction}. A formal definition of the $\mathcal Q$-Wiener process is provided in Appendix~\ref{appendix:Proposition1}. The covariance operator \(
\mathcal  Q: L(\mathcal{E},\mathcal{V}) \rightarrow L(\mathcal{E},\mathcal{V}) \), specified below, characterizes the correlation structure of stochastic perturbations across node--hyperedge incidence relations.

\subsubsection{Hypergraph Neural Drift}
We employ the deterministic hypergraph neural diffusion equation~\cite{zhou2026hypergraph} introduced in Section~III as the drift component of the proposed stochastic dynamics. Specifically, the drift field $\mathbf F(\mathbf X(t))$ is defined as
\begin{equation}
\mathbf F^{\theta}(\mathbf X(t))=-{\mathbf G}^{\top} \mathbf A_{\theta}(\mathbf X(t)) \mathbf G \mathbf X(t),
\label{eq:hsde_drift}
\end{equation}
where $\mathbf{A}_{\theta}(\mathbf{X}(t))=\operatorname{diag}\big(a_{\theta}(\mathbf{X}_v(t),\mathbf{X}_e(t))\big)_{(e,v)\in\mathcal{I}}\in\mathbb{R}^{N\times N}$ is a learnable diagonal modulation matrix defined over the node--hyperedge incidence set $\mathcal I$.

For each incidence pair $(e,v)\in\mathcal I$, let $\mathbf{X}_e(t)=\operatorname{Agg}_u(\{\mathbf{X}_u(t) : u \in e\})$, where $\operatorname{Agg}$ is a permutation-invariant aggregator (e.g., mean or max). Then an unnormalized compatibility score is computed as
\begin{equation}
\begin{aligned}
&r^{\theta}_{(e,v)} = \sigma\left(\operatorname{MLP}_{\theta}(\mathbf{X}_v(t) \| \mathbf{X}_e(t))\right),\\
&a_\theta(\mathbf{X}_v(t), \mathbf{X}_e(t))= \frac{\exp\big(r^{\theta}_{(e,v)}\big)}{\sum_{e' \ni v} \exp\big( r^{\theta}_{(e',v)} \big)},
\end{aligned}
\end{equation}
where $\|$ denotes concatenation, and $\sigma$ is a $\operatorname{LeakyReLU}$ activation function. This formulation guarantees that $a_\theta(\mathbf{X}_v(t), \mathbf{X}_e(t)) > 0$ and \( \sum_{e \ni v} a_\theta(\mathbf{X}_v(t), \mathbf{X}_e(t)) = 1\). This parametrization enables the drift operator to assign incidence-specific drift modulation weights by comparing each node representation with the aggregated representation of its incident hyperedges. Consequently, the deterministic diffusion process adapts to the local node--hyperedge feature configuration while preserving bounded and normalized modulation coefficients.

\subsubsection{Hypergraph Neural Stochastic Forcing}
We introduce a $L(\mathcal{V})$-valued standard Wiener process
\begin{equation}
\mathbf W(t)=[\mathbf W_{1}(t),\ldots,\mathbf W_d(t)]\in L(\mathcal{V}),
\end{equation}
where \(\mathbf W_\ell(t)\) are mutually independent $n$-dimensional standard Wiener processes. Thus, \(\mathbf W(t)\) represents stochastic perturbations defined on hypergraph nodes.

To incorporate the hypergraph incidence geometry into the stochastic forcing, we transform the node-space Brownian motion through the hypergraph gradient:
\begin{equation}
\mathbf Z(t)=\mathbf G\mathbf W(t),
\label{eq:incidence_wiener}
\end{equation}
Since \(\mathbf G\in\mathbb R^{N\times n}\) acts independently on each feature dimension, \(\mathbf Z(t)\in L(\mathcal E,\mathcal V)\) defines a stochastic process over node--hyperedge incidence pairs. Although the node-level perturbations are mutually independent, their transformation through \(\mathbf G\) may induce correlations between different incidence pairs. Hence, the covariance structure of \(\mathbf Z(t)\) is explicitly determined by the hypergraph incidence geometry.

\begin{Proposition}
\label{P1}
Let $\mathbf W(t)$ be the node-space standard Wiener process defined above, and let \(\mathbf Z(t)=\mathbf G\mathbf W(t).\) Then $\mathbf Z(t)$ is an $L(\mathcal E,\mathcal V)$-valued $\mathcal Q$-Wiener process with covariance operator \(\mathcal Q =\mathbf G\mathbf G^{\top}\), where \(\mathcal Q\) is self-adjoint, nonnegative, and trace class and acts independently on each feature dimension. 
\end{Proposition}
The proof is provided in Appendix~\ref{appendix:Proposition1}.

We introduce a learnable diagonal operator to modulate the stochastic forcing over the node--hyperedge incidence space:
\begin{equation}
\mathbf B_{\phi}(\mathbf X(t))=\operatorname{diag}\left(b_{\phi}(\mathbf{X}_v(t),\mathbf{X}_e(t))\right)_{(e,v)\in\mathcal I}\in\mathbb R^{N\times N},
\label{eq:stochastic_modulation}
\end{equation}
where $\phi$ denotes a set of learnable parameters independent of the drift parameters $\theta$. For each incidence pair $(e,v)\in\mathcal I$, the coefficient $b_{\phi}(\mathbf X_v(t),\mathbf X_e(t))$ controls the relative amplitude of the stochastic perturbation assigned to that incidence relation.

We employ the same compatibility mechanism as in the drift module, while using an independent scalar-valued network. The resulting stochastic modulation coefficients are then normalized over all hyperedges incident to node $v$:
\begin{equation}
\begin{aligned}
&r^{\phi}_{(e,v)} = \sigma\left(\operatorname{MLP}_{\phi}(\mathbf{X}_v(t) \| \mathbf{X}_e(t))\right),\\
&b_{\phi}(\mathbf{X}_v(t), \mathbf{X}_e(t))=\frac{\exp\big( r^{\phi}_{(e,v)}\big)}{\sum_{e'\ni v}\exp\big( r^{\phi}_{(e',v)}\big)}.
\end{aligned}
\end{equation}
Hence, $\mathbf B_{\phi}(\mathbf X(t))$ adaptively redistributes the stochastic perturbation relative amplitude among the incidence relations associated with each node. In contrast to $\mathbf A_{\theta}$, which regulates deterministic feature transport, $\mathbf B_{\phi}$ controls how structure-informed stochastic perturbations are injected into the node dynamics.

The stochastic diffusion operator is then defined as
\begin{equation}
\boldsymbol{\Sigma}(\mathbf X(t))=\mathbf G^{\top}\mathbf B_\phi(\mathbf X(t)).
\label{eq:stochastic_operator}
\end{equation}
According to Eqs.~\eqref{eq:incidence_wiener} and~\eqref{eq:stochastic_operator}, the stochastic forcing term in Eq.~\eqref{eq:general_hsde} becomes
\begin{equation}
\boldsymbol{\Sigma}(\mathbf X(t))\mathrm d\mathbf Z(t)=\mathbf G^{\top}
\mathbf B_\phi(\mathbf X(t))\mathbf G\mathrm d\mathbf W(t).
\label{eq:stochastic_force}
\end{equation}
This construction consists of three operations. First, the node-space Brownian motion is transformed into incidence-space noise through \(\mathbf G\). Second, the learnable operator \(\mathbf B_\phi(\mathbf X(t))\)
adaptively redistributes the relative stochastic amplitude. Finally,
\(\mathbf G^{\top}\) aggregates the perturbed incidence fluxes back to the node domain.

\subsubsection{Structure-Informed HyperNSD}
\begin{figure*}[t]
\centering
\includegraphics[scale=0.8]{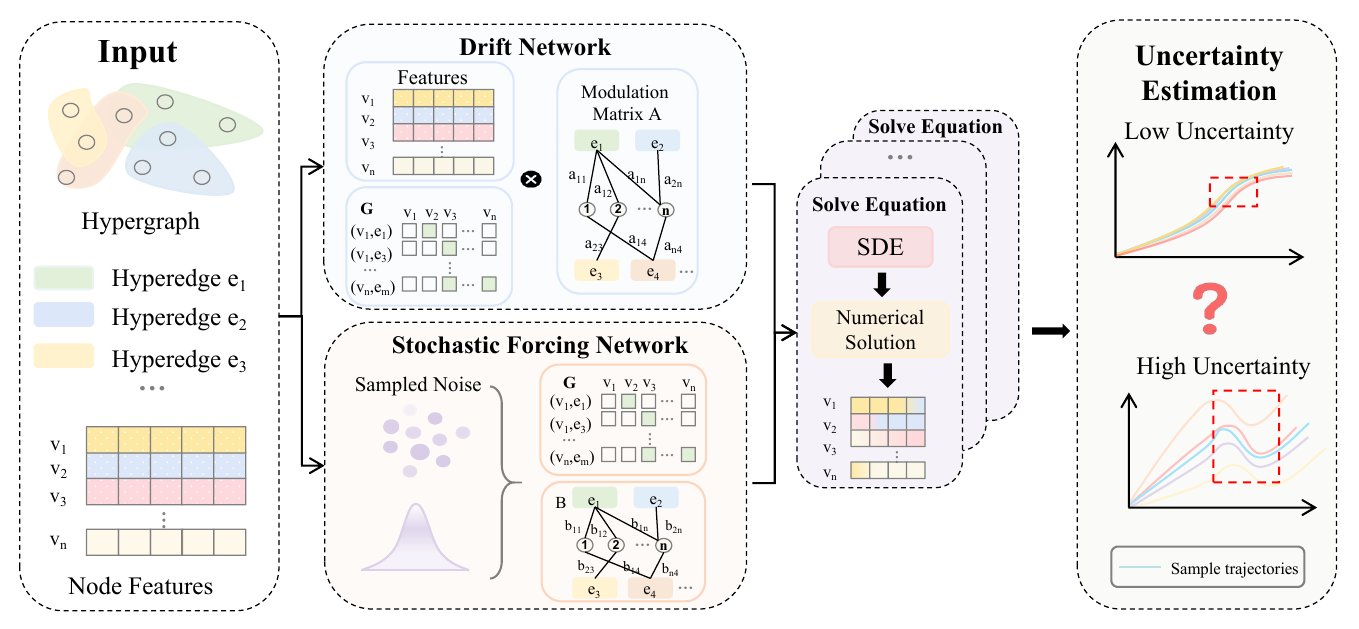}
 \caption{Overview of the proposed HyperNSD framework. The drift and stochastic forcing networks jointly model deterministic hypergraph diffusion and structure-informed perturbations, while repeated stochastic numerical realizations generate trajectories for uncertainty estimation.}
\label{Framework}
\end{figure*}
Combining the deterministic drift in Eq.~\eqref{eq:hsde_drift} with the stochastic forcing in Eq.~\eqref{eq:stochastic_force}, we obtain the proposed structure-informed HyperNSD:
\begin{equation}
\mathrm d\mathbf X(t)=-\mathbf G^{\top}\mathbf A_\theta(\mathbf X(t))\mathbf G\mathbf X(t)\mathrm dt+\mathbf G^{\top} \mathbf B_\phi(\mathbf X(t))\mathbf G\mathrm d\mathbf W(t),
\label{eq:node_driven_hsde}
\end{equation}
with the initial condition $\mathbf{X}(0)=\zeta(\mathbf X_{\text{init}})\in\mathbb R^{n\times d}$, where $\zeta$ is a learnable input encoder. To simplify the notation and subsequent analysis, we denote the corresponding effective node-space diffusion operator by
\(\boldsymbol{\Sigma}^{\phi}(\mathbf X(t))=\mathbf G^{\top}\mathbf B_{\phi}(\mathbf X(t))\mathbf G.\)

The drift term in Eq.~\eqref{eq:node_driven_hsde} governs deterministic higher-order feature propagation, whereas the stochastic term injects adaptively modulated, structure-dependent stochastic perturbations through the same node--hyperedge incidence geometry. Therefore, both deterministic diffusion and stochastic representation variability are consistently defined over the hypergraph incidence domain.

\subsection{Theoretical Analysis}
By construction, the neural drift $\mathbf F^{\theta}$ and diffusion coefficient $\boldsymbol{\Sigma}^{\phi}$ combine the fixed hypergraph gradient operators with softmax-normalized, state-dependent modulation matrices. Consequently, they remain bounded in their modulation strengths and inherit local regularity from the neural compatibility functions. Their local Lipschitz continuity and linear-growth properties are formalized in the following proposition.

\begin{Proposition}
\label{prop:coefficient_regularity}
For every $R>0$, there exists a constant $L_R>0$ such that
\begin{equation}
\begin{aligned}
\|\mathbf F^{\theta}(\mathbf X)-\mathbf F^{\theta}(\mathbf Y)\|_F^2
&+\|\boldsymbol{\Sigma}^{\phi}(\mathbf X)-\boldsymbol{\Sigma}^{\phi}(\mathbf Y)\|_{F}^2\\
&\leq L_R\|\mathbf X-\mathbf Y\|_F^2
\end{aligned}
\label{eq:neural_coefficients_local_lipschitz}
\end{equation}
for all $\mathbf X,\mathbf Y\in L(\mathcal V)$ satisfying \(\|\mathbf X\|_F\leq R,\;\|\mathbf Y\|_F\leq R.\)

Moreover, there exists a constant $C>0$, such that
\begin{equation}
\|\mathbf F^{\theta}(\mathbf X)\|_F^2+\|\boldsymbol{\Sigma}^{\phi}(\mathbf X)\|_{F}^2\leq C\left(1+\|\mathbf X\|_F^2\right).
\label{eq:neural_coefficients_linear_growth}
\end{equation}
Therefore, $\mathbf F^{\theta}$ and $\boldsymbol{\Sigma}^{\phi}$ are locally Lipschitz continuous and satisfy the linear-growth condition.
\end{Proposition}
The proof is provided in Appendix~\ref{appendix:Proposition2}.

We next establish the well-posedness and stability properties of Eq.~\eqref{eq:node_driven_hsde}. Proposition~\ref{prop:coefficient_regularity} shows that the drift and diffusion coefficients satisfy the local regularity conditions required for local existence and uniqueness. To extend the local solution to a global one, we further exploit the dissipative structure of the drift term together with the uniform boundedness of the diffusion coefficient. These properties prevent finite-time explosion and yield explicit moment, energy, and stopped stability estimates, as summarized in the following theorem.

\begin{theorem}
\label{thm:wellposedness_energy_stability}
Assume that $\mathbf X(0)$ is $\mathcal F_0$-measurable and square integrable. Then the following statements hold. 

\textnormal{(i)} Equation~\eqref{eq:node_driven_hsde} admits a unique global strong solution $\mathbf X(t)$ on every finite interval $[0,T]$. Moreover,
\begin{equation}
\mathbb E\left[\sup_{0\leq t\leq T}\|\mathbf X(t)\|_F^2\right]\leq C_T\left(1+\mathbb E\|\mathbf X(0)\|_F^2\right).
\end{equation}

\textnormal{(ii)} For every $t\in[0,T]$,
\begin{equation}
\begin{aligned}
\mathbb E\|\mathbf X(t)\|_F^2
&+2\mathbb E\int_0^t\|\mathbf A_\theta(\mathbf X(s))^{1/2}\mathbf G\mathbf X(s)\|_F^2\mathrm ds\\
&=\mathbb E\|\mathbf X(0)\|_F^2+d\mathbb E\int_0^t\|\boldsymbol{\Sigma}^{\phi}(\mathbf X(s))\|_F^2\mathrm ds.
\end{aligned}
\end{equation}
This implies that
\begin{equation}
\mathbb E\|\mathbf X(t)\|_F^2\leq\mathbb E\|\mathbf X(0)\|_F^2+dt\|\mathbf G\|_F^4.
\end{equation}

\textnormal{(iii)} Let $\mathbf X(t)$ and $\mathbf Y(t)$ be two solutions driven by the same node-space Wiener process, with initial conditions $\mathbf X(0)$ and $\mathbf Y(0)$. For
\begin{equation}
\rho_R=\inf\{t\geq0:\|\mathbf X(t)\|_F\vee\|\mathbf Y(t)\|_F\geq R\},
\end{equation}
there exists $C_{R,T}>0$ such that
\begin{equation}
\begin{aligned}
\mathbb E[\sup_{0\leq t\leq T}\|\mathbf X(t\wedge\rho_R)-\mathbf Y(t\wedge\rho_R)\|_F^2]\\
\leq C_{R,T}\mathbb E[
\|\mathbf X(0)-\mathbf Y(0)\|_F^2].
\end{aligned}
\end{equation}
\end{theorem}
The proof of Theorem~1 is provided in Appendix~\ref{appendix:theorem1}.

Theorem~\ref{thm:wellposedness_energy_stability} establishes that the proposed stochastic dynamics are globally well defined, energetically controlled, and locally stable with respect to perturbations of the initial state before the exit time. Beyond initial-state stability, we further examine how the HyperNSD solution changes when the underlying hypergraph structure is perturbed. The following theorem establishes a local mean-square stability result with respect to aligned perturbations of the hypergraph gradient operator.

\begin{theorem}
\label{thm:G}
Let \(\mathbf X^{\mathbf G}(t)\) and
\(\mathbf X^{\mathbf{G'}}(t)\)
be the solutions of two HyperNSD systems associated with the hypergraph gradient matrices \(\mathbf G,\mathbf G'\in\mathbb R^{N\times n}\), represented in a consistently aligned incidence coordinate system, and satisfying the same initial condition. For \(R>0\), define the stopping time
\begin{equation}
\tau_R=\inf\{t\ge0:\|\mathbf X^{\mathbf G}(t)\|_{F}\vee\|\mathbf X^{\mathbf G'}(t)\|_{F}\ge R\}.
\end{equation}

If there exist constants \(L_{A,R}>0\) and \(L_{B,R}>0\) such that, for every \(\mathbf X\) satisfying \(\|\mathbf X\|_{F}\le R\),
\begin{equation}
\begin{aligned}
\|\mathbf A_{\theta}^\mathbf G(\mathbf X)
-\mathbf A_{\theta}^{\mathbf G'}(\mathbf X)\|_{F}\le L_{A,R}\|\mathbf G-\mathbf G'\|_{F},\\
\|\mathbf B_{\phi}^{\mathbf G}(\mathbf X)
-\mathbf B_{\phi}^{\mathbf G'}(\mathbf X)
\|_{F}\le L_{B,R}\|\mathbf G-\mathbf G'\|_{F}.
\end{aligned}
\end{equation}

Then, for every finite \(T>0\), there exists a constant \(C^G_{R,T}>0\), independent of \(\mathbf G-\mathbf G'\), such that
\begin{equation}
\mathbb E[\sup_{0\le t\le T}\|\mathbf X^{\mathbf G}(t\wedge\tau_R)-\mathbf X^{\mathbf G'}(t\wedge\tau_R)\|_{F}^{2}
]\le C^G_{R,T}\|\mathbf G-\mathbf G'\|_{F}^{2}.
\end{equation}
\end{theorem}
The proof of Theorem~\ref{thm:G}
 is provided in Appendix~\ref{appendix:G}.

In addition to stability under structural perturbations, the factorization of both the drift and diffusion operators through the hypergraph gradient \(\mathbf G\) also imposes intrinsic geometric constraints on the stochastic dynamics. In particular, the null-space components are preserved, while stochastic variability is confined to the nontrivial structural modes induced by the hypergraph. These structure-constrained modes are characterized in the following theorem.

\begin{theorem}
\label{thm:conservation_covariance}
Let $\operatorname{Proj}$ denote the orthogonal projection onto $\ker \mathbf G$. Then the following properties hold.

\textnormal{(i) Conservation of null-space modes.} For every $t\geq0$,
\begin{equation}
\operatorname{Proj}\mathbf X(t)=\operatorname{Proj}\mathbf X(0)
\end{equation}
almost surely.

\textnormal{(ii) Structure-constrained instantaneous covariance.}
For each feature dimension $\ell$,
\begin{equation}
\operatorname{Cov}
\left(\mathrm d\mathbf X_{\ell}(t)\mid\mathbf X(t)\right)=\boldsymbol{\Sigma}^{\phi}(\mathbf X(t))^2\mathrm dt.
\end{equation}

\textnormal{(iii) Equality of structural modes.}
Since the diagonal matrix
$\mathbf B_{\phi}(\mathbf X)$
is positive definite, it follows that
\begin{equation}
\ker\boldsymbol{\Sigma}^{\phi}(\mathbf X)=\ker \mathbf G,\quad\operatorname{rank}\boldsymbol{\Sigma}^{\phi}(\mathbf X)=\operatorname{rank}\mathbf G.
\end{equation}
If the hypergraph is connected and
$\ker \mathbf G=\operatorname{span}\{\mathbf {D_v}^{1/2}\mathbf 1_n\}$, where $\mathbf 1_n=(1,\ldots,1)^\top\in\mathbb{R}^n$, then
\begin{equation}
\left(\mathbf {D_v}^{1/2}\mathbf 1_n\right)^{\top}\mathbf X(t)=\left(\mathbf {D_v}^{1/2}\mathbf 1_n\right)^{\top}\mathbf X(0)
\end{equation}
almost surely.
\end{theorem}
The proof of Theorem~\ref{thm:conservation_covariance} is provided in Appendix~\ref{appendix:theorem2}.

Theorem~\ref{thm:conservation_covariance} characterizes how the hypergraph gradient constrains both the conserved modes and the directions along which stochastic variability can propagate. A complementary requirement is that these structural properties should not depend on the arbitrary ordering of the nodes or incidence pairs. We therefore next establish the permutation equivariance of the HyperNSD dynamics under consistent relabeling of the hypergraph.

Let $\pi$ be an arbitrary permutation of the node set, represented by a permutation matrix $\Pi\in\mathbb R^{n\times n}$. The permutation $\pi$ induces a permutation matrix $\mathbf S\in\mathbb R^{N\times N}$ on the node--hyperedge incidence pairs, mapping $(e,v)$ to $(\pi(e),\pi(v))$. Let $\mathcal G^{\pi}$ denote the relabeled hypergraph. Its hypergraph gradient matrix satisfies
\begin{equation}
\mathbf  G^{\pi}=\mathbf  S\mathbf  G\Pi^{\top}.
\label{eq:permuted_gradient_operator}
\end{equation}

Since the hyperedge aggregation operator is permutation invariant and the compatibility networks are shared across all incidence pairs, the deterministic and stochastic modulation matrices transform as
\begin{equation}
\mathbf A_{\theta}^{\pi}(\Pi\mathbf X)=\mathbf S\mathbf A_{\theta}(\mathbf X)\mathbf S^{\top},\quad\mathbf B_{\phi}^{\pi}(\Pi\mathbf X)=\mathbf S\mathbf B_{\phi}(\mathbf X)\mathbf  S^{\top}.
\label{eq:permuted_modulation_matrices}
\end{equation}
As a result,
\begin{equation}
\mathbf F_{\mathcal G^{\pi}}^{\theta}(\Pi\mathbf X)=\Pi\mathbf F_{\mathcal G}^{\theta}(\mathbf X),\quad\boldsymbol{\Sigma}_{\mathcal G^{\pi}}^{\phi}(\Pi\mathbf X)=\Pi\boldsymbol{\Sigma}_{\mathcal G}^{\phi}(\mathbf X)\Pi^{\top}.
\label{eq:permuted_HNSD_coefficients}
\end{equation}

\begin{Proposition}
\label{prop:permutation_equivariance}
Let $\mathbf X_{\mathcal G}(t;\mathbf X(0),\mathbf W)$ denote the solution of the HyperNSD equation on $\mathcal G$, initialized at $\mathbf X(0)$ and driven by the node-space Wiener process $\mathbf W(t)$. Define
\begin{equation}
\mathbf W^{\pi}(t)=\Pi\mathbf W(t).
\end{equation}
Then $\mathbf W^{\pi}(t)$ is also a node-space standard Wiener process, and
\begin{equation}
\mathbf X_{\mathcal G^{\pi}}\left(t;\Pi\mathbf X(0),\mathbf W^{\pi}
\right)=\Pi\mathbf X_{\mathcal G}\left(t;\mathbf X(0),\mathbf W\right)
\label{eq:pathwise_permutation_equivariance}
\end{equation}
for every $t\geq0$, almost surely. Equivalently, when the two systems are driven by standard Wiener processes with the same distribution,
\begin{equation}
\mathbf X_{\mathcal G^{\pi}}\left(t;\Pi\mathbf X(0)\right)\overset{\mathrm d}{=}\Pi\mathbf X_{\mathcal G}\left(t;\mathbf X(0)\right).
\label{eq:distributional_permutation_equivariance}
\end{equation}
\end{Proposition}
The proof is provided in Appendix~\ref{appendix:Proposition3}.

\paragraph*{Remark}
The above theoretical results provide several implications for the proposed HyperNSD as a neural network architecture. First, the local Lipschitz and linear-growth properties of the neural drift and diffusion coefficients ensure that the continuous-depth stochastic representation dynamics are well defined, rather than being merely a formal SDE construction. Second, the moment and energy estimates show that the stochastic feature evolution remains controlled over finite diffusion horizons, which is important for stable forward propagation and numerical training. Third, the stopped stability estimates with respect to initial-state and hypergraph structural perturbations indicate that small changes in input representations or incidence structures lead to controlled changes in stochastic trajectories before leaving a bounded region. This provides a theoretical explanation for the robustness of HyperNSD under feature and structural distribution shifts. Fourth, the structure-constrained mode analysis shows that the stochastic forcing does not introduce arbitrary perturbations in the node space; instead, stochastic variability is restricted to the nontrivial structural subspace determined by the hypergraph gradient operator. Finally, permutation equivariance guarantees that the HyperNSD outputs are independent of arbitrary node or incidence ordering, which is a basic requirement for hypergraph neural networks. Therefore, these results connect the stochastic diffusion formulation with standard desiderata in neural network learning, including stable representation propagation, structure-aware robustness, geometry-preserving stochasticity, and permutation-consistent prediction.

\subsection{Discretization of HyperNSD}
Although HyperNSD is formulated as a continuous-time SDE, its practical implementation is a finite-depth neural architecture. Specifically, each HyperNSD layer corresponds to one Euler--Maruyama step~\cite{platen1999introduction}, and an $L$-layer HyperNSD performs $L$ successive numerical updates over the diffusion horizon $[0,T]$.

We discretize the diffusion interval $[0,T]$ into $L$ uniform subintervals,
\begin{equation}
0=t_0<t_1<\cdots<t_L=T,\quad t_k=kh,\quad h=\frac{T}{L},
\end{equation}
where $L$ denotes the number of HyperNSD layers and $h$ is the step size associated with each layer. The Wiener increment corresponding to the $k$-th layer is defined as
\begin{equation}
\Delta\mathbf W_k=\mathbf W(t_{k+1})-\mathbf W(t_k)\in\mathbb R^{n\times d}.
\label{eq:wiener_increment}
\end{equation}

Since the feature dimensions of $\mathbf W(t)$ are mutually independent, the columns of $\Delta\mathbf W_k$ are independent Gaussian random vectors satisfying
\begin{equation}
\Delta\mathbf W_{k,j}\sim\mathcal N\left(0,h\mathbf I_n\right),\qquad j=1,\ldots,d,
\end{equation}
where $\mathbf I_n$ denotes the identity matrix.

Let $\mathbf X^{(k)}$ denote the node representation after the $k$-th HyperNSD layer, with \(\mathbf X^{(0)}=\mathbf X(0).\) Applying the Euler--Maruyama method to Eq.~\eqref{eq:node_driven_hsde} gives the layer-wise update
\begin{equation}
\begin{aligned}
\mathbf X^{(k+1)}=\mathbf X^{(k)}
&-h\mathbf G^{\top}\mathbf A_{\theta}\left(\mathbf X^{(k)}\right)\mathbf G\mathbf X^{(k)}\\
&+\mathbf G^{\top}\mathbf B_{\phi}\left(\mathbf X^{(k)}\right)\mathbf G\Delta\mathbf W_k,
\end{aligned}
\label{eq:layerwise_euler_HNSD}
\end{equation}
for $k=0,\ldots,L-1$. Therefore, each HyperNSD layer consists of a deterministic diffusion residual and a stochastic structure-informed residual. For a fixed diffusion horizon $T$, increasing the number of layers $L$ decreases the step size $h$ and produces a finer approximation of the continuous HyperNSD dynamics. 

For the convergence analysis, let \(\mathbf X_k^h=\mathbf X^{(k)}\) and define the continuous-time Euler--Maruyama interpolation
\begin{equation}
\begin{aligned}
\overline{\mathbf X}^{h}(t)={}
&\mathbf X(0)+\int_0^t\mathbf F^{\theta}\left(\mathbf X_{\nu_s^h}^h\right)
\mathrm ds+\int_0^t\boldsymbol{\Sigma}^{\phi}\left(\mathbf X_{\nu_s^h}^h
\right)\mathrm d\mathbf W(s),
\end{aligned}
\label{eq:continuous_em_interpolation}
\end{equation}
where \(\nu^h_s=\max\{k:t_k\leq s\}.\) At each grid point, we have \(
\overline{\mathbf X}^{h}(t_k)=\mathbf X_k^h=\mathbf X^{(k)}.
\)

\begin{Proposition}
\label{prop:euler_maruyama_convergence}
Let $\mathbf X(t)$ be the unique global strong solution of Eq.~\eqref{eq:node_driven_hsde}. For $R>0$, define
\begin{equation}
\eta_R^h=\inf\{t\geq0:\|\mathbf X(t)\|_F\vee\|\overline{\mathbf X}^{h}(t)\|_F\geq R\}.
\end{equation}
Then, for every finite $T>0$, there exists a constant $C^E_{R,T}>0$, independent of $h$, such that
\begin{equation}
\mathbb E\left[\sup_{0\leq t\leq T\wedge\eta_R^h}\|\mathbf X(t)-\overline{\mathbf X}^{h}(t)\|_F^2\right]\leq C^E_{R,T}h.
\label{eq:stopped_em_convergence}
\end{equation}
This implies that
\begin{equation}
\sup_{0\leq t\leq T}\|\mathbf X(t)-\overline{\mathbf X}^{h}(t)\|_F\longrightarrow0
\end{equation}
in probability as $h\rightarrow0$.
\end{Proposition}
The proof is provided in Appendix~\ref{appendix:P3}.

\subsection{Loss Function, Uncertainty Estimation, and Complexity}
\subsubsection{Training Objective}
Let \(\Theta\) denote all trainable parameters of HyperNSD. We adopt the distributional uncertainty
loss~\cite{bilovs2019uncertainty}
\begin{equation}
\mathcal L=\mathbb E_{\hat{\mathbb P}(X(T))}
\left[\mathcal H[\mathbb P(y^{*}|\mathbf X,\mathcal G),\mathbb P(y|
\mathbf X(T))]\right],
\label{eq:expected_pathwise_loss}
\end{equation}
where $\hat{\mathbb P}(\mathbf  X(T))=\hat{\mathbb P}(\mathbf X(T)|\mathbf X(0),\Theta,\mathcal G)$ denotes the conditional distribution of the terminal node representations induced by the stochastic diffusion dynamics. $\mathbb P(y^{*}|\mathbf X,\mathcal G)$ denotes the true class distribution and $\mathbb P(y|\mathbf X(T))$ is the conditional predictive distribution produced by the output decoder $\xi$. Rather than optimizing the cross-entropy associated with a single deterministic representation, the proposed objective minimizes the expected path-conditioned cross-entropy over the distribution of stochastic terminal representations. It therefore accounts for predictive variability induced by the stochastic representation dynamics during training.

\subsubsection{Uncertainty Estimation}
After training, HyperNSD generates $M$ stochastic samples from $\hat{\mathbb P}(\mathbf X(T))$ to estimate both aleatoric and epistemic uncertainties. Each sample corresponds to an independent numerical realization of the proposed hypergraph stochastic diffusion equation and produces a distinct terminal node representation $\mathbf X(T)$. In this framework, the aleatoric component is estimated by $\mathbb E_{\hat{\mathbb P}(X(T))} \left[\mathcal H[\mathbb P(y| \mathbf X(T))]\right]$, while the trajectory-induced epistemic component can be calculated by the variance of the final solution $\mathbf X(T)$. To obtain structure-aware uncertainty estimates, HyperNSD jointly incorporates node attributes and node--hyperedge incidence information into the deterministic drift and stochastic forcing networks. The stochastic forcing term maps node-space Wiener perturbations to the incidence domain, adaptively modulates them according to local node--hyperedge contexts, and propagates them back to the node space, thereby capturing both representation variability and higher-order structural dependencies. This sampling-and-estimation procedure follows the standard SDE-based uncertainty estimation paradigm~\cite{lin2024graph,kong2020sde} and resembles ensemble inference; however, HyperNSD requires only a single trained model, reducing training and storage costs while generating multiple stochastic predictions through repeated forward trajectories. A detailed explanation of how predictive uncertainty is computed and used for the detection tasks is provided in Appendix~\ref{usingUD}.

\subsubsection{Complexity Analysis}
Let $n$ denote the number of nodes, $N=|\mathcal I|$ the number of node--hyperedge incidence pairs, and $d$ the feature dimension.

Under fixed-width drift and diffusion coefficient networks, computing $\mathbf F^{\theta}$ and $\boldsymbol{\Sigma}^{\phi}$ requires $\mathcal O(Nd)$ operations. Sampling the Gaussian increment $\Delta\mathbf W$ requires $\mathcal O(nd)$ operations. The per-layer complexity of one stochastic trajectory is therefore \(\mathcal O(Nd).\) A complete stochastic trajectory contains $L$ Euler--Maruyama layers. Hence, the computational complexity of one forward trajectory is \(
\mathcal O(LNd).\) HyperNSD generates $M$ independent stochastic trajectories. Consequently, the total forward-pass complexity is \(\mathcal O(MLNd),\) excluding the task-specific readout and decoder. The empirical runtime comparison is given in Appendix~\ref{run}.

\section{Experiments}
\label{experiments}
This section evaluates HyperNSD on six hypergraph benchmarks under label leave-out, feature interpolation, and structure manipulation settings. We compare HyperNSD with representative uncertainty estimation and hypergraph learning methods using \textbf{out-of-distribution (OOD) detection} and \textbf{misclassification detection}. We further conduct ablation studies, graph-expansion comparisons, and uncertainty visualizations to examine the effectiveness of the incidence-aware stochastic diffusion mechanism. The source code will be available at https://github.com/CASZhouzhiheng/HyperNSD.

\begin{table}[htbp]
  \centering
  \caption{Dataset Statistics Summary\label{tab:dataset_stats}}
  \setlength{\tabcolsep}{3pt}
  \begin{tabular}{l *{7}{S[table-format=5.0]}}
    \toprule
    {\textbf{Metric}} & {\textbf{Cora}} & {\textbf{Cora-CA}} &{\textbf{Citeseer}} &   {\textbf{DBLP}}   & {\textbf{ModelNet40}} & {\textbf{NTU2012}}  \\
    \midrule
    $|\mathcal{V}|$ & 2708 & 2708  & 3312 & 41302  & 12311 & 2012  \\
    $|\mathcal{E}|$ & 1579 & 1072  & 1079 & 22363  & 24622 & 4024  \\
    $|\mathcal{I}|$ & 4786 & 5428 & 3459 & 99561 & 61555 & 10060\\
    \#features
    & 1433 & 1433 & 3703 & 1425  & 100  & 100   \\
    \#classes & 7 & 7 & 6 & 6 & 40 & 67 \\
    \bottomrule
  \end{tabular}
\end{table}
\subsection{Experimental Setup}
\paragraph{Datasets} We evaluate our model on six benchmark hypergraph datasets including four citation datasets (Cora, Cora-CA, CiteSeer and DBLP)~\cite{sen2008collective,yadati2019hypergcn} and two real--world datasets (ModelNet40, NTU2012)~\cite{wu20153d,chen2003visual}. For each dataset, we randomly divide the samples into training, validation, and test sets using a 50\%/25\%/25\% split. All experiments are independently repeated ten times with different random seeds, and we report the average performance over the ten runs. Detailed statistics of the evaluated datasets are summarized in Table~\ref{tab:dataset_stats}. We provide detailed dataset information in Appendix~\ref{dataset}.

\paragraph{OOD Shift Construction}
We synthetically generate OOD samples using three types of distribution shifts.
\begin{itemize}
    \item \textbf{Label leave-out}. We construct semantic shifts by retaining a subset of classes as in-distribution data and treating the remaining classes as OOD samples. Specifically, samples with class labels greater than 3 are treated as OOD for Cora, Cora-CA, and DBLP; samples with class labels greater than 2 are treated as OOD for Citeseer; samples with class labels greater than 20 are treated as OOD for ModelNet40; and samples with class labels greater than 33 are treated as OOD for NTU2012. All remaining samples are used as ID data.
    \item \textbf{Feature interpolation}. We construct attribute-level OOD samples by randomly interpolating the node features of selected target samples with features drawn from other nodes. The original hypergraph data are retained as ID data, whereas the interpolated hypergraph data are regarded as OOD data. The hypergraph incidence structure remains unchanged.
    \item \textbf{Structure manipulation}. We generate structural OOD samples through degree-preserving incidence rewiring. Specifically, two existing incidences satisfying
\[
H_{u,e_1}=1,\quad H_{v,e_2}=1,\mbox{ and }
H_{u,e_2}=0,\quad H_{v,e_1}=0
\]
are randomly selected. We then replace \((u,e_1)\) and \((v,e_2)\) with \((u,e_2)\) and \((v,e_1)\), respectively. This operation preserves the degrees of nodes \(u\) and \(v\), as well as the cardinalities of hyperedges \(e_1\) and \(e_2\), while altering the local node--hyperedge associations. The rewiring ratio $\gamma$ is set to 0.5 in all experiments. The original hypergraph data are retained as ID data, whereas the incidence rewiring hypergraph data are regarded as OOD data. The hypergraph node features remain unchanged.
\end{itemize}

\begin{table*}[ht]
\belowrulesep=0pt
\aboverulesep=0pt
\renewcommand{\arraystretch}{1.5}
\caption{OOD detection performance measured by AUROC on all datasets with three OOD types (Label leave-out, Feature interpolation and Structure manipulation). Other results for AUPR, FPR95, in-distribution accuracy and standard deviations of AUROC are deferred to Appendix~\ref{result}.}
\label{table:ood_result}
\centering
\begin{adjustbox}{width=\textwidth}
\begin{tabular}{|c|ccc|ccc|ccc|ccc|ccc|ccc|}
\toprule
\multirow{3}{*}{Model}
& \multicolumn{3}{c|}{Cora} & \multicolumn{3}{c|}{Cora-CA} &  \multicolumn{3}{c|}{Citeseer} & \multicolumn{3}{c|}{DBLP}&\multicolumn{3}{c|}{ModelNet40}&\multicolumn{3}{c|}{NTU2012}\\ 
& L & F & S & L & F & S & L & F & S & L & F & S & L & F & S & L & F & S\\
\midrule  

MSP & 78.95 & 67.89 & 60.25 & 71.57 & 75.82 & 58.42 & 71.82 & 70.86 & 56.12 & 94.09 & 89.44 & 61.05 & 86.85 & 56.12 & 49.66 & 80.17 & 54.75 & 50.37\\          

ODIN & 44.15 & 52.78 & 39.75 & 48.51 & 51.58 & 41.58 & 51.19 & 45.68 & 43.17 & 61.16 & 77.93 & 53.13 & 72.71 & 50.21 & 56.87 & 51.49 & 50.84 & 53.11\\

Mahalanobis & 48.41 & 65.95 & 65.27 & 48.29 & 75.06 & 58.81 & 58.06 & 68.93 & 57.33 & 87.81 & 66.87 & 54.25 & 60.62 & 64.25 & 60.21 & 45.49 & 61.87 & 57.07\\
\midrule

GNNSafe & 73.66 & 71.03 & 63.03 & 65.23 & 80.03 & 61.15 & 67.55 & 69.12 & \cellcolor[RGB]{255,223,175}\underline{74.82} & \cellcolor[RGB]{255,190,110}{\textbf{96.08}} & 92.29 & 62.57 & 89.67 & 81.31 & 63.58 & 83.80 & 79.75 & 65.90 \\

GPN & 71.27 & 70.72 & 64.98 & 68.82 & 78.53 & \cellcolor[RGB]{255,223,175} \underline{67.39} & 65.86 & 68.75 & 73.59 & 85.24 & 95.26 & \cellcolor[RGB]{255,223,175}\underline{68.36} & 93.63 & 59.52 & 55.34 & 84.22 & 53.50 & 51.59\\

GNSD & \cellcolor[RGB]{255,190,110} \textbf{87.04} & 76.86 & 60.52 & 85.57 & 83.47 & 63.59 & \cellcolor[RGB]{255,223,175}\underline{76.72} & 73.62 & 65.50 & 93.23 & 95.71 & 63.32 & 73.67 & 83.68 & 51.80 & 62.42 & 90.86 & 53.42\\

LGNSDE & \cellcolor[RGB]{255,223,175} \underline{86.84} & 71.84 & 64.47 & 83.58 & 80.49 & 59.60 & 68.32 & \cellcolor[RGB]{255,223,175}\underline{78.65} & 64.03 & 90.58 &\cellcolor[RGB]{255,223,175}{\underline{97.73}} & 63.81 & 75.96 & 85.14 & 62.68 & 65.91 & 85.82 & 52.43\\
\midrule

HGNN & 77.46 & 71.49 & 64.59 & 83.77 & 81.85 & 53.72 & 71.09 & 73.17 & 60.46 & 74.42 & 85.31 & 58.23 & 84.76 & 84.90 & 57.82 & 79.73 & 91.85 & 59.41\\

HyperGCN & 74.81 & 69.27 & 65.65 & 71.49 & 85.11 & 62.91 & 72.97 & 63.91 & 62.17 & 76.20 & 88.25 & 54.35 & 82.87 & 81.03 & 46.77 & 65.09 & 90.13 & 70.81 \\

HGCN--Ens & 76.20 & 69.82 & 66.92 & 72.69 &84.36 & 61.92 & 73.32 & 65.40 & 63.58 & 77.29 & 88.69 & 55.58 & 84.13 & 79.52 & 45.65 & 66.52 & 89.22 & 70.52 \\

HNDiffN & 83.79 & 74.38 & 59.32 & 83.24 & 79.56 & 57.50 & 70.57 & 66.61 & 58.63 & 73.74 & 96.69 & 53.88 & 86.52 & 85.78 & 64.98 &\cellcolor[RGB]{255,223,175}\underline{88.43} & 88.28 & 69.29\\  

HND & 84.34 & 71.73 & \cellcolor[RGB]{255,223,175} \underline{69.60} & 83.79 & 85.99 & 55.73 & 74.06 & 73.52 & 62.99 & 91.58 & 93.63 & 54.27 & \cellcolor[RGB]{255,223,175}\underline{95.15} & 83.32 & 53.59 &78.79 & 89.84 & 70.78 \\

HyperGOOD &85.13 &\cellcolor[RGB]{255,223,175}\underline{77.44} & 67.71 & \cellcolor[RGB]{255,223,175}\underline{86.56} & \cellcolor[RGB]{255,223,175}\underline{86.61} & 64.30 & 74.62 & 71.43 & 68.41 & 87.32 & 94.47 & 66.79 & 89.25 & \cellcolor[RGB]{255,223,175}\underline{86.68} & \cellcolor[RGB]{255,223,175}\underline{76.79} & 86.87 &\cellcolor[RGB]{255,223,175}\underline{93.53} &\cellcolor[RGB]{255,223,175}\underline{73.20}\\
\midrule

HyperNSD &85.68 & \cellcolor[RGB]{255,190,110}\textbf{80.61} & \cellcolor[RGB]{255,190,110}\textbf{73.34} & \cellcolor[RGB]{255,190,110}\textbf{90.11} & \cellcolor[RGB]{255,190,110}\textbf{90.79} & \cellcolor[RGB]{255,190,110}\textbf{70.37} & \cellcolor[RGB]{255,190,110}\textbf{80.45} & \cellcolor[RGB]{255,190,110}\textbf{81.52} & \cellcolor[RGB]{255,190,110}\textbf{77.71} & \cellcolor[RGB]{255,223,175}\underline{94.73} & \cellcolor[RGB]{255,190,110}\textbf{99.26} & \cellcolor[RGB]{255,190,110}\textbf{71.97} & \cellcolor[RGB]{255,190,110}\textbf{98.64} & \cellcolor[RGB]{255,190,110}\textbf{89.96} & \cellcolor[RGB]{255,190,110}\textbf{79.89} & \cellcolor[RGB]{255,190,110}\textbf{92.33} & \cellcolor[RGB]{255,190,110}\textbf{95.10} &  \cellcolor[RGB]{255,190,110}\textbf{75.95}\\
\bottomrule
\end{tabular}
\end{adjustbox}
\end{table*}

\paragraph{Implementation Details} For HyperNSD, the diffusion horizon is fixed to \(T=1\), while the remaining key hyperparameters are selected by grid search according to validation performance. Specifically, the hidden layer dimension is chosen from \(\{16,32,64,128,256\}\), the learning rate from \(\{10^{-4},10^{-3},10^{-2}\}\), the number of training sample trajectories $M$ from \(\{1,3,5,7\}\), and the Euler--Maruyama step size $h$ from \(\{0.01,0.05,0.1,0.2\}\). All models are trained for 200 epochs. To ensure a fair comparison, HyperNSD and all competing methods use the same hidden layer dimension under each experimental setting. The baseline methods are implemented and fully optimized following their recommended configurations, and training is continued until stable convergence is achieved.
\paragraph{Baselines} We compare HyperNSD with three categories of representative methods. The first category consists of conventional OOD detection approaches, including MSP~\cite{hendrycks2016baseline}, ODIN~\cite{liang2017enhancing}, and Mahalanobis~\cite{lee2018simple}. Since these methods were originally developed for Euclidean or graph data, we replace their original feature encoders with a common hypergraph convolutional encoder to ensure compatibility with hypergraph-structured inputs. The second category includes graph uncertainty estimation methods, namely GNNSafe~\cite{wu2023energy}, GPN~\cite{stadler2021graph}, GNSD~\cite{lin2024graph}, and LGNSDE~\cite{bergna2025uncertainty}. The third category comprises hypergraph neural network methods, including HGNN~\cite{feng2019hypergraph}, HyperGCN~\cite{yadati2019hypergcn}, ensemble--based HyperGCN (HGCN-Ens)~\cite{lakshminarayanan2017simple}, HNDiffN~\cite{lu2025hypergraph}, HND~\cite{zhou2026hypergraph}, and HyperGOOD~\cite{cai2026hypergood}.

For the GNN-based baselines, each hypergraph dataset is first converted into a pairwise graph using clique expansion, after which the corresponding graph models are trained and evaluated. For deterministic GNN and HGNN baselines, aleatoric uncertainty is used for all detection tasks. For uncertainty-aware methods, epistemic uncertainty is adopted for OOD detection, whereas aleatoric uncertainty is used for misclassification detection, following prior studies~\cite{zhao2020uncertainty,stadler2021graph}. Details of all baselines are provided in Appendix~\ref{baseline}. 
\paragraph{Evaluation Metrics} An effective uncertainty-aware HGNN should accurately identify anomalous samples while maintaining strong predictive performance on in-distribution data. Accordingly, we adopt AUROC, AUPR, FPR95, and in-distribution classification accuracy (ID ACC) as the evaluation metrics. Higher AUROC, AUPR, and ID ACC indicate better performance, whereas lower FPR95 is preferred. In each comparison, the best result is highlighted in \colorbox[RGB]{255,190,110}{\textbf{bold}} and the second-best result is \colorbox[RGB]{255,223,175}{\underline{underlined}}. Further details of these metrics are provided in Appendix~\ref{metrics}.

\begin{table*}[!htbp]
\belowrulesep=0pt
\aboverulesep=0pt
\renewcommand{\arraystretch}{1.5}
\caption{Misclassification detection performance (\%) comparison on six datasets. Standard deviations of misclassification detection performance are deferred to Appendix~\ref{result}.}
\label{table:mis}
\centering
\begin{adjustbox}{width=\textwidth}
\begin{tabular}{|c|c|cccc|cccccc|c|@{}}
\toprule

Dataset & \diagbox{Metric}{Model} & GNNSafe & GPN & GNSD & LGNSDE & HGNN  & HyperGCN & HGCN--Ens& HNDiffN &HND & HyperGOOD & HyperNSD\\
\midrule

\multirow[c]{4}{*}{Cora} & AUROC & 83.16 & 79.46 & 80.00 & 83.66  & 84.50 & 82.36 & 83.17 & 81.70 & \cellcolor[RGB]{255,223,175}\underline{85.13} & 82.79 &\cellcolor[RGB]{255,190,110}\textbf{86.28}\\

& AUPR succ & 90.77 & 85.94 & 91.53 & 93.90  & 95.08 & 94.59 & 95.45 & 94.23 & \cellcolor[RGB]{255,223,175}\underline{95.77}  & 91.31 & \cellcolor[RGB]{255,190,110}\textbf{97.25}\\

& AUPR err & 69.74 & 71.89 & 57.51 & 64.13  & 59.20 & 51.20 & 53.46 & 53.81 & 57.75 & \cellcolor[RGB]{255,223,175}\underline{76.81} & \cellcolor[RGB]{255,190,110}\textbf{80.40} \\

&FPR95 & 68.95 & 61.63 & 76.02 & \cellcolor[RGB]{255,223,175}\underline{60.00}  & 66.21 & 73.29 & 71.92 & 72.67 & 65.28 & 69.73 & \cellcolor[RGB]{255,190,110}\textbf{55.88}\\
\midrule

\multirow[c]{4}{*}{Cora-CA} & AUROC & 79.79 & 76.52 & 78.73 & 81.12 & 82.74 & 81.21 & 81.85 & \cellcolor[RGB]{255,223,175}\underline{83.27} & 81.89 & 83.16 & \cellcolor[RGB]{255,190,110}\textbf{85.24}\\

 & AUPR succ & 85.31 & 82.85 & 90.46 & \cellcolor[RGB]{255,223,175}\underline{93.89}  & 93.80 & 93.36 & 91.98 & 93.14 & 93.34& 92.71 & \cellcolor[RGB]{255,190,110}\textbf{95.20}\\

& AUPR err & \cellcolor[RGB]{255,223,175}\underline{71.68} & 67.75 & 53.92 & 53.12 & 58.75 & 59.85 & 58.68 & 64.71 & 55.52 & 68.04 & \cellcolor[RGB]{255,190,110}\textbf{74.04}\\

&FPR95 & 78.98 & 79.79 & 86.87 & 76.13  & 66.67 & 69.88 & 73.07 & 68.15 & \cellcolor[RGB]{255,223,175}\underline{62.12} & 79.50 & \cellcolor[RGB]{255,190,110}\textbf{59.34}\\
\midrule

\multirow[c]{4}{*}{Citeseer} & AUROC & 75.58 &\cellcolor[RGB]{255,223,175}\underline{83.39} & 77.12 & 75.13  & 78.86 & 76.97 & 77.29 &  78.50 & 81.58  & 83.23 &\cellcolor[RGB]{255,190,110}\textbf{85.71}\\

& AUPR succ & 72.38 & 82.81 & \cellcolor[RGB]{255,223,175}\underline{86.87} & 84.08  & 80.23 & 82.00 & 83.35 & 76.86 & 85.11 & 79.01 & \cellcolor[RGB]{255,190,110}\textbf{87.81}\\

& AUPR err & 77.17 & 68.48 & 62.17 & 56.62 & 56.44  & 58.35 & 72.66 & 61.02 & \cellcolor[RGB]{255,223,175}\underline{86.70}  & 74.13 &\cellcolor[RGB]{255,190,110}\textbf{88.06}\\

&FPR95 & 74.83 & \cellcolor[RGB]{255,223,175}\underline{65.93} & 76.51 & 86.11 &  73.60 & 73.95 & 68.69 & 77.74 & 67.40 &  70.80 &\cellcolor[RGB]{255,190,110}\textbf{62.78}\\
\midrule

\multirow[c]{4}{*}{DBLP} & AUROC & 86.50 & 87.93 & \cellcolor[RGB]{255,223,175}\underline{87.12} & 87.57  & 86.02 & 86.66 & 87.07 & 87.20 & 87.09 & 80.71 &\cellcolor[RGB]{255,190,110}\textbf{89.92}\\

& AUPR succ & 92.20 & 94.08 & \cellcolor[RGB]{255,223,175}\underline{98.44} & 98.34  & 98.35 & 97.59 & 97.98 & \cellcolor[RGB]{255,190,110}\textbf{98.92} & 96.34 & 97.50 & 98.11\\

& AUPR err & \cellcolor[RGB]{255,223,175}\underline{74.53} & 73.71 & 63.26 & 69.91  & 62.28 & 71.34 & 70.84 & 61.40 & 69.80 & 66.78 &\cellcolor[RGB]{255,190,110}\textbf{76.13}\\

&FPR95 & 71.06 & 60.23 & \cellcolor[RGB]{255,223,175}\underline{58.40} & 61.32  & 61.13 & 67.23 & 59.92 & 60.00 & 69.03 & 68.78 &\cellcolor[RGB]{255,190,110}\textbf{53.65}\\
\midrule

\multirow[c]{4}{*}{ModelNet40} & AUROC & 71.76 & 77.76 & 96.51 & 97.30  & 98.24 & 98.10 & 98.52 & 98.15 & 98.48 & \cellcolor[RGB]{255,223,175}\underline{99.16} & \cellcolor[RGB]{255,190,110}\textbf{99.53}\\

& AUPR succ & 81.88 & 84.75 & 98.01 & 98.44  & 98.06 & 98.19 & 98.37 & 99.06 & 98.45 & \cellcolor[RGB]{255,190,110}\textbf{99.60} & \cellcolor[RGB]{255,223,175}\underline{99.52}\\

& AUPR err & 58.15 & 75.15 & 92.88 & 94.38  & 98.60 & 97.80 & 98.58 & 98.07 & \cellcolor[RGB]{255,223,175}\underline{99.51} & 97.08 &\cellcolor[RGB]{255,190,110}\textbf{99.59}\\

&FPR95 &43.65 & 35.32 & 25.55 & 10.11 & 2.56 & 2.37 & 2.33 & 3.73 & \cellcolor[RGB]{255,190,110}\textbf{1.42} & 5.82 & \cellcolor[RGB]{255,223,175}\underline{2.13}\\
\midrule

\multirow[c]{4}{*}{NTU2012} & AUROC & 70.97 & 80.32 & 97.37 & 98.68 & 99.26 & 99.30 & 99.32 & 99.30 & 99.27 & \cellcolor[RGB]{255,190,110}\textbf{99.51} & \cellcolor[RGB]{255,223,175}\underline{99.38}\\

& AUPR succ & 81.21 & 85.43 & 98.49 & 99.30 & 98.96 & 99.38 & 99.22 & 99.49 & 99.52 & \cellcolor[RGB]{255,190,110}\textbf{99.75} & \cellcolor[RGB]{255,190,110}\underline{99.55}\\

& AUPR err & 57.63 & 70.54 & 93.42 & 96.15 & 98.57 & 98.99 & 99.06 & 99.00 & 98.82 & \cellcolor[RGB]{255,223,175}\underline{99.17} &\cellcolor[RGB]{255,190,110}\textbf{99.84}\\

&FPR95 & 43.14 & 28.12 & 10.52 & 16.52 & 3.19 & 2.96 & 2.53 & \cellcolor[RGB]{255,223,175}\underline{2.09} & \cellcolor[RGB]{255,190,110}\textbf{0.61} & 8.33 & 2.28\\
\bottomrule
\end{tabular}
\end{adjustbox}
\end{table*}

\subsection{OOD Detection}
We evaluate our model in three OOD scenarios (label leave-out, feature interpolation, and structure manipulation denote by L, F, and S, respectively). The OOD detection results are provided in Table~\ref{table:ood_result} and more experimental results are provided in Appendix~\ref{result} (Tables~\ref{table:ood_cora}--\ref{table:ood_std}). From Table~\ref{table:ood_result}, we draw the following observations. First, HyperNSD achieves the best AUROC in 16 out of 18 evaluation settings, including all feature interpolation and structure manipulation experiments. Compared with the strongest competing method in each setting, HyperNSD obtains its largest improvement of 4.18\% on Cora-CA under feature interpolation. Moreover, it improves the average AUROC by 2.77\% and 3.18\% under feature interpolation and structure manipulation, respectively. Second, the consistent improvements across citation and real-world hypergraphs demonstrate that HyperNSD can provide more informative uncertainty scores for identifying OOD samples arising from semantic, attribute, and structural distribution shifts. Third, the particularly strong performance under structure manipulation indicates that directly modeling stochastic dynamics over the original node--hyperedge incidence domain is more effective than applying graph uncertainty methods to pairwise approximations of hypergraphs.

\subsection{Misclassification Detection}
Besides OOD detection, we also evaluate the ability of different methods to identify misclassified samples. Misclassification detection examines whether the estimated uncertainty can effectively distinguish incorrect predictions from correct ones at test time~\cite{hendrycks2016baseline}. As shown in Table~\ref{table:mis}, HyperNSD achieves the best performance in 18 out of 24 evaluation settings across the six datasets. On average, compared with the strongest competing baseline for each metric, HyperNSD improves AUROC by 2.10\%, AUPR succ by 1.49\%, and AUPR err by 6.01\%, while reducing FPR95 by 4.97\%. These results demonstrate that the path-conditioned predictive entropy estimated by HyperNSD is highly informative for identifying incorrect predictions while preserving strong classification performance.

\subsection{Comparison with Graph Expansions}
\begin{figure*}[ht]
    \centering
    \includegraphics[scale=0.9]{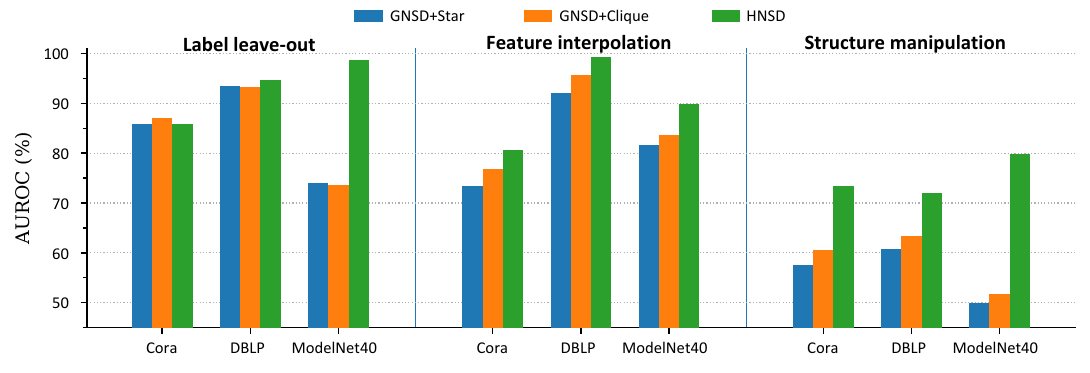}
   \caption{AUROC comparison between GNSD implemented with star and clique expansions and the proposed HyperNSD under three OOD types (label leave-out, feature interpolation, and structure manipulation) on Cora, DBLP and ModelNet40.}
\label{graphandhypergraph}
\end{figure*}
To examine the necessity of modeling stochastic diffusion directly over the original hypergraph incidence domain, we compare HyperNSD with GNSD implemented using star and clique expansions. As shown in Fig.~\ref{graphandhypergraph}, HyperNSD consistently achieves superior or competitive performance across the three OOD settings. Its advantage is especially pronounced under feature interpolation and structure manipulation, while remaining competitive under label leave-out. These results indicate that preserving the original node--hyperedge incidence geometry enables more effective uncertainty propagation and improves the detection of attribute- and structure-induced distribution shifts compared with pairwise graph expansions.

\subsection{Model Variants}
\begin{figure*}[ht]
    \centering
    \includegraphics[scale=0.47]{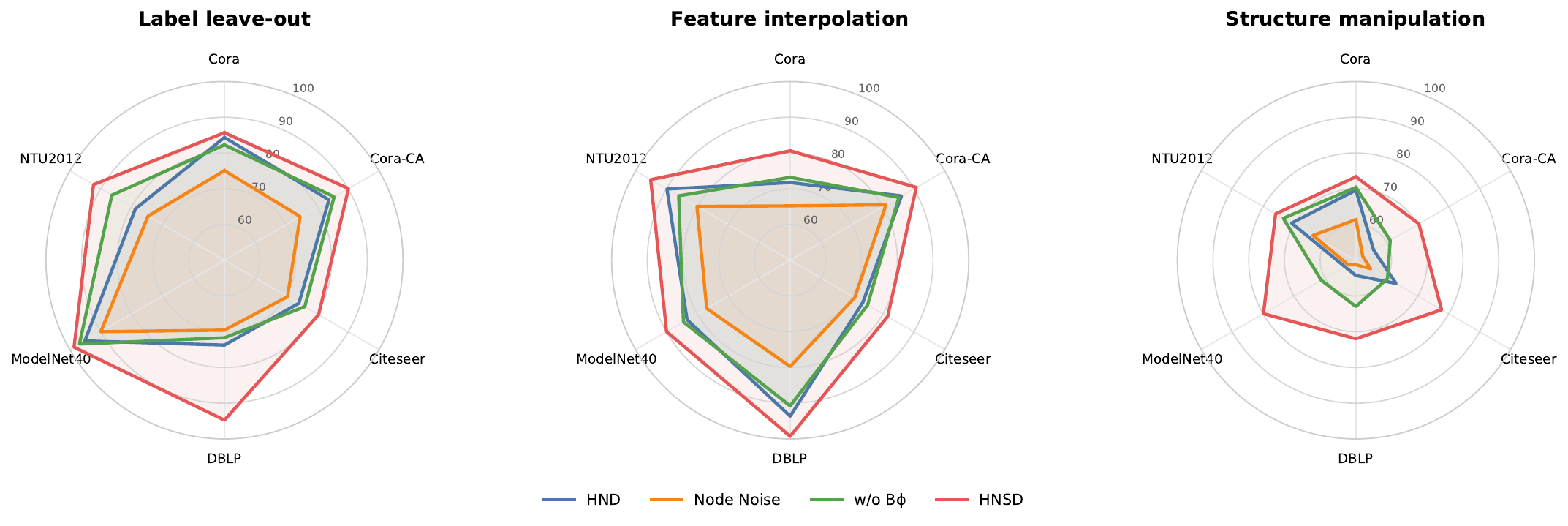}
   \caption{AUROC comparison of HyperNSD and its variants under three OOD types (label leave-out, feature interpolation, and structure manipulation).}
\label{modelvariant}
\end{figure*}
To investigate the effect of different noise modeling strategies, we compare HyperNSD with deterministic HND, a variant using node-level noise, and a variant without the learnable \(B_{\phi}\). As shown in Fig.~\ref{modelvariant}, the complete HyperNSD consistently outperforms its variants across the three OOD settings. The deterministic HND and the variant with simple node-level noise generally achieve lower performance, indicating that stochasticity alone is insufficient for reliable OOD detection. Removing \(B_{\phi}\) also leads to performance degradation, particularly under feature and structure shifts. These results confirm that incidence-aware stochastic forcing and adaptive noise modulation are both important to the effectiveness of HyperNSD.

\subsection{Perturbation Intensity Analysis}
\begin{figure}[ht]
    \centering
    \includegraphics[scale=0.45]{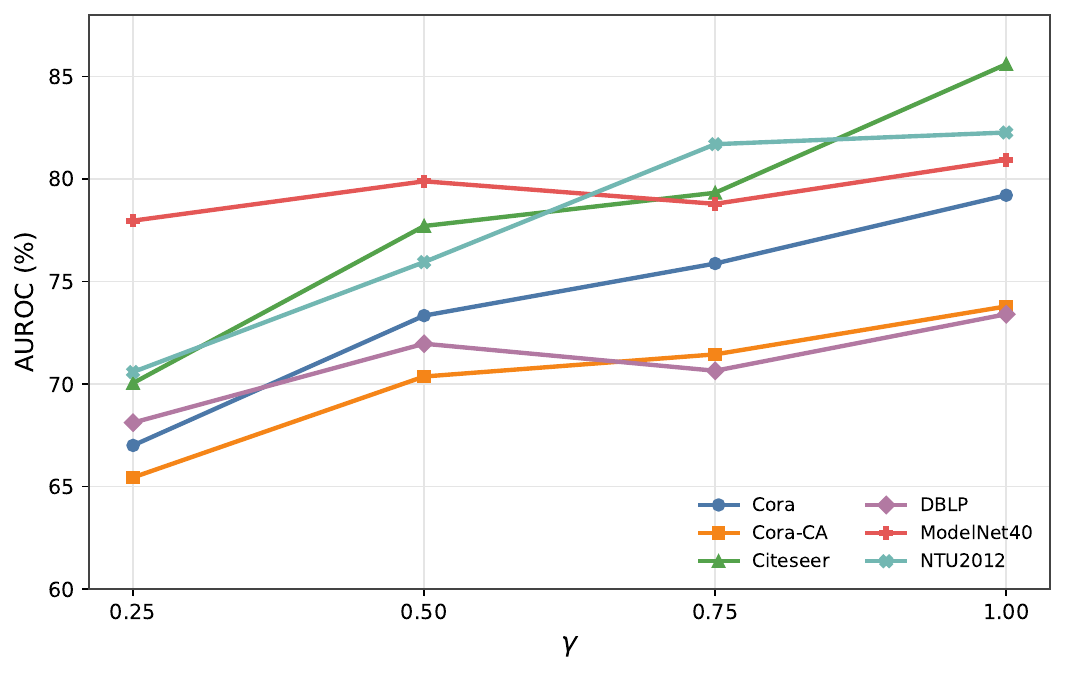}
   \caption{Effect of the structural perturbation intensity \(\gamma\) on the OOD detection performance of HyperNSD.}
\label{perturbation}
\end{figure}
We further investigate the influence of the rewiring ratio \(\gamma\) on the OOD detection performance of HyperNSD, with the results reported in Fig.~\ref{perturbation}. From Fig.~\ref{perturbation}, we draw the following observations. First, the AUROC generally increases as the rewiring ratio grows across most datasets, indicating that stronger structural shifts produce more distinguishable stochastic representations between ID and OOD samples. Second, HyperNSD still maintains satisfactory detection performance under relatively mild structural perturbations, demonstrating its ability to capture subtle structural distribution shifts. Third, although slight fluctuations are observed on several datasets, the overall performance remains stable or improves with increasing rewiring ratio. These results demonstrate that the uncertainty estimated by HyperNSD responds effectively and consistently to different levels of hypergraph structural perturbations.

\subsection{Visual Analysis}
\begin{figure*}[ht]
    \centering
    \includegraphics[scale=0.47]{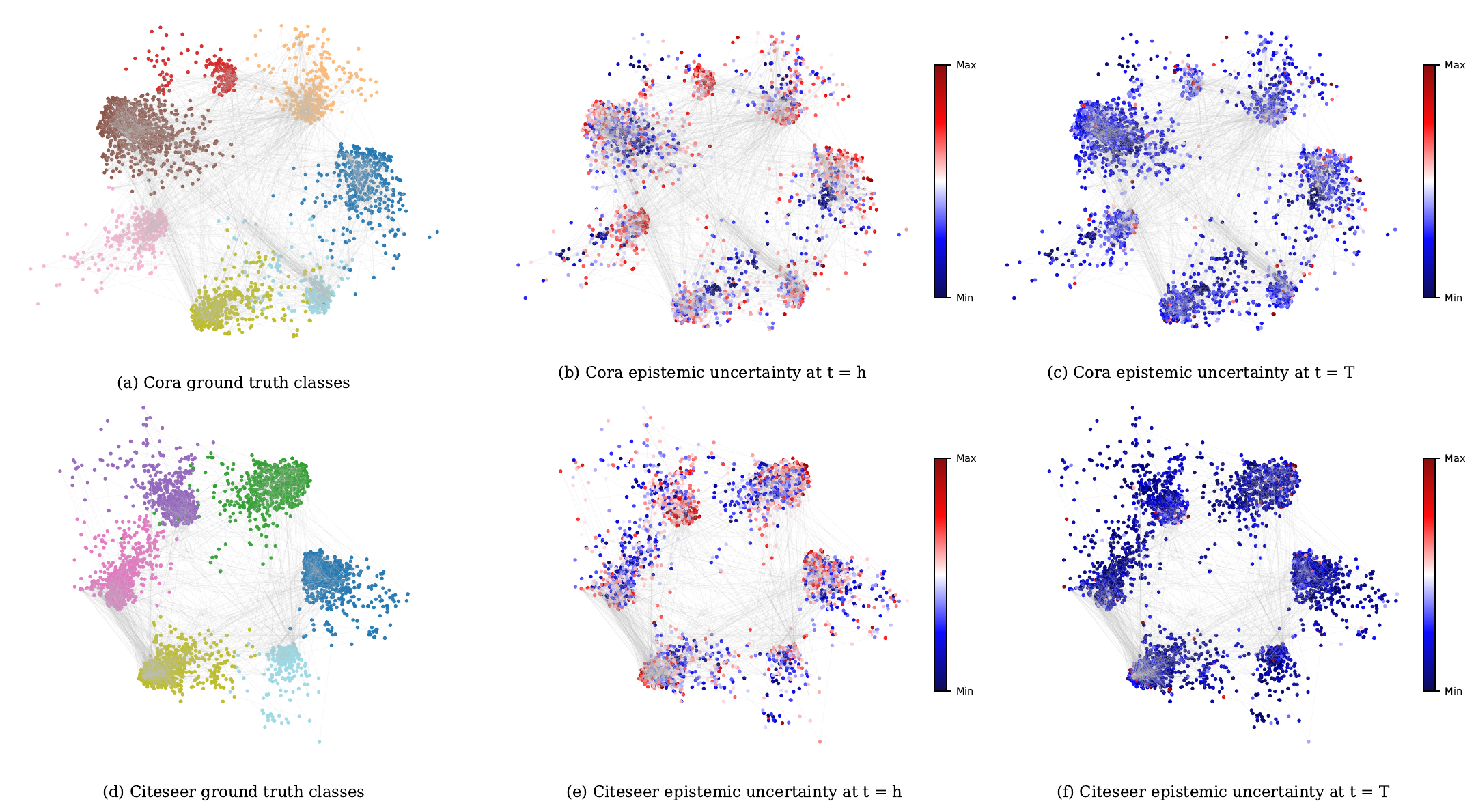}
   \caption{Visualization of ground-truth classes and node-level epistemic uncertainty at $t=h$ and terminal diffusion states on Cora and Citeseer.}
\label{epistemic}
\end{figure*}
To qualitatively examine how uncertainty evolves during stochastic diffusion, we visualize the node-level epistemic uncertainty at the $t=h$ and terminal states on Cora and Citeseer. As shown in Fig.~\ref{epistemic}, the $t=h$ representations exhibit broadly distributed high-uncertainty regions, whereas the uncertainty is substantially reduced after diffusion. Most nodes within coherent class clusters become more certain, while relatively high uncertainty remains around class boundaries and structurally ambiguous regions. These observations indicate that HyperNSD progressively refines node representations while preserving uncertainty for difficult or potentially unreliable samples.

\section{Conclusion}
In this paper, we presented Hypergraph Neural Stochastic Diffusion (HyperNSD), an SDE-based framework for uncertainty estimation on hypergraphs. HyperNSD models node representations as stochastic processes evolving over the node--hyperedge incidence domain, where a learnable drift term governs deterministic higher-order feature propagation and an incidence-aware stochastic forcing term captures structure-dependent representation variability. We established theoretical guarantees for the proposed dynamics, including global well-posedness, finite-horizon moment control, stability under initial-state and structural perturbations, structure-constrained stochastic modes, permutation equivariance, and convergence of the Euler--Maruyama discretization. Extensive experiments on six hypergraph benchmarks demonstrated that HyperNSD provides informative uncertainty estimates for OOD and misclassification detection while maintaining competitive in-distribution predictive performance. Additional comparisons and analyses further verified the advantages of directly modeling stochastic diffusion over the original incidence structure and the effectiveness of adaptively modulated incidence-level stochastic perturbations. Overall, HyperNSD provides a principled stochastic-dynamical foundation for uncertainty-aware higher-order representation learning. Future work will investigate more expressive stochastic processes and extend the framework to dynamic, heterogeneous, and large-scale hypergraphs.

\bibliographystyle{IEEEtran}
\bibliography{IEEEabrv,refer}
\begin{IEEEbiography}[{\includegraphics[width=1in,height=1.25in,clip,keepaspectratio]{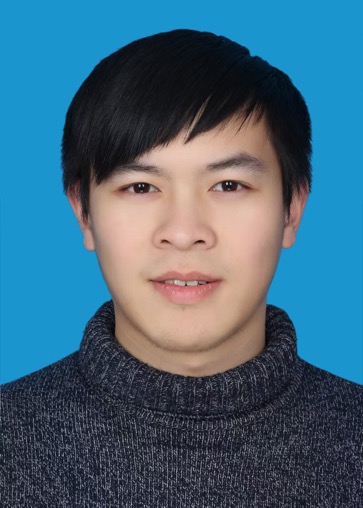}}]{Zhiheng Zhou }
received the Ph.D. degree from the Academy of Mathematics and Systems Science, Chinese Academy of Sciences in 2023. He is a
postdoctoral researcher at School of Mathematics and Statistics, Shandong University. His research interests include graph representation learning and bioinformatics. He
has published over 20 papers in academic journals and conferences, including KDD, MedIA, KBS, JBHI, AIIM, CIKM and ECML.
\end{IEEEbiography}

\begin{IEEEbiography}[{\includegraphics[width=1in,height=1.25in,clip,keepaspectratio]{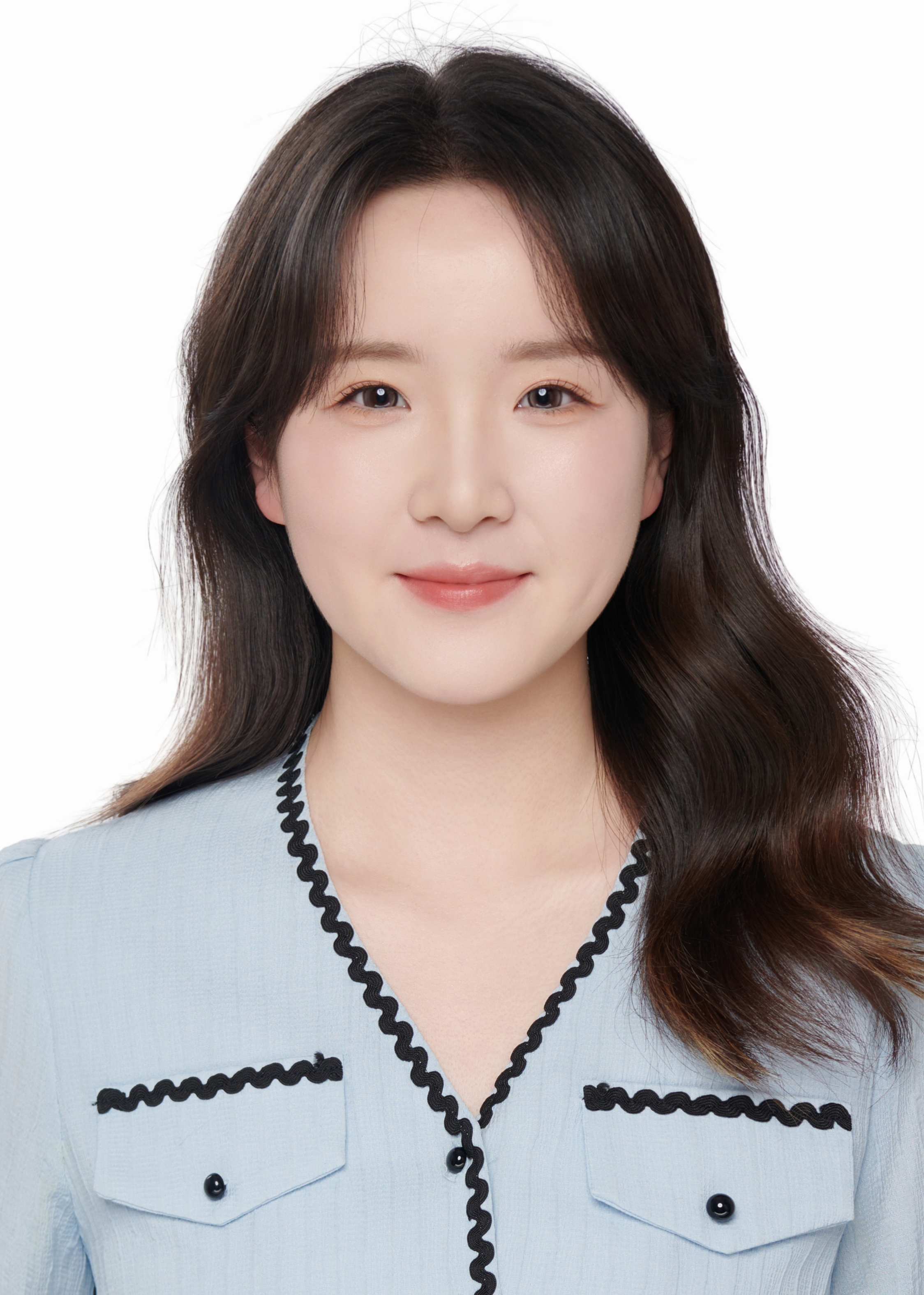}}]{Mengyao Zhou }received the B.S. degree in mathematics from Dalian University of Technology, Dalian, China, in 2022. She is currently pursuing the Ph.D. degree with the Academy of Mathematics and Systems Science, Chinese Academy of Sciences, Beijing, China. Her research interests include graph representation learning and bioinformatics.
\end{IEEEbiography}

\begin{IEEEbiography}[{\includegraphics[width=1in,height=1.25in,clip,keepaspectratio]{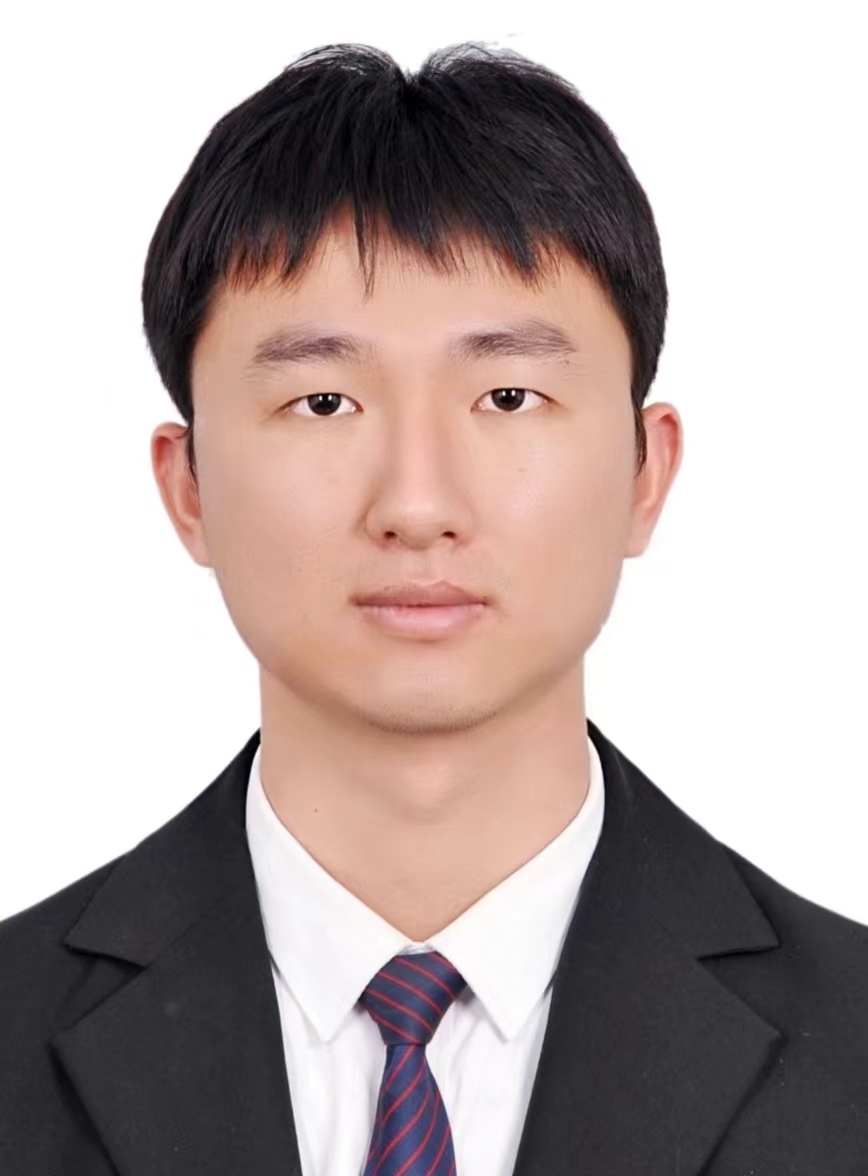}}]{Dengyi Zhao} is currently a Ph.D. student with the School of Mathematics and Statistics, Shandong University, Weihai, China. His current research interests include complex networks, topological data analysis, graph neural networks, and their applications in brain imaging.
\end{IEEEbiography}

\begin{IEEEbiography}[{\includegraphics[width=1in,height=1.25in,clip,keepaspectratio]{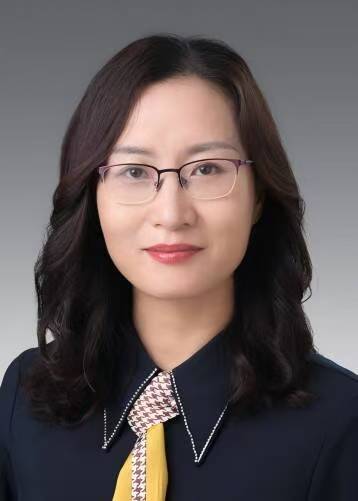}}]{Xingqin Qi }
received the Ph.D. degree from Shandong University, Jinan, China, in 2006. She is currently a Professor with the School of Mathematics and Statistics, Shandong University, Weihai, China. Her current research interests include complex networks and graph data mining.
\end{IEEEbiography}

\begin{IEEEbiography}[{\includegraphics[width=1in,height=1.25in,clip,keepaspectratio]{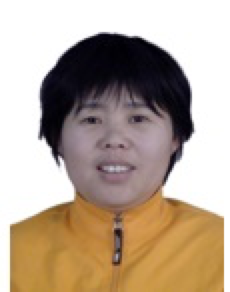}}]{Guiying Yan}
received the B.S., M.S., and Ph.D. degrees in mathematics from Shandong University. She is currently a Professor with the Academy of Mathematics and Systems Science, Chinese Academy of Sciences, Beijing, China. Her research interests include graph theory and its applications.
\end{IEEEbiography}

\clearpage
\appendices
\section{Proof of Proposition~\ref{P1}}
\label{appendix:Proposition1}
We first recall the definition of a $\mathcal Q$-Wiener process in a separable Hilbert space~\cite{lord2014introduction}.

\begin{Definition}[$\mathcal Q$-Wiener Process]
\label{def:q_wiener}
Let $\mathcal U$ be a separable Hilbert space, and let $(\Omega,\mathcal F,\{\mathcal F_t\}_{t\ge0},\mathbb P)$ be a filtered probability space. Let \(\mathcal Q:\mathcal U\rightarrow\mathcal U\) be a self-adjoint, nonnegative, and trace-class operator. An $\mathcal U$-valued stochastic process $\mathbf Z(t)$ is called a $\mathcal Q$-Wiener process if the following conditions hold:
\begin{enumerate}
\item $\mathbf Z(0)=0$ almost surely;
\item $\mathbf Z(t)$ is adapted to ${\mathcal F_t}$ and has continuous sample paths;
\item $\mathbf Z(t)$ has independent increments;
\item for every $0\leq s<t$, \(\mathbf Z(t)-\mathbf Z(s)\sim \mathcal N\left(0,(t-s)\mathcal Q\right).\)
\end{enumerate}
\end{Definition}
Based on Definition~\ref{def:q_wiener}, we now verify that the incidence-space process $\mathbf Z(t)=\mathbf G\mathbf W(t)$ constructed in the proposed HSDE in Eq.~\eqref{eq:general_hsde} is a $\mathcal Q$-Wiener process.
\begin{proof}[Proof of Proposition~\ref{P1}]
Since $G$ is deterministic and $\mathbf W(t)$ is ${\mathcal F_t}$-adapted, the process $\mathbf Z(t)=\mathbf G\mathbf W(t)$ is also ${\mathcal F_t}$-adapted. For each feature dimension $\ell=1,\ldots,d$, define \(\mathbf Z_\ell(t)=\mathbf G\mathbf W_\ell(t)\in\mathbb R^N.\) Because $\mathbf W_\ell(t)$ is an $n$-dimensional standard Wiener process and $\mathbf G\in\mathbb R^{N\times n}$ is deterministic, $\mathbf Z_\ell(t)$ is a continuous Gaussian process with $\mathbf Z_\ell(0)=0$. For any $0\leq s<t$, its increment satisfies
\begin{equation}
\mathbf Z_\ell(t)-\mathbf Z_\ell(s)=\mathbf G\left(\mathbf W_\ell(t)-\mathbf W_\ell(s)\right).
\end{equation}
Since \(\mathbf W_\ell(t)-\mathbf W_\ell(s)\sim\mathcal N\left(0,(t-s)\mathbf I_n\right),\) it follows that
\begin{equation}
\mathbf Z_\ell(t)-\mathbf Z_\ell(s)\sim\mathcal N\left(0,(t-s)\mathbf G\mathbf G^{\top}\right).
\label{eq:incidence_increment_distribution}
\end{equation}
Furthermore, the independent and stationary increment properties of $\mathbf W_\ell(t)$ are preserved under the deterministic linear transformation $\mathbf G$. Therefore, $\mathbf Z_\ell(t)$ is a possibly degenerate Wiener process in $\mathbb R^N$ with covariance matrix $\mathbf G\mathbf G^{\top}$.

Given $\mathbf g,\mathbf g'\in L(\mathcal E,\mathcal V)$, and their feature-wise columns as
\begin{equation}
\mathbf g=\left[\mathbf g_1,\ldots,\mathbf g_d\right],\qquad\mathbf g'=\left[\mathbf g'_1,\ldots,\mathbf g'_d\right].
\end{equation}
For the increment \(\Delta\mathbf Z=\mathbf Z(t)-\mathbf Z(s),\) the Frobenius inner product gives
\begin{equation}
\left\langle\Delta\mathbf Z,\mathbf g\right\rangle_F=\sum_{\ell=1}^d\left\langle G\Delta \mathbf W_\ell,\mathbf g_\ell\right\rangle,
\end{equation}
where \(\Delta \mathbf W_\ell=\mathbf W_\ell(t)-\mathbf W_\ell(s).
\) Because the Wiener processes associated with different feature dimensions are mutually independent, the cross-dimension covariance terms vanish. Hence,
\begin{equation}
\begin{aligned}
&\mathbb E\left[\left\langle\Delta\mathbf Z,\mathbf g\right\rangle_F\left\langle\Delta\mathbf Z,\mathbf g'\right\rangle_F\right]=\sum_{\ell=1}^d\mathbb E\left[\left\langle\mathbf G\Delta \mathbf W_\ell,\mathbf g_\ell\right\rangle\left\langle\mathbf G\Delta \mathbf W_\ell,\mathbf g'_\ell\right\rangle\right]\\
&=(t-s)\sum_{\ell=1}^d\mathbf g_\ell^{\top}\mathbf G\mathbf G^{\top}\mathbf g'_\ell=(t-s)\left\langle\mathcal Q\mathbf g,\mathbf g'\right\rangle_F,
\end{aligned}
\label{eq:q_wiener_covariance_identity}
\end{equation}
where \(\mathcal Q\mathbf g=\mathbf G\mathbf G^{\top}\mathbf g.\) This proves that $\mathbf Z(t)$ is an $L(\mathcal E,\mathcal V)$-valued $\mathcal Q$-Wiener process and that $\mathcal Q$ acts independently on the $d$ feature dimensions. Equivalently, vectorizing the matrix-valued process gives
\begin{equation}
\operatorname{Cov}\left(\operatorname{vec}\left(\mathbf Z(t)-\mathbf Z(s)
\right)\right)=(t-s)\left(I_d\otimes \mathbf G\mathbf G^{\top}\right).
\end{equation}

It remains to verify the stated properties of $\mathcal Q$. For arbitrary
$\mathbf g,\mathbf g'\in L(\mathcal E,\mathcal V)$,
\begin{equation}
\begin{aligned}
\left\langle\mathcal Q\mathbf g,\mathbf g'\right\rangle_F=\left\langle\mathbf G\mathbf G^{\top}\mathbf g,\mathbf g'\right\rangle_F=\left\langle\mathbf g,\mathbf G\mathbf G^{\top}\mathbf g'
\right\rangle_F=\left\langle\mathbf g,\mathcal Q\mathbf g'\right\rangle_F.
\end{aligned}
\end{equation}
Thus, $\mathcal Q$ is self-adjoint. Moreover,
\begin{equation}
\begin{aligned}
\left\langle\mathcal Q\mathbf U,\mathbf U\right\rangle_F=\left\langle\mathbf G\mathbf G^{\top}\mathbf U,\mathbf U
\right\rangle_F=\|\mathbf G^{\top}\mathbf U\|_F^2\geq 0,
\end{aligned}
\end{equation}
so $\mathcal Q$ is positive semidefinite.

Finally, since $L(\mathcal E,\mathcal V)$ is finite-dimensional, every linear operator on this space is trace class. More explicitly, the trace of the full channel-wise covariance operator is
\begin{equation}
\begin{aligned}
\operatorname{Tr}(\mathcal Q)=d\operatorname{Tr}\left(\mathbf G\mathbf G^{\top}\right)=d\|\mathbf G\|_F^2<\infty.
\end{aligned}
\end{equation}
Therefore, $\mathcal Q$ is trace class. This completes the proof.
\end{proof}

\section{Proof of Proposition~\ref{prop:coefficient_regularity}}
\label{appendix:Proposition2}
\begin{proof}[Proof of Proposition~\ref{prop:coefficient_regularity}]
Assume that the aggregation operator is locally Lipschitz. This condition is satisfied by the mean and max aggregators considered in this work. The concatenation operation is linear, and the scalar-valued neural networks together with the LeakyReLU activation are locally Lipschitz. In addition, the softmax mapping is continuously differentiable with a uniformly bounded Jacobian. Consequently, for every $R>0$, there exist constants $L_{A,R}>0$ and $L_{B,R}>0$ such that
\begin{equation}
\begin{aligned}
\|\mathbf A_{\theta}(\mathbf X)-\mathbf A_{\theta}(\mathbf Y)\|_2\leq L^{X}_{A,R}\|\mathbf X-\mathbf Y\|_F\\
\|\mathbf B_{\phi}(\mathbf X)-\mathbf B_{\phi}(\mathbf Y)\|_2\leq L^{X}_{B,R}
\|\mathbf X-\mathbf Y\|_F
\label{eq:A_local_lipschitz}
\end{aligned}
\end{equation}
whenever $\|\mathbf X\|_F,\|\mathbf Y\|_F\leq R$.

Since $\mathbf A_{\theta}$ and $\mathbf B_{\phi}$ are diagonal matrices, it follows that
\begin{equation}
\|\mathbf A_{\theta}(\mathbf X)\|_2\leq1,\qquad\|\mathbf B_{\phi}(\mathbf X)\|_2\leq1
\label{eq:modulation_uniform_bound}
\end{equation}
for every $\mathbf X$.

We first consider the drift coefficient. By adding and subtracting
$\mathbf G^{\top}\mathbf A_{\theta}(\mathbf X)\mathbf G\mathbf Y$, we obtain
\begin{equation}
\begin{aligned}
\mathbf F^{\theta}(\mathbf X)-\mathbf F^{\theta}(\mathbf Y)=
&-\mathbf G^{\top}\mathbf A_{\theta}(\mathbf X)\mathbf G(\mathbf X-\mathbf Y)\\
&-\mathbf G^{\top}\left[\mathbf A_{\theta}(\mathbf X) -\mathbf A_{\theta}(\mathbf Y)\right]\mathbf G\mathbf Y.
\end{aligned}
\end{equation}
Therefore,
\begin{equation}
\begin{aligned}
\|\mathbf F^{\theta}(\mathbf X)-\mathbf F^{\theta}(\mathbf Y)\|_F
&\leq\|\mathbf G\|_F^2\|\mathbf A_{\theta}(\mathbf X)\|_2\|\mathbf X-\mathbf Y\|_F\\
&+\|\mathbf G\|_F^2\|\mathbf A_{\theta}(\mathbf X)-\mathbf A_{\theta}(\mathbf Y)\|_2\|\mathbf Y\|_F.
\end{aligned}
\end{equation}
Using Eqs.~\eqref{eq:A_local_lipschitz} and \eqref{eq:modulation_uniform_bound}, together with $|\mathbf Y|_F\leq R$, yields
\begin{equation}
\|\mathbf F^{\theta}(\mathbf X)-\mathbf F^{\theta}(\mathbf Y)\|_F\leq\|\mathbf G\|_F^2\left(1+L^{X}_{A,R}R\right)\|\mathbf X-\mathbf Y\|_F.
\label{eq:F_local_lipschitz_bound}
\end{equation}
Thus, $\mathbf F^{\theta}$ is locally Lipschitz continuous.

We next consider the stochastic diffusion operator. Similar to the analysis of $\mathbf F^{\theta}$, we obtain
\begin{equation}
\begin{aligned}
&\|\boldsymbol{\Sigma}^{\phi}(\mathbf X)-\boldsymbol{\Sigma}^{\phi}(\mathbf Y)\|_{F}^2\leq C_{G,R}\|\mathbf X-\mathbf Y\|_F^2.
\end{aligned}
\label{eq:Sigma_l2_bound}
\end{equation}
where \(C_{G,R}=\|\mathbf G\|_F^4L^2_{B,R}.\) Together with Eq.~\eqref{eq:F_local_lipschitz_bound}, this proves Eq.~\eqref{eq:neural_coefficients_local_lipschitz}.

It remains to verify the linear-growth condition. From
Eq.~\eqref{eq:modulation_uniform_bound},
\begin{equation}
\begin{aligned}
\|\mathbf F^{\theta}(\mathbf X)\|_F=\|\mathbf G^{\top}\mathbf A_{\theta}(\mathbf X)\mathbf G\mathbf X\|_F
&\leq\|\mathbf G\|_F^2\|\mathbf A_{\theta}(\mathbf X)\|_2\|\mathbf X\|_F\\
&\leq\|\mathbf G\|_F^2\|\mathbf X\|_F.
\end{aligned}
\end{equation}
Hence,
\begin{equation}
\|\mathbf F^{\theta}(\mathbf X)\|_F^2\leq\|\mathbf G\|_F^4\|\mathbf X\|_F^2.
\label{eq:F_linear_growth_bound}
\end{equation}

Similarly, using the boundedness of $\mathbf B_{\phi}(\mathbf X)$ in
Eq.~\eqref{eq:modulation_uniform_bound}, we obtain
\begin{equation}
\|\boldsymbol{\Sigma}^{\phi}(\mathbf X)\|_{F}^2\leq \|\mathbf G\|_F^4.
\label{eq:Sigma_linear_growth_bound}
\end{equation}
Combining Eqs.~\eqref{eq:F_linear_growth_bound} and
\eqref{eq:Sigma_linear_growth_bound} gives
\begin{equation}
\begin{aligned}
\|\mathbf F^{\theta}(\mathbf X)\|_F^2+\|\boldsymbol{\Sigma}^{\phi}(\mathbf X)\|_{F}^2
&\leq\|\mathbf G\|_F^4\|\mathbf X\|_F^2+\|\mathbf G\|_F^4\\
&\leq \|\mathbf G\|_F^4\left(1+\|\mathbf X\|_F^2\right).
\end{aligned}
\end{equation}
This proves the linear-growth condition and completes the proof.
\end{proof}

\section{Proof of Theorem~\ref{thm:wellposedness_energy_stability}}
\label{appendix:theorem1}
\begin{proof}[Proof of Theorem~\ref{thm:wellposedness_energy_stability}]
The Eq.~\eqref{eq:node_driven_hsde} can be written as
\begin{equation}
\mathrm d\mathbf X(t)=\mathbf F^{\theta}(\mathbf X(t))\mathrm dt+\boldsymbol{\Sigma}^{\phi}(\mathbf X(t))\mathrm d\mathbf W(t).
\label{eq:proof_compact_hsde}
\end{equation}

We divide the proof into five steps.
\paragraph{Step 1: Construction of a maximal local strong solution}
For $R>0$, define the radial projection
\begin{equation}
\pi_R(\mathbf X)=
\begin{cases}
\mathbf X,
&\|\mathbf X\|_F\leq R,\\
\dfrac{R\mathbf X}{\|\mathbf X\|_F},
&\|\mathbf X\|_F>R.
\end{cases}
\end{equation}
The radial projection is nonexpansive:
\begin{equation}
\|\pi_R(\mathbf X)-\pi_R(\mathbf Y)\|_F\leq\|\mathbf X-\mathbf Y\|_F.
\end{equation}
Define the truncated coefficients by
\begin{equation}
\mathbf F^{\theta}_R(\mathbf X)=\mathbf F^{\theta}(\pi_R(\mathbf X)),\qquad \boldsymbol{\Sigma}^{\phi}_R(\mathbf X)=\boldsymbol{\Sigma}^{\phi}(\pi_R(\mathbf X)).
\end{equation}
By Proposition~\ref{prop:coefficient_regularity}, the truncated coefficients are globally Lipschitz. Hence, the truncated equation admits a unique global strong solution. By the standard localization argument, Eq.~\eqref{eq:proof_compact_hsde} admits a unique maximal local strong solution on $[0,\tau_{\infty})$, where
\begin{equation}
\tau_R=\inf\{t\geq0:\|\mathbf X(t)\|_F\geq R\},\qquad\tau_{\infty}=\lim_{R\rightarrow\infty}\tau_R.
\label{eq:proof_explosion_times}
\end{equation}

\paragraph{Step 2: Drift dissipativity and stopped energy estimate}
Since $\mathbf A_{\theta}(\mathbf X)$ is diagonal with positive entries,
\begin{equation}
\begin{aligned}
\left\langle\mathbf X,\mathbf F^{\theta}(\mathbf X)\right\rangle_F
&=-\operatorname{Tr}\left(\mathbf X^{\top}\mathbf G^{\top}\mathbf A_{\theta}(\mathbf X)\mathbf G\mathbf X\right)\\
&=-\|\mathbf A_{\theta}(\mathbf X)^{1/2}\mathbf G\mathbf X\|_F^2\leq0.
\end{aligned}
\label{eq:proof_drift_dissipativity}
\end{equation}
Thus, the deterministic drift is dissipative.

Moreover, the node-wise softmax normalization implies
\begin{equation}
\|\boldsymbol{\Sigma}^{\phi}(\mathbf X)\|_F=\|\mathbf G^{\top}\mathbf B_{\phi}(\mathbf X)\mathbf G\|_F\leq\|\mathbf G\|_F^2.
\label{eq:proof_diffusion_uniform_bound}
\end{equation}

Applying It\^{o}'s formula to the stopped process $\|\mathbf X(t\wedge\tau_R)\|_F^2$ and using Eq.~\eqref{eq:proof_drift_dissipativity}, we obtain
\begin{equation}
\begin{aligned}
&\|\mathbf X(t\wedge\tau_R)\|_F^2\\
&=\|\mathbf X(0)\|_F^2-2\int_0^{t\wedge\tau_R}\|\mathbf A_{\theta}(\mathbf X(s))^{1/2}G\mathbf X(s)\|_F^2\mathrm ds\\
&+d\int_0^{t\wedge\tau_R}\|\boldsymbol{\Sigma}^{\phi}(\mathbf X(s))
\|_F^2\mathrm ds+2M_R(t),
\end{aligned}
\label{eq:proof_stopped_energy_identity}
\end{equation}
where \(M_R(t)=\int_0^{t\wedge\tau_R}\left\langle\mathbf X(s),\boldsymbol{\Sigma}^{\phi}(\mathbf X(s))\mathrm d\mathbf W(s)\right\rangle_F\) is a continuous local martingale.

Taking expectations in Eq.~\eqref{eq:proof_stopped_energy_identity} gives
\begin{equation}
\begin{aligned}
&\mathbb E\left[\|\mathbf X(t\wedge\tau_R)\|_F^2\right]+2\mathbb E\int_0^{t\wedge\tau_R}\|\mathbf A_{\theta}(\mathbf X(s))^{1/2}
\mathbf G\mathbf X(s)\|_F^2\mathrm ds\\
&=\mathbb E\left[\|\mathbf X(0)\|_F^2\right]+d\mathbb E\int_0^{t\wedge\tau_R}\|\boldsymbol{\Sigma}^{\phi}(\mathbf X(s))\|_F^2\mathrm ds.
\end{aligned}
\label{eq:proof_stopped_expected_energy}
\end{equation}
Using Eq.~\eqref{eq:proof_diffusion_uniform_bound}, we obtain
\begin{equation}
\mathbb E\left[\|\mathbf X(t\wedge\tau_R)\|_F^2\right]\leq\mathbb E
\left[\|\mathbf X(0)\|_F^2\right]+dt\|\mathbf G\|_F^4.
\label{eq:proof_stopped_pointwise_bound}
\end{equation}

To obtain the supremum estimate, note that the quadratic variation of $M_R$ satisfies
\begin{equation}
\begin{aligned}
\langle M_R\rangle_t
&=\int_0^{t\wedge\tau_R}\|\boldsymbol{\Sigma}^{\phi}(\mathbf X(s))^{\top}
\mathbf X(s)\|_F^2\mathrm ds\\
&\leq\|\mathbf G\|_F^4\int_0^{t\wedge\tau_R}\|\mathbf X(s)\|_F^2\mathrm ds.
\end{aligned}
\label{eq:proof_energy_martingale_qv}
\end{equation}
The Burkholder--Davis--Gundy and Young inequalities therefore imply
\begin{equation}
\begin{aligned}
2\mathbb E[\sup_{0\leq u\leq t}\|M_R(u)\|]\leq\frac{1}{2}\mathbb E[\sup_{0\leq u\leq t}\|\mathbf X(u\wedge\tau_R)\|_F^2]+C t\|\mathbf G\|_F^4.
\end{aligned}
\label{eq:proof_energy_bdg}
\end{equation}
Taking the supremum in Eq.~\eqref{eq:proof_stopped_energy_identity}, dropping the nonnegative dissipative term, and applying Eq.~\eqref{eq:proof_energy_bdg} yield
\begin{equation}
\mathbb E\left[\sup_{0\leq u\leq T}\|\mathbf X(u\wedge\tau_R)\|_F^2\right]
\leq C_T\left(1+\mathbb E\|\mathbf X(0)\|_F^2\right),
\label{eq:proof_uniform_stopped_bound}
\end{equation}
where $C_T$ is independent of $R$.

\paragraph{Step 3: Absence of finite-time explosion and global existence}
On the event ${\tau_R\leq T}$,
\begin{equation}
R^2\leq\sup_{0\leq t\leq T}\|\mathbf X(t\wedge\tau_R)\|_F^2.
\end{equation}
Hence, by Eq.~\eqref{eq:proof_uniform_stopped_bound},
\begin{equation}
\mathbb P(\tau_R\leq T)\leq\frac{C_T\left(1+\mathbb E\|\mathbf X(0)\|_F^2
\right)}{R^2}.
\end{equation}
Letting $R\rightarrow\infty$ gives \(\mathbb P(\tau_{\infty}\leq T)=0.
\) Since $T>0$ is arbitrary, \(\tau_{\infty}=\infty\) almost surely. Therefore, the maximal local solution is global.

Since $\tau_R\uparrow\infty$ almost surely, Fatou's lemma and Eq.~\eqref{eq:proof_uniform_stopped_bound} give
\begin{equation}
\mathbb E\left[\sup_{0\leq t\leq T}\|\mathbf X(t)\|_F^2\right]\leq C_T
\left(1+\mathbb E\|\mathbf X(0)\|_F^2\right).
\label{eq:proof_global_supremum_bound}
\end{equation}

\paragraph{Step 4: Stochastic energy balance}
Equation~\eqref{eq:proof_stopped_pointwise_bound} and Fatou's lemma imply
\begin{equation}
\mathbb E\|\mathbf X(t)\|_F^2\leq\mathbb E\|\mathbf X(0)\|_F^2+
dt\|\mathbf G\|_F^4.
\label{eq:proof_global_pointwise_bound}
\end{equation}
Furthermore,
\begin{equation}
\begin{aligned}
\mathbb E\int_0^t\|\boldsymbol{\Sigma}^{\phi}(\mathbf X(s))^{\top}
\mathbf X(s)\|_F^2\mathrm ds\leq\|\mathbf G\|_F^4\int_0^t\mathbb E
|\mathbf X(s)|_F^2\mathrm ds<\infty.
\end{aligned}
\end{equation}
Therefore, the stochastic integral in the global It\^{o} formula is a square-integrable martingale with zero expectation. Taking expectations gives
\begin{equation}
\begin{aligned}
\mathbb E\|\mathbf X(t)\|_F^2
&+2\mathbb E\int_0^t\|\mathbf A_{\theta}(\mathbf X(s))^{1/2}\mathbf G\mathbf X(s)\|_F^2\mathrm ds\\
&=\mathbb E\|\mathbf X(0)\|_F^2+d\mathbb E\int_0^t\|\boldsymbol{\Sigma}^{\phi}(\mathbf X(s))\|_F^2\mathrm ds.
\end{aligned}
\label{eq:proof_global_energy_balance}
\end{equation}
This proves the stochastic energy identity and its consequence.

\paragraph{Step 5: Initial-condition stability and pathwise uniqueness}
Let $\mathbf X(t)$ and $\mathbf Y(t)$ be two solutions driven by the same node-space Wiener process, and define \(\mathbf D(t)=\mathbf X(t)-\mathbf Y(t).\) For $R>0$, let
\begin{equation}
\rho_R=\inf\{t\geq0:\|\mathbf X(t)\|_F\vee\|\mathbf Y(t)\|_F\geq R\}.
\end{equation}
Set
\begin{equation}
\begin{aligned}
\Delta\mathbf F^{\theta}(t)
&=\mathbf F^{\theta}(\mathbf X(t))-\mathbf F^{\theta}(\mathbf Y(t))\\
\Delta\boldsymbol{\Sigma}^{\phi}(t)
&=\boldsymbol{\Sigma}^{\phi}(\mathbf X(t))-\boldsymbol{\Sigma}^{\phi}(\mathbf Y(t)).
\end{aligned}
\end{equation}
By Proposition~\ref{prop:coefficient_regularity}, there exists $L_R>0$ such that, before $\rho_R$,
\begin{equation}
\|\Delta\mathbf F^{\theta}(t)\|_F^2+\|\Delta\boldsymbol{\Sigma}^{\phi}(t)
\|_{F}^2\leq L_R\|\mathbf D(t)\|_F^2.
\label{eq:proof_difference_local_lipschitz}
\end{equation}

Applying It\^{o}'s formula to $\|\mathbf D(t\wedge\rho_R)\|_F^2$ gives
\begin{equation}
\begin{aligned}
\|\mathbf D(t\wedge\rho_R)\|_F^2
&=\|\mathbf D(0)\|_F^2+2\int_0^{t\wedge\rho_R}\left\langle\mathbf D(s),
\Delta\mathbf F^{\theta}(s)\right\rangle_F\mathrm ds\\
&+d\int_0^{t\wedge\rho_R}\|\Delta\boldsymbol{\Sigma}^{\phi}(s)
\|_{F}^2\mathrm ds+2N_R(t),
\end{aligned}
\label{eq:proof_difference_ito}
\end{equation}
where \(N_R(t)=\int_0^{t\wedge\rho_R}\left\langle\mathbf D(s),\Delta\boldsymbol{\Sigma}^{\phi}(s)\mathrm d\mathbf W(s)\right\rangle_F.\)

By Young's inequality and Eq.~\eqref{eq:proof_difference_local_lipschitz},
\begin{equation}
2\left\langle\mathbf D,\Delta\mathbf F^{\theta}\right\rangle_F+d\|\Delta\boldsymbol{\Sigma}^{\phi}\|_{F}^2\leq
C_R\|\mathbf D\|_F^2.
\label{eq:proof_difference_drift_bound}
\end{equation}
Moreover, the BDG and Young inequalities give
\begin{equation}
\begin{aligned}
&2\mathbb E\left[\sup_{0\leq u\leq t}\|N_R(u)\|\right]\leq\frac{1}{2}
\mathbb E\left[\sup_{0\leq u\leq t}\|\mathbf D(u\wedge\rho_R)\|_F^2\right]\\
&+C_R\int_0^t\mathbb E\left[\sup_{0\leq r\leq s}\|\mathbf D(r\wedge\rho_R)
\|_F^2\right]\mathrm ds.
\end{aligned}
\label{eq:proof_difference_bdg}
\end{equation}

Taking the supremum in Eq.~\eqref{eq:proof_difference_ito}, taking expectations, and applying Eqs.~\eqref{eq:proof_difference_drift_bound} and \eqref{eq:proof_difference_bdg}, we obtain
\begin{equation}
\begin{aligned}
&\mathbb E\left[\sup_{0\leq u\leq t}\|\mathbf D(u\wedge\rho_R)\|_F^2\right]
\leq2\mathbb E\|\mathbf D(0)\|_F^2\\
&+C_R\int_0^t\mathbb E\left[\sup_{0\leq r\leq s}\|\mathbf D(r\wedge\rho_R)
\|_F^2\right]\mathrm ds.
\end{aligned}
\end{equation}
Gronwall's inequality yields
\begin{equation}
\begin{aligned}
\mathbb E[\sup_{0\leq t\leq T}\|\mathbf X(t\wedge\rho_R)-\mathbf Y(t\wedge\rho_R)\|_F^2]\\
\leq2e^{C_RT}\mathbb E[\|\mathbf X(0)-\mathbf Y(0)\|_F^2].
\end{aligned}
\label{eq:proof_final_initial_stability}
\end{equation}
Thus, the initial-condition stability estimate holds with \(C_{R,T}=2e^{C_RT}.\)

If $\mathbf X(0)=\mathbf Y(0)$ almost surely, then the right-hand side of Eq.~\eqref{eq:proof_final_initial_stability} is zero. Hence,
\begin{equation}
\mathbf X(t\wedge\rho_R)=\mathbf Y(t\wedge\rho_R)
\end{equation}
for all $t\in[0,T]$ on a common event of probability one. Letting $R\rightarrow\infty$ and using the nonexplosion result establishes global pathwise uniqueness. This completes the proof.
\end{proof}

\section{Proof of Theorem~\ref{thm:G}}
\label{appendix:G}
Consider two hypergraphs \(\mathcal G\) and \(\mathcal G'\) defined on the same node set. After consistently aligning their incidence coordinates, let their hypergraph gradient matrices be \(
\mathbf G,\mathbf G' \in\mathbb R^{N\times n}.\)

The two HyperNSD systems share the same trainable parameters and are driven by the same node-space Wiener process \(\mathbf W(t)\):
\begin{equation}
\label{twosystem}
\begin{aligned}
d\mathbf X^{\mathbf G}(t)
&= \mathbf F^{\theta,\mathbf G}
\bigl(\mathbf X^{\mathbf G}(t)\bigr)dt
+ \mathbf \Sigma^{\phi,\mathbf G}
\bigl(\mathbf X^{\mathbf G}(t)\bigr)d\mathbf W(t),\\
d\mathbf X^{\mathbf G'}(t)
&= \mathbf F^{\theta,\mathbf G'}
\bigl(\mathbf X^{\mathbf G'}(t)\bigr)dt
+ \mathbf \Sigma^{\phi,\mathbf G'} \bigl(\mathbf X^{\mathbf G'}(t)\bigr)d\mathbf W(t),
\end{aligned}
\end{equation}

where
\begin{equation}
\begin{aligned}
\mathbf F^{\theta,\mathbf G}(\mathbf X)
&=-\mathbf G^{\top}\mathbf A_{\theta}^{\mathbf G}(\mathbf X)
\mathbf G\mathbf X,\\
\mathbf \Sigma^{\phi,\mathbf G}(\mathbf X)
&= \mathbf G^{\top}\mathbf B_{\phi}^{\mathbf G}(\mathbf X)
\mathbf G.
\end{aligned}
\end{equation}

The notation \(\mathbf A_{\theta}^\mathbf G\) and \(\mathbf B_{\phi}^\mathbf G\) emphasizes that the incidence-wise modulation coefficients depend on the hypergraph structure through the node--hyperedge aggregation contexts. We first prove the following lemma.
\begin{Lemma}
\label{lem}
Let \(\mathbf X^{\mathbf G}(t)\) and
\(\mathbf X^{\mathbf{G'}}(t)\)
be the solutions of Eq.~\eqref{twosystem}, and $K=\max\{\|\mathbf G\|_F,\|\mathbf G'\|_F\}$, if there exist constants \(L_{A,R}>0\) and \(L_{B,R}>0\) such that, for every
\(\mathbf X\) satisfying \(\|\mathbf X\|_{F}\le R\),
\begin{equation}
\begin{aligned}
\|\mathbf A_{\theta}^\mathbf G(\mathbf X)
-\mathbf A_{\theta}^{\mathbf G'}(\mathbf X)\|_{F} \le L_{A,R}\|\mathbf G-\mathbf G'\|_{F},\\
\|\mathbf B_{\phi}^{\mathbf G}(\mathbf X)
-\mathbf B_{\phi}^{\mathbf G'}(\mathbf X)
\|_{F}\le L_{B,R} \|\mathbf G-\mathbf G'\|_{F}.
\end{aligned}
\end{equation}
Then, we have
\begin{equation}
\|\mathbf F^{\theta,\mathbf G}(\mathbf X)-\mathbf F^{\theta,\mathbf G'}(\mathbf X)\|_{F}\le C_{F,R}\|\mathbf G-\mathbf G'\|_{F},
\end{equation}

where \(C_{F,R}=R\left(2K+K^{2}L_{A,R}\right),\) and
\begin{equation}
\|\mathbf \Sigma^{\phi,\mathbf G}(\mathbf X)-\mathbf \Sigma^{\phi,\mathbf G'}(\mathbf X)\|_{F}\le C_{\Sigma,R}
\|\mathbf G-\mathbf G'\|_{F},
\end{equation}
where \(C_{\Sigma,R}=\left(2K+K^{2}L_{B,R}\right).\)
\end{Lemma}

\begin{proof}[Proof of Lemma~\ref{lem}]
Let \(\Delta\mathbf G=\mathbf G-\mathbf G'.\) For the drift coefficient, we have
\begin{equation}
\begin{aligned}
&\mathbf F^{\theta,\mathbf G}(\mathbf X)
-\mathbf F^{\theta,\mathbf G'}(\mathbf X)
\\
&=-\left(\mathbf G^{\top}\mathbf A_{\theta}^{\mathbf G}(\mathbf X)
\mathbf G-\mathbf G'^{\top}
\mathbf A_{\theta}^{\mathbf G'}(\mathbf X)
\mathbf G'\right)\mathbf X.
\end{aligned}
\end{equation}

The operator difference can be decomposed as
\begin{equation}
\begin{aligned}
&\mathbf G^{\top}\mathbf A_{\theta}^{\mathbf G}(\mathbf X)
\mathbf G-\mathbf G'^{\top}
\mathbf A_{\theta}^{\mathbf G'}(\mathbf X)
\mathbf G'=\Delta\mathbf G^{\top}
\mathbf A_{\theta}^{\mathbf G}(\mathbf X)
\mathbf G\\
&+\mathbf G'^{\top}\left(\mathbf A_{\theta}^{\mathbf G}(\mathbf X)
-\mathbf A_{\theta}^{\mathbf G'}(\mathbf X)\right)\mathbf G+\mathbf G'^{\top}
\mathbf A_{\theta}^{\mathbf G'}(\mathbf X)
\Delta\mathbf G.
\end{aligned}
\end{equation}

Using the submultiplicativity of the spectral and Frobenius norms gives
\begin{equation}
\begin{aligned}
&\|\mathbf F^{\theta,\mathbf G}(\mathbf X)
-\mathbf F^{\theta,\mathbf G'}(\mathbf X)
\|_{F}\le\|\Delta\mathbf G\|_{F}
\|\mathbf A_{\theta}^{\mathbf G}(\mathbf X)\|_{2}\|\mathbf G\|_{F}\|\mathbf X\|_{F}\\
&+\|\mathbf G'\|_{F}\|\mathbf A_{\theta}^{\mathbf G}(\mathbf X)
-\mathbf A_{\theta}^{\mathbf G'}(\mathbf X)\|_{F}\|\mathbf G\|_{F}\|\mathbf X\|_{F}\\
&+\|\mathbf G'\|_{F}\|\mathbf A_{\theta}^{\mathbf G'}(\mathbf X)\|_{2}\|\Delta\mathbf G\|_{F}
\|\mathbf X\|_{F}.
\end{aligned}
\end{equation}

Applying assumptions and using
\(\|\mathbf X\|_{F}\le R,\|\mathbf G\|,\|\mathbf G'\|\le K\), we obtain
\begin{equation}
\begin{aligned}
&\|\mathbf F^{\theta,\mathbf G}(\mathbf X)
-\mathbf F^{\theta,\mathbf G'}(\mathbf X)
\|_{F}\\
&\le R\left(K+K^{2}L_{A,R}+K\right)
\|\mathbf G-\mathbf G'\|_{F}\\
&=R\left(2K+K^{2}L_{A,R}
\right)\|\mathbf G-\mathbf G'\|_{F}.
\end{aligned}
\end{equation}
This proves the first estimate.

For the stochastic diffusion coefficient,
\begin{equation}
\begin{aligned}
&\mathbf \Sigma^{\phi, \mathbf G}(\mathbf X)-\mathbf \Sigma^{\phi,\mathbf G'}(\mathbf X)=\Delta\mathbf G^{\top}
\mathbf B_{\phi}^{\mathbf G}(\mathbf X)
\mathbf G\\
&+\mathbf G'^{\top}\left[
\mathbf B_{\phi}^{\mathbf G}(\mathbf X)
-\mathbf B_{\phi}^{\mathbf G'}(\mathbf X)
\right]\mathbf G+\mathbf G'^{\top}
\mathbf B_{\phi}^{\mathbf G'}(\mathbf X)
\Delta\mathbf G.
\end{aligned}
\end{equation}

Therefore,
\begin{equation}
\begin{aligned}
&\|\mathbf \Sigma^{\phi, \mathbf G}(\mathbf X)-\mathbf \Sigma^{\phi, \mathbf G'}(\mathbf X)\|_{F}\le \left(2K+K^{2}L_{B,R}\right)
\|\mathbf G-\mathbf G'\|_{F}.
\end{aligned}
\end{equation}
\end{proof}

\begin{proof}[Proof of Theorem~\ref{thm:G}]
Define \(\mathbf \Delta(t)=\mathbf X^{\mathbf G}(t)-\mathbf X^{\mathbf G'}(t).\) For notational simplicity, let \(\overline{\mathbf \Delta}(t)=\mathbf \Delta(t\wedge\tau_R).\) Using the synchronous coupling of the two systems, we obtain
\begin{equation}
\begin{aligned}
\overline{\mathbf \Delta}(t)=\mathbf \Delta(0)+\int_{0}^{t}\mathbf 1_{{s\le\tau_R}}\mathbf D_{\mathbf F}(s)ds
+\int_{0}^{t}\mathbf 1_{{s\le\tau_R}}
\mathbf D_{\Sigma}(s)d\mathbf W(s),
\end{aligned}
\end{equation}

where
\begin{equation}
\begin{aligned}
\mathbf D_{\mathbf F}(s)
&=\mathbf F^{\theta,\mathbf G}
\bigl(\mathbf X^{\mathbf G}(s)\bigr)
-\mathbf F^{\theta,\mathbf G'}\bigl(\mathbf X^{\mathbf G'}(s)\bigr),\\
\mathbf D_{\Sigma}(s)
&=\mathbf \Sigma_{\mathbf G}^{\phi}
\bigl(\mathbf X^{\mathbf G}(s)\bigr)
-\mathbf \Sigma_{\mathbf G'}^{\phi}
\bigl(\mathbf X^{\mathbf G'}(s)\bigr).
\end{aligned}
\end{equation}

We first split the drift difference into a state perturbation and a structural perturbation:
\begin{equation}
\begin{aligned}
\mathbf D_{\mathbf F}(s)=&
\mathbf F^{\theta,\mathbf G}
\bigl(\mathbf X^{\mathbf G}(s)\bigr)
-\mathbf F^{\theta,\mathbf G}
\bigl(\mathbf X^{\mathbf G'}(s)\bigr)
\\
&+\mathbf F^{\theta,\mathbf G}
\bigl(\mathbf X^{\mathbf G'}(s)\bigr)
-\mathbf F^{\theta,\mathbf G'}
\bigl(\mathbf X^{\mathbf G'}(s)\bigr).
\end{aligned}
\end{equation}

By Lemma~\ref{lem},
\begin{equation}
\begin{aligned}
\|\mathbf D_{\mathbf F}(s)\|_{F}^{2}
\le2L_R\|\overline{\mathbf \Delta}(s)\|_{F}^{2}+2C_{F,R}^{2}\|\mathbf G-\mathbf G'\|_{F}^{2}
\end{aligned}
\end{equation}
for \(s\le\tau_R\).

Similarly,
\begin{equation}
\begin{aligned}
\mathbf D_{\Sigma}(s)&=\mathbf \Sigma_{\mathbf G}^{\phi}
\bigl(\mathbf X^{\mathbf G}(s)\bigr)
-\mathbf \Sigma_{\mathbf G}^{\phi}
\bigl(\mathbf X^{\mathbf G'}(s)\bigr)
\\
&+\mathbf \Sigma_{\mathbf G}^{\phi}
\bigl(\mathbf X^{\mathbf G'}(s)\bigr)
-\mathbf \Sigma_{\mathbf G'}^{\phi}
\bigl(\mathbf X^{\mathbf G'}(s)\bigr),
\end{aligned}
\end{equation}

which gives
\begin{equation}
\begin{aligned}
|\mathbf D_{\Sigma}(s)|_{F}^{2}\le 2L_R
\|\overline{\mathbf \Delta}(s)\|_{F}^{2}
+ 2C_{\Sigma,R}^{2}\|\mathbf G-\mathbf G'\|_{F}^{2}.
\end{aligned}
\end{equation}

Using the inequality
\begin{equation}
\|\mathbf a+\mathbf b+\mathbf c\|_{F}^{2}
\le3\|\mathbf a\|_{F}^{2}
+3\|\mathbf b\|_{F}^{2}
+3\|\mathbf c\|_{F}^{2},
\end{equation}

we obtain, for \(t\le T\),
\begin{equation}
\begin{aligned}
&\mathbb E\left[\sup_{0\le u\le t}
\|\overline{\mathbf \Delta}(u)\|_{F}^{2}
\right]\le3\mathbb E\|\mathbf \Delta(0)\|_{F}^{2}\\
&+3\mathbb E\left[\sup_{0\le u\le t}
\|\int_{0}^{u}\mathbf 1_{{s\le\tau_R}}
\mathbf D_{\mathbf F}(s)ds\|_{F}^{2}\right]\\
&+3\mathbb E\left[\sup_{0\le u\le t}
\|\int_{0}^{u}\mathbf 1_{{s\le\tau_R}}
\mathbf D_{\Sigma}(s)d\mathbf W(s)
\|_{F}^{2}\right].
\end{aligned}
\end{equation}

By the Cauchy--Schwarz inequality,
\begin{equation}
\begin{aligned}
\sup_{0\le u\le t}\|\int_{0}^{u}
\mathbf 1_{{s\le\tau_R}} \mathbf D_{\mathbf F}(s)ds \|_{F}^{2}\le t
\int_{0}^{t} \mathbf 1_{{s\le\tau_R}}
\|\mathbf D_{\mathbf F}(s)\|_{F}^{2}ds.
\end{aligned}
\end{equation}

By the Burkholder--Davis--Gundy inequality, there exists a universal constant \(C_{\mathrm{BDG}}>0\) such that
\begin{equation}
\begin{aligned}
&\mathbb E \left[\sup_{0\le u\le t}
\|\int_{0}^{u}\mathbf 1_{{s\le\tau_R}}
\mathbf D_{\Sigma}(s)d\mathbf W(s)
\|_{F}^{2}\right]\\
&\le C_{\mathrm{BDG}}d \mathbb E
\int_{0}^{t} \mathbf 1_{{s\le\tau_R}}
\|\mathbf D_{\Sigma}(s)\|_{F}^{2}ds.
\end{aligned}
\end{equation}

Define
\begin{equation}
\Phi(t) = \mathbb E\left[\sup_{0\le u\le t} \|\overline{\mathbf \Delta}(u)\|_{F}^{2}\right].
\end{equation}

Combining the preceding estimates and using \(t\le T\), we obtain constants \(C_{1,R,T}>0\) and \(C_{2,R,T}>0\) such that
\begin{equation}
\begin{aligned}
\Phi(t) \le 3\mathbb E\|\mathbf \Delta(0)\|_{F}^{2}+C_{1,R,T}
\int_{0}^{t}\Phi(s)ds+C_{2,R,T}
t\|\mathbf G-\mathbf G'\|_{F}^{2}.
\end{aligned}
\end{equation}

Since \(t\le T\) and initial condition equal, this implies
\begin{equation}
\Phi(t)\le C_{1,R,T}\int_{0}^{t}\Phi(s)ds
+C_{2,R,T}T \|\mathbf G-\mathbf G'\|_{F}^{2}.
\end{equation}

Applying Grönwall's inequality yields
\begin{equation}
\begin{aligned}
\Phi(T)\le C_{2,R,T}T\|\mathbf G-\mathbf G'\|_{F}^{2}\exp(C_{1,R,T}T).
\end{aligned}
\end{equation}

Absorbing all constants and the relevant Lipschitz constants into \(C^G_{R,T}\), we obtain
\begin{equation}
\begin{aligned}
& \mathbb E[\sup_{0\le t\le T}\|\mathbf X^{\mathbf G}(t\wedge\tau_R)-\mathbf X^{\mathbf G'}(t\wedge\tau_R)\|_{F}^{2}
]\le C_{R,T} \|\mathbf G-\mathbf G'\|_{F}^{2}.
\end{aligned}
\end{equation}
\end{proof}

\section{Proof of Theorem~\ref{thm:conservation_covariance}}
\label{appendix:theorem2}
\begin{proof}[Proof of Theorem~\ref{thm:conservation_covariance}]
Recall that the Eq.~\eqref{eq:node_driven_hsde} is given by
\begin{equation}
\mathrm d\mathbf X(t)=-\mathbf G^{\top}\mathbf A_{\theta}(\mathbf X(t))
\mathbf G\mathbf X(t)\mathrm dt+\boldsymbol{\Sigma}^{\phi}(\mathbf X(t)) \mathrm d\mathbf W(t),
\label{eq:proof_conservation_hsde}
\end{equation}
\paragraph{Part (i): Conservation of null-space modes}
 Since \(\operatorname{Range}(\mathbf G^{\top})=(\ker \mathbf G)^{\perp},\) the projection $\operatorname{Proj}$ annihilates every vector in
$\operatorname{Range}(\mathbf G^{\top})$. Therefore,
\begin{equation}
\operatorname{Proj}\mathbf G^{\top}=0.
\label{eq:projection_annihilates_gt}
\end{equation}

Applying $\operatorname{Proj}$ to both sides of Eq.~\eqref{eq:proof_conservation_hsde} yields
\begin{equation}
\begin{aligned}
\mathrm d\bigl(\operatorname{Proj}\mathbf X(t)\bigr)
&=-\operatorname{Proj}\mathbf G^{\top}\mathbf A_{\theta}(\mathbf X(t))\mathbf G\mathbf X(t)\mathrm dt\\
&+\operatorname{Proj}\mathbf G^{\top}\mathbf B_{\phi}(\mathbf X(t))
\mathbf G\mathrm d\mathbf W(t).
\end{aligned}
\end{equation}
Using Eq.~\eqref{eq:projection_annihilates_gt}, both terms on the
right-hand side vanish. Hence, \(\mathrm d\bigl(\operatorname{Proj}\mathbf X(t)\bigr)=0.\) Integrating from $0$ to $t$ gives
\begin{equation}
\operatorname{Proj}\mathbf X(t)=\operatorname{Proj}\mathbf X(0)
\end{equation}
for every $t\geq0$, almost surely.

\paragraph{Part (ii): Structure-constrained instantaneous covariance}
Let $\mathbf X_{\ell}(t)$ and $\mathbf W_{\ell}(t)$ denote the
$\ell$-th feature columns of $\mathbf X(t)$ and $\mathbf W(t)$,
respectively. The stochastic component of the $\ell$-th channel is
\begin{equation}
\mathrm d\mathbf X_{\ell}(t)=\cdots+\boldsymbol{\Sigma}^{\phi}(\mathbf X(t))
\mathrm d\mathbf W_{\ell}(t),
\end{equation}
where $\mathbf W_{\ell}(t)$ is an $n$-dimensional standard Wiener process. Since
\begin{equation}
\operatorname{Cov}\left(\mathrm d\mathbf W_{\ell}(t)\mid\mathbf X(t)\right)=\mathbf I_n\mathrm dt,
\end{equation}
the conditional instantaneous covariance is
\begin{equation}
\begin{aligned}
&\operatorname{Cov}\left(\mathrm d\mathbf X_{\ell}(t)\mid\mathbf X(t)\right)=\boldsymbol{\Sigma}^{\phi}(\mathbf X(t))\boldsymbol{\Sigma}^{\phi}(\mathbf X(t))^{\top}\mathrm dt.
\end{aligned}
\end{equation}
Because $\mathbf B_{\phi}(\mathbf X)$ is diagonal, $\boldsymbol{\Sigma}^{\phi}(\mathbf X)$ is symmetric; therefore we obtain
\begin{equation}
\operatorname{Cov}\left(\mathrm d\mathbf X_{\ell}(t)\mid\mathbf X(t)\right)
=\boldsymbol{\Sigma}^{\phi}(\mathbf X(t))^2\mathrm dt.
\label{eq:instantaneous_covariance_proof}
\end{equation}

For any vector $\mathbf z\in\mathbb R^n$,
\begin{equation}
\boldsymbol{\Sigma}^{\phi}(\mathbf X)\mathbf z
=\mathbf G^{\top}\mathbf B_{\phi}(\mathbf X)\mathbf G\mathbf z\in\operatorname{Range}(\mathbf G^{\top}).
\end{equation}
It follows that
\begin{equation}
\operatorname{Range}\left(\boldsymbol{\Sigma}^{\phi}(\mathbf X)
\right)\subseteq\operatorname{Range}(\mathbf G^{\top}).
\end{equation}

Moreover, if $\mathbf z\in\ker \mathbf G$, then $\mathbf G\mathbf z=0$, and hence
\begin{equation}
\boldsymbol{\Sigma}^{\phi}(\mathbf X)\mathbf z=\mathbf G^{\top}\mathbf B_{\phi}(\mathbf X)\mathbf G\mathbf z=0.
\end{equation}
Therefore,
\begin{equation}
\ker \mathbf G\subseteq\ker\boldsymbol{\Sigma}^{\phi}(\mathbf X).
\end{equation}

\paragraph{Part (iii): Equality of structural modes}
We already know that \(\ker \mathbf G\subseteq\ker\boldsymbol{\Sigma}^{\phi}(\mathbf X).\) To prove the reverse inclusion, let \(\mathbf z\in\ker\boldsymbol{\Sigma}^{\phi}(\mathbf X).\) Then \(\mathbf G^{\top}
\mathbf B_{\phi}(\mathbf X)\mathbf G\mathbf z=0.\)

Taking the Euclidean inner product with $\mathbf z$ gives
\begin{equation}
\begin{aligned}
0=\mathbf z^{\top}\mathbf G^{\top}\mathbf B_{\phi}(\mathbf X)\mathbf G\mathbf z=\|\mathbf B_{\phi}(\mathbf X)^{1/2}\mathbf G\mathbf z\|_2^2.
\end{aligned}
\end{equation}
Since $\mathbf B_{\phi}(\mathbf X)^{1/2}$ is invertible, this implies
\(\mathbf G\mathbf z=0.\) Thus, $\mathbf z\in\ker \mathbf G$, and consequently,
\begin{equation}
\ker\boldsymbol{\Sigma}^{\phi}(\mathbf X)=\ker \mathbf G.
\label{eq:equal_diffusion_gradient_kernel}
\end{equation}

Since $\boldsymbol{\Sigma}^{\phi}(\mathbf X)\in\mathbb R^{n\times n}$
and $\mathbf G\in\mathbb R^{N\times n}$ have the same null space in
$\mathbb R^n$, the rank--nullity theorem gives
\begin{equation}
\begin{aligned}
\operatorname{rank}
\boldsymbol{\Sigma}^{\phi}(\mathbf X)
&=n-\dim\ker\boldsymbol{\Sigma}^{\phi}(\mathbf X)\\
&=n-\dim\ker \mathbf G=\operatorname{rank}\mathbf G.
\end{aligned}
\end{equation}

Finally, suppose that the hypergraph is connected and that \(\ker \mathbf G
=\operatorname{span}\{\mathbf{D_v}^{1/2}\mathbf 1_n\}.\) Let \(\mathbf q = \mathbf{D_v}^{1/2}\mathbf 1_n.\) Then $\mathbf q\in\ker \mathbf G$, and hence \(\mathbf G\mathbf q=0.\) Because $\mathbf q$ belongs to the range of $\operatorname{Proj}$, Part (i) implies that the component of each feature channel along $\mathbf q$ is conserved. Equivalently,
\begin{equation}
\mathbf q^{\top}\mathbf X(t)=\mathbf q^{\top}\mathbf X(0)
\end{equation}
for every $t\geq0$, almost surely. Substituting $\mathbf q=\mathbf{D_v}^{1/2}\mathbf 1_n$ gives
\begin{equation}
\left(\mathbf{D_v}^{1/2}\mathbf 1_n\right)^{\top}\mathbf X(t)=\left(\mathbf{D_v}^{1/2}\mathbf 1_n\right)^{\top}\mathbf X(0)
\end{equation}
almost surely. This completes the proof.
\end{proof}

\section{Proof of Proposition~\ref{prop:permutation_equivariance}}
\label{appendix:Proposition3}
\begin{proof}[Proof of Proposition~\ref{prop:permutation_equivariance}]
We define \(\widetilde{\mathbf X}(t)=\Pi\mathbf X_{\mathcal G}(t)\) and \(
\mathbf W^{\pi}(t)=\Pi\mathbf W(t).\) For each feature dimension, $\mathbf W_{\ell}^{\pi}(t)=\Pi\mathbf W_{\ell}(t)$. Since $\Pi$ is orthogonal,
\begin{equation}
\operatorname{Cov}\left(\mathbf W_{\ell}^{\pi}(t)\right)= t\Pi\Pi^{\top}=t\mathbf I_n.
\end{equation}
Thus, $\mathbf W^{\pi}(t)$ is also a node-space standard Wiener process.

Multiplying the HyperNSD equation on $\mathcal G$ by $\Pi$ gives
\begin{equation}
\begin{aligned}
\mathrm d\widetilde{\mathbf X}(t)=\Pi\mathbf F_{\mathcal G}^{\theta}\left(
\mathbf X_{\mathcal G}(t)\right)\mathrm dt+\Pi\boldsymbol{\Sigma}_{\mathcal G}^{\phi}\left(\mathbf X_{\mathcal G}(t)\right)\mathrm d\mathbf W(t).
\end{aligned}
\end{equation}
By Eq.~\eqref{eq:permuted_HNSD_coefficients},
\begin{equation}
\begin{aligned}
\mathrm d\widetilde{\mathbf X}(t)=\mathbf F_{\mathcal G^{\pi}}^{\theta}\left(\widetilde{\mathbf X}(t)\right)\mathrm dt+\boldsymbol{\Sigma}_{\mathcal G^{\pi}}^{\phi}\left(\widetilde{\mathbf X}(t)\right)\mathrm d\mathbf W^{\pi}(t),
\end{aligned}
\end{equation}
where \(\widetilde{\mathbf X}(0)=\Pi\mathbf X(0).\)

Therefore, $\widetilde{\mathbf X}(t)$ satisfies the HyperNSD equation on the relabeled hypergraph $\mathcal G^{\pi}$. By pathwise uniqueness,
\begin{equation}
\mathbf X_{\mathcal G^{\pi}}\left(t;\Pi\mathbf X(0),\mathbf W^{\pi}\right)=\widetilde{\mathbf X}(t)=\Pi\mathbf X_{\mathcal G}\left(t;\mathbf X(0),\mathbf W\right)
\end{equation}
almost surely. Since $\mathbf W^{\pi}$ and $\mathbf W$ have the same distribution, the distributional equivariance in
Eq.~\eqref{eq:distributional_permutation_equivariance} follows.
\end{proof}

\section{Proof of Proposition~\ref{prop:euler_maruyama_convergence}}
\label{appendix:P3}
\begin{proof}[Proof of Proposition~\ref{prop:euler_maruyama_convergence}]
Define the piecewise-constant interpolation \(\widehat{\mathbf X}^{h}(t)=\mathbf X_k^h,\;t\in[t_k,t_{k+1}),\) so that the continuous Euler--Maruyama interpolation satisfies
\begin{equation}
\begin{aligned}
\overline{\mathbf X}^{h}(t)=\mathbf X(0)
&+\int_0^t\mathbf F^{\theta}\left(\widehat{\mathbf X}^{h}(s)\right)\mathrm ds\\
&+\int_0^t\boldsymbol{\Sigma}^{\phi}\left(\widehat{\mathbf X}^{h}(s)\right)
\mathrm d\mathbf W(s).
\end{aligned}
\label{eq:proof_em_interpolation}
\end{equation}

For notational convenience, write \(\eta_R=\eta_R^h.\) Before the exit time $\eta_R$, the processes $\mathbf X(t)$, $\overline{\mathbf X}^{h}(t)$, and $\widehat{\mathbf X}^{h}(t)$ remain in the closed ball of radius $R$. Indeed, if $t\in[t_k,t_{k+1})$ and $t\leq\eta_R$, then \(\widehat{\mathbf X}^{h}(t)=\mathbf X_k^h=\overline{\mathbf X}^{h}(t_k),\) and $t_k\leq t\leq\eta_R$.

By Proposition~\ref{prop:coefficient_regularity}, there exists a constant $L_R>0$ such that
\begin{equation}
\begin{aligned}
\|\mathbf F^{\theta}(\mathbf U)-\mathbf F^{\theta}(\mathbf V)\|_F^2
&+\|\boldsymbol{\Sigma}^{\phi}(\mathbf U)-\boldsymbol{\Sigma}^{\phi}(\mathbf V)\|_{F}^2\\
&\leq L_R\|\mathbf U-\mathbf V\|_F^2
\end{aligned}
\label{eq:proof_em_local_lipschitz}
\end{equation}
for all $|\mathbf U|_F,|\mathbf V|_F\leq R$.

\paragraph{Local interpolation error}
For all $s\in[t_k,t_{k+1})$, Eq.~\eqref{eq:proof_em_interpolation} gives
\begin{equation}
\begin{aligned}
\overline{\mathbf X}^{h}(s)-\widehat{\mathbf X}^{h}(s)
&=\int_{t_k}^{s}\mathbf F^{\theta}\left(\widehat{\mathbf X}^{h}(r)\right)
\mathrm dr\\
&+\int_{t_k}^{s}\boldsymbol{\Sigma}^{\phi}\left(\widehat{\mathbf X}^{h}(r)
\right)\mathrm d\mathbf W(r).
\end{aligned}
\label{eq:proof_em_local_increment}
\end{equation}
On the stopped interval, the coefficients are bounded on the ball of radius $R$. Therefore, the Cauchy--Schwarz inequality and It\^{o}'s isometry imply
\begin{equation}
\begin{aligned}
&\mathbb E\left[\mathbb{I}({{s\leq\eta_R}})\|\overline{\mathbf X}^{h}(s)-
\widehat{\mathbf X}^{h}(s)\|_F^2\right]\\
&\leq C_R(s-t_k)^2+C_R(s-t_k) \leq C_Rh,
\end{aligned}
\label{eq:proof_em_increment_bound}
\end{equation}
where $C_R>0$ is independent of $h$. Integrating over $[0,T]$ yields
\begin{equation}
\mathbb E
\int_0^{T\wedge\eta_R}\|\overline{\mathbf X}^{h}(s)-\widehat{\mathbf X}^{h}(s)\|_F^2\mathrm ds\leq C^E_{R,T}h.
\label{eq:proof_em_integrated_increment}
\end{equation}

\paragraph{Step 2: Stopped strong-error estimate}
Define the error process \(\mathbf E^{h}(t)=\mathbf X(t)-\overline{\mathbf X}^{h}(t).\) Subtracting Eq.~\eqref{eq:proof_em_interpolation} from the integral form of Eq.~\eqref{eq:node_driven_hsde} gives
\begin{equation}
\begin{aligned}
\mathbf E^{h}(t\wedge\eta_R)
&=\int_0^{t\wedge\eta_R}\left[\mathbf F^{\theta}(\mathbf X(s))-\mathbf F^{\theta}\left(\widehat{\mathbf X}^{h}(s)\right)\right]\mathrm ds\\
&+\int_0^{t\wedge\eta_R}\left[\boldsymbol{\Sigma}^{\phi}(\mathbf X(s))-\boldsymbol{\Sigma}^{\phi}\left(\widehat{\mathbf X}^{h}(s)\right)\right]\mathrm d\mathbf W(s).
\end{aligned}
\label{eq:proof_em_error_equation}
\end{equation}

Using the Cauchy--Schwarz inequality for the drift integral and the Burkholder--Davis--Gundy inequality for the stochastic integral, we obtain
\begin{equation}
\begin{aligned}
&\mathbb E
\left[\sup_{0\leq u\leq t\wedge\eta_R}\|\mathbf E^{h}(u)\|_F^2\right]\\
&\leq C_T\mathbb E\int_0^{t\wedge\eta_R}\|\mathbf F^{\theta}(\mathbf X(s))-
\mathbf F^{\theta}\left(\widehat{\mathbf X}^{h}(s)\right)\|_F^2\mathrm ds\\
&+C\mathbb E\int_0^{t\wedge\eta_R}\|\boldsymbol{\Sigma}^{\phi}(\mathbf X(s))
-\boldsymbol{\Sigma}^{\phi}\left(\widehat{\mathbf X}^{h}(s)\right)\|_{F}^2
\mathrm ds.
\end{aligned}
\label{eq:proof_em_bdg_estimate}
\end{equation}
Applying Eq.~\eqref{eq:proof_em_local_lipschitz} gives
\begin{equation}
\begin{aligned}
&\mathbb E\left[\sup_{0\leq u\leq t\wedge\eta_R}\|\mathbf E^{h}(u)\|_F^2
\right]\\
&\leq C^E_{R,T}\mathbb E\int_0^{t\wedge\eta_R}\|\mathbf X(s)-\widehat{\mathbf X}^{h}(s)\|_F^2\mathrm ds.
\end{aligned}
\label{eq:proof_em_lipschitz_error}
\end{equation}
Since \(\mathbf X(s)-\widehat{\mathbf X}^{h}(s)=\mathbf E^{h}(s)+\overline{\mathbf X}^{h}(s)-\widehat{\mathbf X}^{h}(s),\) we have
\begin{equation}
\begin{aligned}
&\mathbb E\left[\sup_{0\leq u\leq t\wedge\eta_R}\|\mathbf E^{h}(u)\|_F^2 \right]\\
&\leq C^E_{R,T}\int_0^t\mathbb E\left[\sup_{0\leq r\leq s\wedge\eta_R}\|
\mathbf E^{h}(r)\|_F^2\right]\mathrm ds\\
&+C^E_{R,T}\mathbb E\int_0^{t\wedge\eta_R}\|\overline{\mathbf X}^{h}(s)-
\widehat{\mathbf X}^{h}(s)\|_F^2\mathrm ds.
\end{aligned}
\label{eq:proof_em_gronwall_preparation}
\end{equation}
Using Eq.~\eqref{eq:proof_em_integrated_increment}, we obtain
\begin{equation}
\begin{aligned}
&\mathbb E\left[\sup_{0\leq u\leq t\wedge\eta_R}\|\mathbf E^{h}(u)\|_F^2
\right]\\
&\leq C^E_{R,T}\int_0^t\mathbb E\left[\sup_{0\leq r\leq s\wedge\eta_R}\|\mathbf E^{h}(r)\|_F^2\right]\mathrm ds+C^E_{R,T}h.
\end{aligned}
\end{equation}
Gronwall's inequality then yields
\begin{equation}
\mathbb E\left[\sup_{0\leq t\leq T\wedge\eta_R^h}\|\mathbf X(t)-\overline{\mathbf X}^{h}(t)\|_F^2\right]\leq C^E_{R,T}h.
\end{equation}

\paragraph{Step 3: Global convergence in probability}
The linear-growth condition in Proposition~\ref{prop:coefficient_regularity}, together with the BDG and Gronwall inequalities, gives a uniform moment estimate for the Euler--Maruyama interpolation:
\begin{equation}
\sup_{0<h\leq1}\mathbb E\left[\sup_{0\leq t\leq T}\|\overline{\mathbf X}^{h}(t)\|_F^2\right]\leq C_T\left(1+\mathbb E\|\mathbf X(0)\|_F^2\right).
\label{eq:proof_em_uniform_moment}
\end{equation}
Hence,
\begin{equation}
\begin{aligned}
&\mathbb P\left(\eta_R^h\leq T\right)\\
&\leq\mathbb P\left(\sup_{0\leq t\leq T}\|\mathbf X(t)\|_F\geq R\right)+\mathbb P\left(\sup_{0\leq t\leq T}\|\overline{\mathbf X}^{h}(t)\|_F\geq R\right)\\
&\leq\frac{C_T\left(1+\mathbb E\|\mathbf X(0)\|_F^2\right)}{R^2},
\end{aligned}
\label{eq:proof_em_exit_probability}
\end{equation}
uniformly in $h$.

For any $\varepsilon>0$,
\begin{equation}
\begin{aligned}
&\mathbb P\left(\sup_{0\leq t\leq T}\|\mathbf X(t)-\overline{\mathbf X}^{h}(t)\|_F>\varepsilon\right)\\
&\leq\mathbb P\left(\eta_R^h\leq T\right)+\mathbb P\left(\sup_{0\leq t\leq T\wedge\eta_R^h}\|\mathbf X(t)-\overline{\mathbf X}^{h}(t)\|_F>\varepsilon\right)\\
&\leq\frac{C_T\left(1+\mathbb E\|\mathbf X(0)\|_F^2\right)}{R^2}+\frac{C^E_{R,T}h}{\varepsilon^2},
\end{aligned}
\label{eq:proof_em_probability_bound}
\end{equation}
where the last inequality follows from Markov's inequality and the stopped strong-error estimate.

For fixed $R$, letting $h\rightarrow0$ eliminates the second term. Subsequently letting $R\rightarrow\infty$ eliminates the first term. Therefore,
\begin{equation}
\sup_{0\leq t\leq T}\|\mathbf X(t)-\overline{\mathbf X}^{h}(t)\|_F\longrightarrow0
\end{equation}
in probability as $h\rightarrow0$. This completes the proof.
\end{proof}

\section{Using uncertainty for Detection}
\label{usingUD}
Our experiments consider two uncertainty-aware tasks: OOD detection and misclassification detection. Taking OOD detection as an example, we construct OOD samples through label leave-out, feature interpolation, or structure manipulation, as described in Section~\ref{experiments}. HyperNSD then generates multiple stochastic representation trajectories for each ID and OOD node, and quantifies node-level uncertainty from the variability of their terminal representations. Since the model is trained on ID data, ID representations are generally more stable and exhibit lower uncertainty, whereas OOD nodes contain unseen semantic, attribute, or structural patterns and therefore tend to produce greater trajectory variability. The resulting uncertainty scores are used to distinguish ID from OOD nodes. Misclassification detection follows an analogous uncertainty-ranking principle, using aleatoric uncertainty as the detection score.

Because the pathwise OOD score is estimated directly in the node-representation space, it is not intrinsically tied to a classification decoder and can, in principle, be adapted to regression and representation-based anomaly detection. In contrast, many deterministic confidence-based methods rely on classification logits or probabilities and require task-specific modifications for non-classification settings. This representation-level formulation therefore improves the flexibility and general applicability of HyperNSD.

\section{Runtime Comparison}
\label{run}
We provide an empirical comparison of the per-epoch training time between HyperNSD and nine representative baselines under the label leave-out setting on Cora. As shown in Fig.~\ref{runtime}, deterministic hypergraph models such as HGNN, HyperGCN, HND, and HNDiffN require relatively low computational cost, whereas stochastic diffusion methods introduce additional overhead due to numerical integration and stochastic trajectory generation. HyperNSD requires 0.76 seconds per epoch, which is higher than GNSD and LGNSDE because it performs both deterministic and stochastic propagation directly over the node--hyperedge incidence domain. Nevertheless, its runtime remains below that of HGNN-Ensemble and stays within one second per epoch. These results demonstrate that HyperNSD achieves reliable uncertainty estimation and higher-order stochastic modeling with an acceptable computational cost.
\begin{figure}[ht]
    \centering
    \includegraphics[scale=0.4]{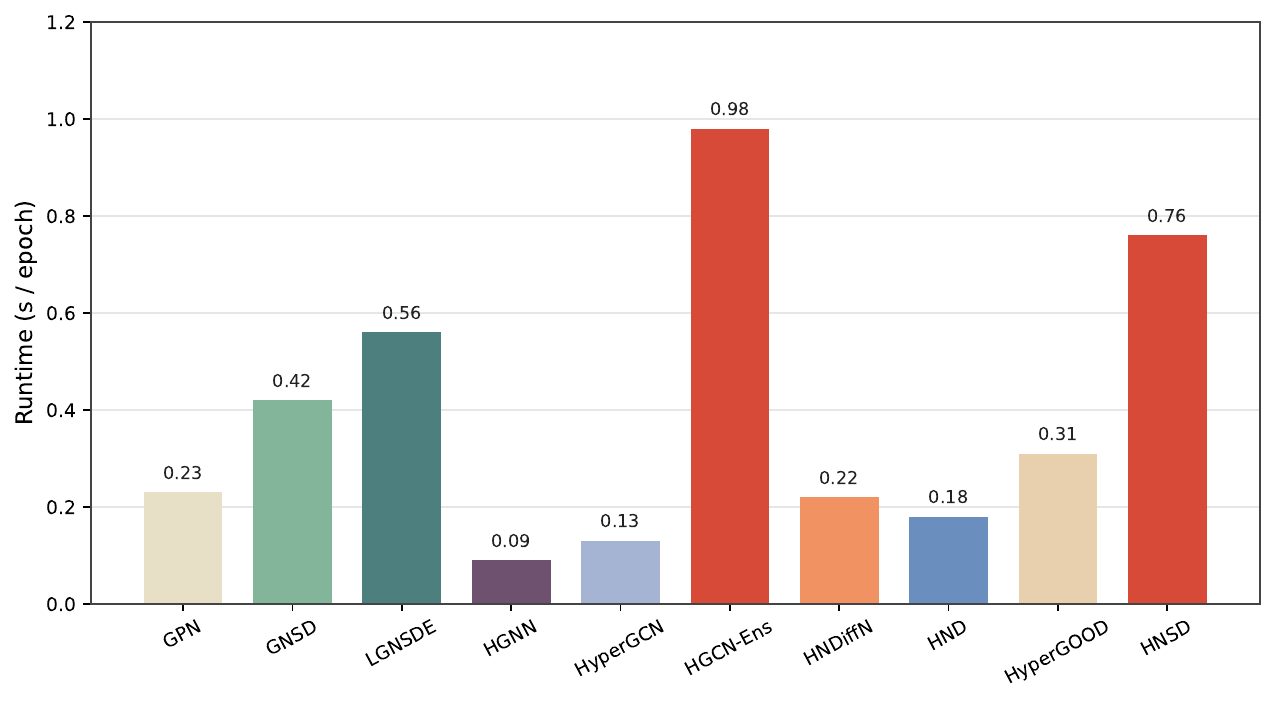}
   \caption{Runtime comparison of one epoch (in seconds) on Cora.}
\label{runtime}
\end{figure}

\section{Experimental Details}
\label{appendix:experiment}
\subsection{Experimental Environment}
All experiments were conducted on a Linux server running Ubuntu 22.04, equipped with an Intel(R) Xeon(R) E5-2699 v4 CPU at 2.20 GHz, 384 GB of RAM, and an NVIDIA GeForce RTX 3090 GPU with 24 GB of memory. HyperNSD and all baseline methods were implemented in Python 3.10 using PyTorch 2.6.0.
\subsection{Dataset Information}
\label{dataset}
The datasets used in our experiments are all publicly available as common benchmarks for evaluating hypergraph learning models.
\begin{itemize}
    \item Cora~\cite{sen2008collective,yadati2019hypergcn} is a co-citation hypergraph constructed from a scientific publication network. Each node represents a paper, while each hyperedge connects the papers cited by the same publication. The task is to predict the research topic of each paper. The dataset contains 2708 nodes, 1579 hyperedges, 1433-dimensional bag-of-words features, and 7 classes.
    \item Cora-CA~\cite{yadati2019hypergcn} is a co-authorship hypergraph constructed from the Cora publication corpus. Each node denotes a paper, and each author defines a hyperedge connecting all papers written by that author. The prediction task is to classify each paper into its corresponding research topic. The dataset contains 2708 nodes, 1072 hyperedges, 1433 features, and 7 classes.
    \item Citeseer~\cite{sen2008collective,yadati2019hypergcn} is a co-citation hypergraph of scientific publications. Each node corresponds to a paper, and each hyperedge groups the papers cited by the same publication. Node attributes are represented by bag-of-words features, and the task is to predict the topic of each paper. The dataset contains 3312 nodes, 1079 hyperedges, 3703 features, and 6 classes.
    \item DBLP~\cite{yadati2019hypergcn} is a co-authorship hypergraph constructed from computer science publications. Each node represents a paper, while each author forms a hyperedge connecting all papers authored by that individual. The task is to classify papers into different research areas. The dataset contains 41302 nodes, 22363 hyperedges, 1425 features, and 6 classes.
    \item ModelNet40~\cite{wu20153d} is a large-scale 3D object dataset containing CAD models from 40 object categories. Each 3D object is treated as a node, and hyperedges are constructed by connecting each object with its 10 nearest neighbors in the visual feature space. The dataset contains 12311 nodes, 24622 hyperedges.
    \item NTU2012~\cite{chen2003visual} is a 3D shape dataset containing objects from 67 semantic categories. Each object is represented as a node, and hyperedges are generated by grouping each object with its 10 nearest neighbors according to visual-feature similarity. The dataset contains 2012 nodes, 4024 hyperedges.
\end{itemize}

\subsection{Baseline Descriptions}
\label{baseline}
\begin{itemize}
    \item MSP~\cite{hendrycks2016baseline} uses the maximum probability produced by the softmax classifier as the confidence score. It is based on the observation that correctly classified ID samples generally receive higher maximum softmax probabilities than misclassified or OOD samples.
    \item ODIN~\cite{liang2017enhancing} improves softmax-based OOD detection through temperature scaling and small input perturbations. These operations enlarge the separation between the confidence-score distributions of ID and OOD samples without modifying the pretrained classifier.
    \item Mahalanobis~\cite{lee2018simple} models the intermediate representations of each class using class-conditional Gaussian distributions. The Mahalanobis distance between a test representation and its closest class distribution is then used to derive the OOD confidence score.
    \item GNNSafe~\cite{wu2023energy} is an energy-based OOD detection method designed for graph-structured data. It extracts node-level energy scores from a GNN trained with the standard classification objective and further propagates these scores over the graph to exploit relational dependencies. Since auxiliary OOD samples are unavailable in our setting, we adopt its variant without OOD exposure.
    \item GPN~\cite{stadler2021graph} extends posterior networks to graph-structured data and represents node predictions using Dirichlet distributions. It estimates class-specific pseudo-counts through density estimation and propagates the resulting evidence across neighboring nodes to quantify predictive uncertainty.
    \item GNSD~\cite{lin2024graph} formulates graph representation learning as a stochastic diffusion process driven by a structure-dependent Wiener process. It estimates predictive uncertainty from multiple stochastic representation trajectories, thereby modeling how uncertainty evolves and propagates over pairwise graph structures.
    \item LGNSDE~\cite{bergna2025uncertainty} extends graph neural ordinary differential equations to latent graph neural stochastic differential equations. By introducing Brownian perturbations into the continuous latent dynamics, it generates stochastic node representations and quantifies uncertainty from their predictive variability.
    \item HGNN~\cite{feng2019hypergraph} introduces a spectral hypergraph convolution operator based on the normalized hypergraph Laplacian. It propagates information through node--hyperedge--node interactions and captures higher-order dependencies encoded by the incidence structure.
    \item HyperGCN~\cite{yadati2019hypergcn} approximates each hyperedge using a set of representative pairwise connections and subsequently applies GCN-style propagation to the resulting graph. This construction enables conventional graph convolution to process hypergraph-structured data.
    \item HGCN-Ens~\cite{lakshminarayanan2017simple} trains 10 independently initialized HyperGCN models and aggregates their predictive distributions. The disagreement among the ensemble members is used to estimate predictive uncertainty, at the cost of repeated training and increased model storage.
    \item HNDiffN~\cite{lu2025hypergraph} interprets hypergraph representation learning from a diffusion perspective and derives neural propagation layers from hypergraph diffusion dynamics. It provides a deterministic diffusion-based baseline for modeling higher-order feature propagation.
    \item HND~\cite{zhou2026hypergraph} formulates hypergraph message passing as a nonlinear anisotropic diffusion process defined through hypergraph gradient and divergence operators. Learnable incidence-level coefficients adaptively regulate deterministic feature transport between nodes and hyperedges.
    \item HyperGOOD~\cite{cai2026hypergood} is an energy-based framework specifically developed for OOD detection on hypergraphs. It combines multi-scale spectral decomposition with structure-aware uncertainty propagation to identify samples whose higher-order relational contexts differ from those observed during training.
\end{itemize}

\subsection{Evaluation Metrics}
\label{metrics}
Following prior studies~\cite{kong2020sde,stadler2021graph,wu2023energy}, we evaluate both detection reliability and in-distribution predictive performance using AUROC, AUPR, FPR95, and ID ACC.
\begin{itemize}
    \item AUROC denotes the area under the receiver operating characteristic curve. It evaluates the ability of a model to distinguish in-distribution and out-of-distribution samples across all possible decision thresholds. A higher AUROC indicates stronger discrimination performance.
    \item AUPR is the area under the precision--recall curve. It is particularly informative when the numbers of positive and negative samples are imbalanced, and is therefore widely adopted for OOD detection. For misclassification detection, we report both AUPR succ, where correctly classified samples are treated as positive, and AUPR err, where incorrectly classified samples are treated as positive.
    \item FPR95 measures the false positive rate when the true positive rate is fixed at 95\%. It reflects the extent to which negative samples are incorrectly identified as positive under a high-recall operating point. Lower FPR95 indicates better detection performance.
    \item ID ACC denotes the classification accuracy evaluated on in-distribution test samples. It is used to assess whether a method can provide reliable uncertainty estimates without sacrificing its predictive performance on normal data.
\end{itemize}

\subsection{More Empirical Results}
\label{result}
In Tables~\ref{table:ood_cora}--\ref{table:ood_ntu2012}, we report the AUROC/AUPR/FPR95 and in-distribution test accuracy on all datasets as complementary results to Table~\ref{table:ood_result} in the main text. In addition, Table~\ref{table:ood_std} reports the standard deviations of AUROC scores for OOD detection over ten independent runs, and Table~\ref{table:mis_std} reports the standard deviations of misclassification detection metrics.

\begin{table*}[ht]
\belowrulesep=0pt
\aboverulesep=0pt
\renewcommand{\arraystretch}{1.5}
\caption{Performance (\%) comparison of OOD detection on Cora.}
\label{table:ood_cora}
\centering
\begin{adjustbox}{width=\textwidth}
\begin{tabular}{c|cccc|cccc|cccc@{}}
\toprule
\multirow{3}{*}{Model}
& \multicolumn{4}{c|}{Label leave-out} & \multicolumn{4}{c|}{Feature interpolation} &  \multicolumn{4}{c}{Structure manipulation}\\ 
& AUROC & AUPR & FPR95 & ID ACC & AUROC & AUPR & FPR95 &  ID ACC & AUROC & AUPR & FPR95 &  ID ACC\\
\midrule  

MSP & 78.95   & 50.44 & 73.66 &  89.36 &  67.89 &  31.20 & 80.43 &  78.58 & 60.25 & 42.93 & 90.37 & 78.49\\          

ODIN & 44.15  & 21.20 & 96.21 & 89.41 & 52.78 &  24.60 & 90.05 &  78.55& 39.75 & 41.44 & 97.93 & 78.47\\

Mahalanobis    & 48.41   & 19.86 & 96.84  & 89.35 & 65.95 &  29.73 & 69.87 &  78.68& 65.27 & 41.67 & 81.70 & 78.36\\
\midrule

GNNSafe & 73.66   & 40.05 & 76.92  & 91.10 & 71.03  & 31.69 & 62.44 &  79.62 & 63.03 & 44.21 & 85.38 & 79.13\\

GPN & 71.27 & 39.72 & 82.09 & 90.68 & 70.72 & 33.66 & 70.27 & 77.85 & 64.98 & 39.59 & 81.76 & 78.26 \\

GNSD    & 87.04  & 67.31 & 54.69  & 91.49 & 76.86 &  36.89 & 60.71 &  76.37 & 60.52 & 51.68 & 90.65 & 76.34 \\

LGNSDE & 86.84 & 65.57 & 56.59 & 91.95 & 71.84 & 34.47 & 63.02 & 76.51 & 64.47 & 43.51 & 88.62 & 76.67 \\
\midrule

HGNN    & 77.46 & 62.88 & 68.81  & 91.95 & 71.49 &  35.11 & 76.14 &  77.10 & 64.59 & 47.28 &  87.65 & 76.88 \\

HyperGCN   & 74.81   & 43.00 & 85.14  & 93.32 & 69.27  & 34.30 & 71.11 &  78.32 & 65.65 & 49.25 & 89.45& 79.08\\

HGCN--Ens & 76.20 & 43.81 & 86.46 & 92.14 & 69.82 & 33.28 & 75.71 & 78.58 & 66.92 & 51.07 & 87.69 & 78.98 \\

HNDiffN    & 83.79 & 59.72 & 69.65  & 92.37 & 74.38  & 37.23 & 71.64 &  78.43 & 59.32 & 50.47 & 75.76& 76.96\\  

HND & 84.34 & 64.27 & 64.02 & 92.03 & 71.73 & 37.90 & 69.96 & 79.71 & 69.60 & 52.85 & 70.03 & 79.40 \\

HyperGOOD   & 85.13   & 66.02 & 58.18  & 93.27 & 77.44 &  49.86 & 58.51 &  78.37 & 67.71 & 53.24 &  78.62& 79.25  \\
\midrule

HyperNSD    & 85.68   & 68.03 & 55.68  & 93.49 & 80.61 & 53.70 & 52.27 & 80.25 & 73.34 & 57.95 & 64.07 & 80.07\\
\bottomrule
\end{tabular}
\end{adjustbox}
\end{table*}

\begin{table*}[ht]
\belowrulesep=0pt
\aboverulesep=0pt
\renewcommand{\arraystretch}{1.5}
\caption{Performance (\%) comparison of OOD detection on Cora-CA.}
\label{table:ood_cacora}
\centering
\begin{adjustbox}{width=\textwidth}
\begin{tabular}{c|cccc|cccc|cccc@{}}
\toprule
\multirow{3}{*}{Model}
& \multicolumn{4}{c|}{Label leave-out} & \multicolumn{4}{c|}{Feature interpolation} &  \multicolumn{4}{c}{Structure manipulation}\\ 
& AUROC & AUPR & FPR95 & ID ACC & AUROC & AUPR & FPR95 &  ID ACC & AUROC & AUPR & FPR95 &  ID ACC\\
\midrule  

MSP  & 71.57   & 40.62 & 87.25  & 78.81 & 75.82 & 41.61 & 74.48  & 68.03 & 58.42 & 32.20 & 93.86 & 68.13\\

ODIN & 48.51    & 20.53 & 96.52 & 78.81 & 51.58  & 45.08 & 97.05  & 68.01 & 41.58 & 24.35 & 95.61 & 68.32\\

Mahalanobis    & 48.29   & 19.55 & 92.94 & 78.78 & 75.06 & 32.67 & 54.17  & 67.99 & 58.81 & 36.85 & 93.09 & 68.24\\
\midrule

GNNSafe          & 65.23  & 34.25 & 82.20  & 79.24 & 80.03  & 41.15 & 57.61  & 69.44 & 61.15 & 43.67& 79.58 & 69.93\\

GPN & 68.82 & 37.76 & 79.63 & 75.85 & 78.53 & 42.21 & 73.60 & 66.66 & 67.39 & 49.01 & 71.71 & 66.61\\

GNSD    & 85.57  & 65.87 & 64.70  & 91.53 & 83.47  & 55.24 & 36.19  & 75.78 & 63.59 & 35.29  & 75.18 & 73.99 \\

LGNSDE & 83.58 & 60.98 & 79.24 & 89.41 & 80.49 & 55.74 & 35.66 & 73.86 & 59.60 & 34.04 & 86.98 & 72.38 \\

\midrule
HGNN    & 83.77   & 61.23 & 65.54  & 88.98 & 81.85  & 46.15 & 44.67  & 74.15 & 53.72 & 41.22 & 85.94 & 73.37\\

HyperGCN    & 71.49  & 42.71 & 84.62  & 91.10 & 85.11 & 48.69 & 55.06  & 73.10 & 62.91 & 40.43 & 82.14 & 72.28\\

HGCN--Ens & 72.69 & 43.56 & 78.51 & 89.75 & 84.36 & 50.02 & 56.87 & 73.66 & 61.92 & 38.54 & 82.43 & 73.29 \\

HNDiffN    & 83.24   & 58.93 & 70.28  & 88.14 & 79.56  & 50.44 & 49.60  & 74.45 & 57.50 & 42.23 & 83.42 & 72.52\\  

HND & 83.79 & 55.30 & 75.04 & 90.68 & 85.99 & 52.89 & 44.12 & 74.30 & 55.73 & 37.53 & 85.29 & 74.50 \\

HyperGOOD   & 86.56   & 67.60 & 43.56  & 91.91 & 86.61  & 54.57 & 38.68  & 76.81 & 64.30 & 44.52 & 77.92& 75.23\\
\midrule

HyperNSD     & 90.11  & 71.45 & 37.03  & 92.21 & 90.79 & 57.98& 20.70  & 75.78 & 70.37 & 53.44 & 63.35 & 75.84 \\
\bottomrule
\end{tabular}
\end{adjustbox}
\end{table*}

\begin{table*}[ht]
\belowrulesep=0pt
\aboverulesep=0pt
\renewcommand{\arraystretch}{1.5}
\caption{Performance (\%) comparison of OOD detection on CiteSeer.}
\label{table:ood_citeseer}
\centering
\begin{adjustbox}{width=\textwidth}
\begin{tabular}{c|cccc|cccc|cccc@{}}
\toprule
\multirow{3}{*}{Model}
& \multicolumn{4}{c|}{Label leave-out} & \multicolumn{4}{c|}{Feature interpolation} &  \multicolumn{4}{c}{Structure manipulation}\\ 
& AUROC & AUPR & FPR95 & ID ACC & AUROC & AUPR & FPR95 &  ID ACC & AUROC & AUPR & FPR95 &  ID ACC\\
\midrule 

MSP  & 71.82  & 46.27 & 79.26  & 75.72 & 70.86 & 32.95 & 68.57 & 71.14 & 56.12 & 51.30 & 81.69 & 71.14\\

ODIN & 51.19   & 25.54 & 96.30  & 75.51 & 45.68  & 20.50 & 96.65  & 70.99 & 43.17 & 43.37 & 96.89 & 71.10\\

Mahalanobis    & 58.06  & 32.19 & 97.22 & 75.66 & 68.93  & 36.36 & 75.00  &71.09 & 57.33 & 45.79 & 86.36& 71.04\\
\midrule

GNNSafe          & 67.55   & 38.95 & 89.82  & 75.40 & 69.12 & 31.98 & 68.03  & 73.53 & 74.82 & 57.33 & 76.70 & 72.59\\

GPN & 65.86 & 40.45 & 86.68 & 74.07 & 68.75 & 31.39 & 72.93 & 72.58 & 73.59 & 47.54 & 79.62 & 73.35 \\

GNSD    & 76.72   & 53.14 & 85.04  & 75.13 & 73.62  & 34.61 & 71.65  & 74.28 & 65.50 & 56.07 & 84.47& 74.21\\

LGNSDE & 68.32 & 47.48 & 84.50 & 74.78 & 78.65  & 47.44 & 59.60  & 76.71 & 64.03 & 55.77 & 88.97 & 76.98 \\
\midrule

HGNN    & 71.09   & 45.71 & 88.59  & 74.87 & 73.17  & 43.48 & 76.78  & 78.55 & 60.46 & 52.15 & 73.44 & 77.72 \\

HyperGCN    & 72.97   & 46.58 & 80.57  & 76.19 & 63.91  & 27.90 & 89.82  & 74.64 & 62.17 & 36.73 &90.46 & 74.83\\

HGCN--Ens & 73.32 & 47.83 & 82.42 & 76.07 & 65.40 & 30.05 & 83.19 & 75.28 & 63.58 & 34.89 & 87.39 & 75.81\\

HNDiffN    & 70.57  & 42.68 & 89.75  & 74.07 & 66.61 & 30.95 & 84.23 & 72.83 & 58.63 & 53.62 & 68.77 & 71.55 \\  

HND & 74.06 & 45.25 & 81.29 & 78.04 & 73.52 & 44.43 & 61.73 & 77.50 & 62.99 & 54.18 & 77.98 & 76.63 \\

HyperGOOD   & 74.62  & 58.36 & 61.52 & 76.93 & 71.43 & 37.21 & 71.92 & 74.76 & 68.41 & 52.11& 68.62 & 73.50\\
\midrule

HyperNSD     & 80.45    & 60.64 & 55.88 & 77.51 & 81.52  & 50.01 & 46.17  & 77.73 & 77.71 & 59.62 & 60.47 & 77.73\\
\bottomrule
\end{tabular}
\end{adjustbox}
\end{table*}

\begin{table*}[ht]
\belowrulesep=0pt
\aboverulesep=0pt
\renewcommand{\arraystretch}{1.5}
\caption{Performance (\%) comparison of OOD detection on DBLP.}
\label{table:ood_dblp}
\centering
\begin{adjustbox}{width=\textwidth}
\begin{tabular}{c|cccc|cccc|cccc@{}}
\toprule
\multirow{3}{*}{Model}
& \multicolumn{4}{c|}{Label leave-out} & \multicolumn{4}{c|}{Feature interpolation} &  \multicolumn{4}{c}{Structure manipulation}\\ 
& AUROC & AUPR & FPR95 & ID ACC & AUROC & AUPR & FPR95 &  ID ACC & AUROC & AUPR & FPR95 &  ID ACC\\
\midrule 

MSP     & 94.09   & 86.16 & 39.48  & 98.59 & 89.44  & 64.73 & 43.55  & 88.13 &  61.05  & 24.26 & 85.93 & 88.25\\

ODIN & 61.16   & 58.24 & 98.55  & 98.63 & 77.93  & 48.12 & 79.89  & 88.15 &  53.13 & 20.46 & 94.59 & 88.23\\

Mahalanobis    & 87.81  & 76.16 & 65.87  & 98.43 &  66.87  & 43.10 & 85.24 & 88.17  & 54.25  & 21.45 & 94.80   & 88.24\\
\midrule

GNNSafe       & 96.08   & 89.80 & 22.82  & 99.03 & 92.29   & 83.22  & 20.10 & 87.97 & 62.57 & 44.91 & 82.52 & 87.65\\

GPN & 85.24 & 74.08 & 60.98 & 99.11 & 95.26 & 85.52 & 21.86 & 88.82 & 68.36 & 56.78 & 87.68 & 88.90\\

GNSD    & 93.23  & 85.47 & 46.18  & 99.27 & 95.71  &  88.10 & 18.69 & 90.81 & 63.32 & 46.31 &  86.03  & 89.98\\

LGNSDE & 90.58 & 69.99 & 46.41 & 99.25 & 97.73 & 90.46 & 17.91 & 90.58 & 63.81 & 36.15 & 84.15 & 89.82 \\
\midrule

HGNN   & 74.42    & 66.24  & 64.47  & 99.23 & 85.31 & 80.43 &  57.72  & 87.70 & 58.23 & 42.09 &  87.28  & 87.75\\

HyperGCN    & 76.20  & 77.13 & 41.15  & 99.17 & 88.25 &  85.83  & 35.99  & 91.17 & 54.35 & 42.69 &  84.85  & 90.37\\

HGCN--Ens & 77.29 & 79.51 & 36.89 & 99.13 & 88.69 & 85.50 & 29.23 & 90.92 & 55.58 & 42.97 & 82.17 & 90.24 \\

HNDiffN    & 73.74   & 71.72 & 67.38  & 99.03 &96.69 & 87.71 &  13.47  &90.96  & 53.88 & 43.57 &  86.97  & 90.42\\  

HND & 91.58 & 85.39 & 24.41 & 99.34 & 93.63 & 85.71 & 17.25 & 89.82 & 54.27 & 43.04 & 83.91 & 88.61 \\

HyperGOOD   & 87.32   & 78.54 & 30.65  & 99.23 & 94.47 &  89.13  & 29.73  & 91.01 &   66.79  & 55.76 & 77.23   & 90.78\\
\midrule

HyperNSD     &  94.73  & 88.05 & 25.94  & 99.50 & 99.26 & 92.98 & 5.43  & 90.90 & 71.97 & 61.42 & 75.90   & 90.95 \\
\bottomrule
\end{tabular}
\end{adjustbox}
\end{table*}

\begin{table*}[ht]
\belowrulesep=0pt
\aboverulesep=0pt
\renewcommand{\arraystretch}{1.5}
\caption{Performance (\%) comparison of OOD detection on ModelNet40.}
\label{table:ood_model40}
\centering
\begin{adjustbox}{width=\textwidth}
\begin{tabular}{c|cccc|cccc|cccc@{}}
\toprule
\multirow{3}{*}{Model}
& \multicolumn{4}{c|}{Label leave-out} & \multicolumn{4}{c|}{Feature interpolation} &  \multicolumn{4}{c}{Structure manipulation}\\ 
& AUROC & AUPR & FPR95 & ID ACC & AUROC & AUPR & FPR95 &  ID ACC & AUROC & AUPR & FPR95 &  ID ACC\\
\midrule 

MSP   & 86.85   & 70.22 & 36.59  & 95.67  & 56.12 & 27.28 & 95.82 & 72.53 & 49.66 & 19.95 & 97.95 & 72.24\\

ODIN & 72.71  & 65.88 & 78.67 &  95.60 & 50.21 & 35.67 & 90.54 & 72.55 & 56.87 & 20.50 & 94.89  & 72.22 \\

Mahalanobis    & 60.62  & 54.97 & 90.29 &  95.67 & 64.25 & 53.04 & 98.64 & 72.59 & 60.21 & 21.67  & 81.86 & 72.24\\
\midrule

GNNSafe         & 89.67  & 50.79 & 45.83 &  96.93 & 81.31 & 52.71 & 55.51 & 80.81 & 63.58 & 29.34 & 81.52 & 80.39\\

GPN & 93.63 & 78.71 & 29.79 & 97.60 & 59.52 & 46.87 & 94.92 & 81.73 & 55.34 & 23.12 & 94.81 & 81.37 \\

GNSD    & 73.67   & 65.46 & 79.14  & 92.33 & 83.68 & 54.72 & 45.33 & 81.47 & 51.80 & 24.36 & 96.29 &80.55\\

LGNSDE & 75.96 & 55.50 & 63.16 & 92.67 & 85.14 & 58.19 & 33.82 & 83.41 & 62.68 & 34.69 & 77.84 & 83.12\\
\midrule

HGNN   & 84.76  & 71.48 & 38.27 & 93.73 &84.90 & 58.65 & 37.00 & 88.73 & 57.82 & 33.39 & 74.27 & 88.62\\

HyperGCN    & 82.87   & 51.28 & 68.61  & 93.73 & 81.03 & 45.52 & 47.24 & 87.52 & 46.77 & 23.85 & 87.92 & 87.72\\

HGCN--Ens & 84.13 & 51.94 & 70.83 & 94.20 & 79.52 & 46.87 & 44.29 & 87.31 & 45.65 & 24.46 & 87.37 & 88.39\\

HNDiffN    & 86.52   & 82.19 & 28.63  & 94.87 & 85.78 & 56.31 & 32.65 & 91.71 & 64.98 & 38.53 & 73.14 & 90.35\\  

HND & 95.15 & 90.27 & 22.38 & 98.59 & 83.32 & 51.89 & 50.82 & 95.62 & 53.59 & 35.92 & 85.17 & 94.24 \\

HyperGOOD   & 89.25  & 88.95 & 36.61  & 97.47 & 86.68 & 61.21 & 30.62 & 93.61 & 76.79 & 43.76 & 65.23 & 92.98\\
\midrule

HyperNSD   & 98.64   & 92.55 & 9.30  & 98.70 & 89.96 & 68.55  & 18.77 &  95.51 & 79.89 & 46.12 & 58.42 & 95.60\\
\bottomrule
\end{tabular}
\end{adjustbox}
\end{table*}

\begin{table*}[ht]
\belowrulesep=0pt
\aboverulesep=0pt
\renewcommand{\arraystretch}{1.5}
\caption{Performance (\%) comparison of OOD detection on NTU2012.}
\label{table:ood_ntu2012}
\centering
\begin{adjustbox}{width=\textwidth}
\begin{tabular}{c|cccc|cccc|cccc@{}}
\toprule
\multirow{3}{*}{Model}
& \multicolumn{4}{c|}{Label leave-out} & \multicolumn{4}{c|}{Feature interpolation} &  \multicolumn{4}{c}{Structure manipulation}\\ 
& AUROC & AUPR & FPR95 & ID ACC & AUROC & AUPR & FPR95 &  ID ACC & AUROC & AUPR & FPR95 &  ID ACC\\
\midrule  

MSP    & 80.17  & 64.61 & 49.93 & 88.72 & 54.75 & 39.08 & 95.83 & 67.20 & 50.37 & 16.15 & 97.98 & 67.14\\

ODIN & 51.49  & 43.35 & 89.22 & 88.78 & 50.84 &  33.92& 90.68 & 67.28 & 53.11 & 19.84 & 95.26 & 67.15 \\

Mahalanobis    & 45.49  & 35.07 & 91.04  &  88.86 & 61.87 & 55.12 & 87.79 & 67.27 & 57.07 & 17.85 & 92.77 & 67.14\\
\midrule

GNNSafe        & 83.80  & 55.54 & 67.76  & 89.88 & 79.75 & 53.98 & 60.76 & 76.53 & 65.90 & 18.64 & 69.51 & 75.65\\

GPN & 84.22 & 63.57 & 61.43 & 86.38 & 53.50 & 38.17 & 92.01 & 67.30 &51.59 & 22.15 & 94.73 & 67.30 \\

GNSD    & 62.42  & 44.85 & 86.27  & 81.87 & 90.86 & 59.74 & 35.27 & 84.31 & 53.42 & 26.38 & 93.47 & 84.17\\

LGNSDE & 65.91 & 50.45 & 74.82 & 82.68 & 85.82 & 61.09 & 30.25 & 83.82 & 52.43 & 34.19 & 77.02 & 84.36 \\
\midrule

HGNN    & 79.73   & 66.25 & 58.87 & 89.73 & 91.85 & 59.93 & 23.41 & 87.40 & 59.41 & 19.24 & 85.53 & 85.57\\

HyperGCN   
& 65.09   & 45.90 & 78.19  & 85.95 & 90.13 & 55.87 & 40.61 & 79.54 & 70.81 & 32.31 & 73.48 & 79.13\\

HGCN--Ens & 66.52 & 46.87 & 74.98 & 86.73 & 89.22 & 56.53 & 41.43 & 80.38 & 70.52 & 33.59 & 75.33 & 80.65 \\

HNDiffN    & 88.43 & 69.46 & 57.53 &  87.17 & 88.28 & 58.06 & 37.92 & 82.67 & 69.29 & 28.74 & 76.31 & 82.37\\  

HND & 78.79 & 68.09 & 43.33 & 88.66 & 89.84 & 62.02 & 31.81 & 87.84 & 70.78 & 38.22 & 70.12 & 88.77 \\

HyperGOOD   & 86.87   & 71.38 & 47.33 & 88.21 & 93.53 & 64.45 & 26.47 & 87.20 & 73.20 & 43.77 & 65.37 & 88.65\\
\midrule

HyperNSD  & 92.33   & 73.16 & 35.74  & 90.78 & 95.10  & 72.35 & 21.16  & 89.17 & 75.95 & 49.91 & 59.54 & 90.11\\
\bottomrule
\end{tabular}
\end{adjustbox}
\end{table*}

\begin{table*}[ht]
\belowrulesep=0pt
\aboverulesep=0pt
\renewcommand{\arraystretch}{1.5}
\caption{Standard deviations of AUROC scores for OOD detection over ten independent runs. L, F, and S denote label leave-out, feature interpolation, and structure manipulation, respectively.}
\label{table:ood_std}
\centering
\begin{adjustbox}{width=\textwidth}
\begin{tabular}{|c|ccc|ccc|ccc|ccc|ccc|ccc|}
\toprule
\multirow{3}{*}{Model}
& \multicolumn{3}{c|}{Cora} & \multicolumn{3}{c|}{Cora-CA} &  \multicolumn{3}{c|}{Citeseer} & \multicolumn{3}{c|}{DBLP}&\multicolumn{3}{c|}{ModelNet40}&\multicolumn{3}{c|}{NTU2012}\\ 
& L & F & S & L & F & S & L & F & S & L & F & S & L & F & S & L & F & S\\
\midrule  

MSP & 1.32 & 1.38 & 0.78 & 1.69 & 2.64 & 1.57 & 1.60 & 0.67 & 1.18 & 1.22 & 1.46 & 0.31 & 2.83 & 0.57 & 0.68 & 2.06 & 1.07 & 1.64\\          

ODIN & 3.52 & 0.62 & 0.97 & 3.27 & 0.92 & 1.11 & 3.53 & 3.58 & 2.93 & 2.45 & 3.60 & 2.21 & 3.29 & 0.72 & 0.94 & 2.84 & 1.31 & 1.84\\  

Mahalanobis & 1.76 & 0.41 & 1.81 & 1.17 & 1.39 & 2.11 & 0.58 & 0.39 & 1.13 & 0.49 & 0.48 & 0.56 & 0.85 & 0.30 & 0.72 & 0.50 & 0.56 & 2.43\\  
\midrule

GNNSafe & 2.28 & 0.72 & 0.83 & 1.99 & 0.73 & 2.20 & 1.20 & 0.78 & 0.89 & 3.79 & 0.52 & 0.69 & 3.19 & 3.66 & 2.76 & 2.36 & 0.33 & 1.46\\  

GPN & 1.97 & 1.10 & 1.90 & 2.16 & 0.95 & 2.35 & 1.26 & 2.54 & 3.08 & 3.50 & 2.25 & 3.22 & 1.35 & 0.70 & 0.52 & 0.96 & 1.16 & 1.18\\  

GNSD & 2.78 & 1.69 & 1.99 & 1.88 & 2.86 & 1.60 & 1.30 & 1.55 & 2.52 & 1.70 & 0.53 & 0.31 & 1.41 & 2.31 & 2.81 & 4.27 & 4.41 & 1.95\\  

LGNSDE & 2.92 & 0.83 & 3.07 & 2.22 & 1.54 & 1.69 & 1.48 & 0.67 & 1.29 & 2.15 & 0.83 & 1.62 & 1.68 & 1.91 & 4.37 & 3.47 & 2.25 & 1.68\\  
\midrule

HGNN & 0.69 & 1.09 & 1.23 & 0.78 & 1.22 & 2.82 & 1.38 & 0.74 & 0.97 & 1.05 & 0.25 & 0.37 & 2.06 & 1.81 & 1.07 & 2.09 & 1.09 & 1.20\\  

HyperGCN & 2.01 & 1.24 & 1.88 & 1.58 & 0.71 & 1.19 & 1.28 & 0.72 & 1.04 & 0.50 & 0.14 & 0.74 & 1.72 & 1.44 & 2.67 & 2.04 & 1.10 & 3.72\\

HGCN--Ens & 2.29 & 1.38 & 1.97 & 1.55 & 1.13 & 1.42 & 1.37 & 0.94 & 1.47 & 1.65 & 0.27 & 0.86 & 2.49 & 3.04 & 2.08 & 2.40 & 1.32 & 3.58\\  

HNDiffN  & 1.82 & 2.71 & 2.69 & 1.92 & 3.67 & 2.67 & 2.22 & 2.75 & 2.34 & 0.17 & 1.12 & 2.22 & 2.58 & 1.39 & 1.89 & 3.64 & 1.82 & 4.37\\  

HND & 0.27 & 2.72 & 4.16 & 1.13 & 1.42 & 1.56 & 1.75 & 2.54 & 3.45 & 0.33 & 0.45 & 1.28 & 1.15 & 0.25 & 3.37 & 2.62 & 0.37 & 2.51\\  

HyperGOOD & 2.53 & 1.11 & 0.71 & 2.06 & 1.06 & 1.98 & 1.21 & 0.95 & 2.24 & 0.65 & 0.48 & 0.88 & 0.61 & 3.16 & 1.10 & 1.03 & 3.21 & 3.75\\  
\midrule

HyperNSD & 0.65 & 2.38 & 0.31 & 0.92 & 2.17 & 1.00 & 0.27 & 1.82 & 0.47 & 1.34 & 0.27 & 0.34 & 0.16 & 2.18 & 0.32 & 0.58 & 1.37 & 1.16\\  
\bottomrule
\end{tabular}
\end{adjustbox}
\end{table*}

\begin{table*}[!htbp]
\belowrulesep=0pt
\aboverulesep=0pt
\renewcommand{\arraystretch}{1.5}
\caption{Standard deviations of misclassification detection performance over ten independent runs.}
\label{table:mis_std}
\centering
\begin{adjustbox}{width=\textwidth}
\begin{tabular}{|c|c|cccc|cccccc|c|@{}}
\toprule

Dataset & \diagbox{Metric}{Model} & GNNSafe & GPN & GNSD & LGNSDE & HGNN  & HyperGCN & HGCN--Ens& HNDiffN &HND & HyperGOOD & HyperNSD\\
\midrule

\multirow[c]{4}{*}{Cora} & AUROC & 1.77 & 1.26 & 3.34 & 1.19  & 1.87 & 1.38 & 2.04 & 2.46 & 0.73 & 2.22 & 1.61\\

& AUPR succ & 4.12 & 3.56 & 3.38 & 4.54  & 1.37 & 1.02 & 1.19 & 7.69 & 1.62 & 2.87 & 1.13\\

& AUPR err & 5.88 & 4.59 & 2.11 & 8.41 & 3.40 & 4.08 & 2.77 & 0.91 & 3.74 & 2.73 & 2.59\\

&FPR95 & 7.32 & 8.52 & 3.78 & 8.63  &  1.25 & 6.68 & 3.84 & 5.65 & 4.30 & 3.19 & 3.17\\
\midrule

\multirow[c]{4}{*}{Cora-CA} &AUROC & 0.93 & 1.17 & 2.35 & 1.54  & 2.41 & 2.26 & 1.05 & 3.27 & 2.96 & 2.50 & 1.81\\

 & AUPR succ & 1.19 & 2.09 & 1.00 &  3.84 & 1.11 & 1.41 & 2.14 & 4.29 & 3.79 & 1.99 & 2.48\\

& AUPR err & 5.28 & 4.50 & 3.46 &  7.33 & 3.42 & 6.90 & 3.70 & 4.43 & 4.56 & 5.09 & 2.24\\

&FPR95 & 5.53 & 1.61 & 4.47 & 5.63  & 4.75 & 3.03 & 1.99 & 2.91 & 4.31 & 3.33 & 3.83\\
\midrule

\multirow[c]{4}{*}{Citeseer}& AUROC & 1.89 & 2.01 & 1.14 & 1.22  & 1.27 & 1.91 & 2.37 & 3.96 & 4.96 & 4.18 & 2.77\\

& AUPR succ & 6.75 & 5.42 & 2.23 &  4.23 & 1.27 & 5.83 & 4.44 & 7.47 & 3.08 & 4.15 & 1.82\\

& AUPR err & 5.05 & 5.43 & 4.18 & 4.70  & 3.39 & 6.51 & 5.31 & 6.28 & 3.94 & 3.02 & 1.34\\

&FPR95 & 4.96 & 4.15 & 5.79 &  5.08 & 3.33 & 3.82 & 3.53 & 3.28 & 3.99 & 3.79 & 2.31\\
\midrule

\multirow[c]{4}{*}{DBLP} & AUROC & 0.98 & 4.37 & 0.58 & 4.28  & 0.69 & 0.39 & 0.93 & 0.86 & 0.76 & 1.09 & 1.19\\

& AUPR succ & 0.48 & 0.67 & 0.09 &  0.73 & 0.07 & 0.06 & 0.11 & 0.64 & 0.13 & 0.26 & 0.74\\

& AUPR err & 4.01 & 3.81 & 1.77 & 4.74  & 2.06 & 1.50 & 1.37 & 3.58 & 3.15 & 3.59 & 1.39\\

&FPR95 & 1.53 & 7.56 & 2.20 & 7.16  & 2.68 & 2.28 & 2.85 & 4.09 & 2.66 & 3.46 & 2.15\\
\midrule

\multirow[c]{4}{*}{ModelNet40} & AUROC & 1.59 & 2.10 & 0.53 & 1.14  & 0.26 & 0.04 & 0.13 & 0.70 & 0.17 & 0.32 & 0.06\\

& AUPR succ & 0.89 & 1.45 & 0.52 &  0.65 & 0.28 & 0.05 & 0.35 & 0.87 & 0.30 & 0.23 & 0.11\\

& AUPR err & 1.32 & 1.60 & 0.73 & 0.88  & 0.22 & 0.02 & 0.08 & 0.05 & 0.25 & 0.24 & 0.14\\

&FPR95 & 2.24 & 2.91 & 4.73 &  1.05 & 1.25 & 0.24 & 0.36 & 0.56 & 0.24 & 1.15 & 0.22\\
\midrule

\multirow[c]{4}{*}{NTU2012} & AUROC & 2.17 & 1.51 & 0.33 &  0.17 & 0.19 & 0.21 & 0.36 & 0.59 & 0.14 & 0.22 & 0.11\\

& AUPR succ & 2.33 & 2.95 & 0.21 &  0.32 & 0.11 & 0.16 & 0.07 & 0.27 & 0.31 & 0.12 & 0.20\\

& AUPR err & 1.17 & 1.25 & 0.78 &  1.04 &  0.47 & 0.27 & 0.30 & 0.45 & 0.26 & 0.23 & 0.09\\

&FPR95 & 2.23 & 1.52 & 3.69 &  1.78 & 0.80 & 0.47 & 0.63 & 0.77 & 0.16 & 0.52 & 0.36\\
\bottomrule
\end{tabular}
\end{adjustbox}
\end{table*}

\end{document}